\documentclass{article}

\IfFileExists{iclr2027_conference.sty}{%
  \usepackage{iclr2027_conference,times}
}{%
  \usepackage{times,natbib}
}
\usepackage{amsmath}
\usepackage{amssymb}
\usepackage{mathtools}
\usepackage{booktabs}
\usepackage{xcolor}
\usepackage{colortbl}
\usepackage{graphicx}
\usepackage{microtype}
\usepackage{multirow}
\usepackage{array}
\usepackage{tabularx}
\usepackage{placeins}
\usepackage{algorithm}
\usepackage{algpseudocode}
\usepackage{hyperref}
\usepackage{url}
\usepackage{wrapfig}
\usepackage{needspace}
\usepackage{subcaption}
\usepackage{enumitem}
\usepackage{caption}

\usepackage{amsmath,amsfonts,bm,amsthm}

\def\eqref#1{equation~\ref{#1}}

\def\1{\bm{1}}

\DeclareMathAlphabet{\mathsfit}{\encodingdefault}{\sfdefault}{m}{sl}
\SetMathAlphabet{\mathsfit}{bold}{\encodingdefault}{\sfdefault}{bx}{n}

\newtheorem{theorem}{Theorem}
\newtheorem{assumption}{Assumption}
\newtheorem{definition}{Definition}
\newtheorem{lemma}{Lemma}
\newtheorem{proposition}{Proposition}
\newtheorem{corollary}{Corollary}

\definecolor{eecaccent}{HTML}{256D72}
\definecolor{eecwash}{HTML}{EDF5F4}
\definecolor{tablehead}{HTML}{F1F3F4}
\definecolor{tablemuted}{HTML}{747C80}
\newcommand{\eqcmethod}{\textcolor{eecaccent}{\textbf{EqC (ours)}}}
\newcommand{\eecrow}{\rowcolor{eecwash}}
\newcommand{\headrow}{\rowcolor{tablehead}}
\newcommand{\paneltitle}[1]{\textcolor{eecaccent}{\bfseries #1}}
\newcommand{\best}[1]{\textbf{#1}}
\newcommand{\eh}[2]{#1\,\textcolor{tablemuted}{/}\,#2}
\newcommand{\tabstyle}{\footnotesize\setlength{\tabcolsep}{3pt}\renewcommand{\arraystretch}{1.12}}

\definecolor{revisionblue}{RGB}{0,65,175}

\title{EqC: Equilibrium Energy Control for \mbox{Introspective} Lifelong Robot Learning}

\author{
  \makebox[0.96\textwidth][c]{\small\bfseries
    Haodong Yang$^{1,*}$\hspace{1.3em}
    Runhan Huang$^{1,*,\dagger}$\hspace{1.3em}
    Bowen Zhang$^{1,2,*}$\hspace{1.3em}
    Iofai Lei$^{1,*}$\hspace{1.3em}
    Zunzhe Zhang$^{1}$}\\[-1pt]
  \makebox[0.96\textwidth][c]{\small\bfseries
    Xiangcheng Zhang$^{1}$\hspace{1.3em}
    Shaoting Zhu$^{1,2}$\hspace{1.3em}
    Linzhan Mou$^{3}$\hspace{1.3em}
    Yicheng Liu$^{1,2}$\hspace{1.3em}
    Hang Zhao$^{1,2,\ddagger}$}\\[4pt]
  \makebox[0.96\textwidth][c]{\normalfont\small
    $^{1}$Tsinghua University\hspace{2em}
    $^{2}$Galaxea AI\hspace{2em}
    $^{3}$Princeton University}\\[-1pt]
  \makebox[0.96\textwidth][c]{\normalfont\scriptsize
    $^{*}$Equal contribution, random order.\hspace{1.5em}
    $^{\dagger}$Project lead.\hspace{1.5em}
    $^{\ddagger}$Corresponding author.}
}

\iclrfinalcopy

\begin{document}
\raggedbottom

\maketitle
\vspace{-1.5em}
\begin{center}
  \small
  \textbf{Code:}
  \href{https://github.com/hrh6666/EqC}
       {\texttt{https://github.com/hrh6666/EqC}}
  \\[2pt]
  \textbf{Project page:}
  \href{https://equilibrium-energy-control.github.io/}
       {\texttt{https://equilibrium-energy-control.github.io}}
\end{center}
\vspace{1.5em}
\lhead{Preprint}
\begingroup
\renewcommand{\thefootnote}{\ensuremath{\ddagger}}
\footnotetext[0]{Correspondence to Hang Zhao:
  \href{mailto:hangzhao@mail.tsinghua.edu.cn}
       {\texttt{hangzhao@mail.tsinghua.edu.cn}}.}
\endgroup
\vspace{-1.2em}

\begin{center}
  \begin{minipage}{\linewidth}
    \centering
    \includegraphics[width=0.85\linewidth]{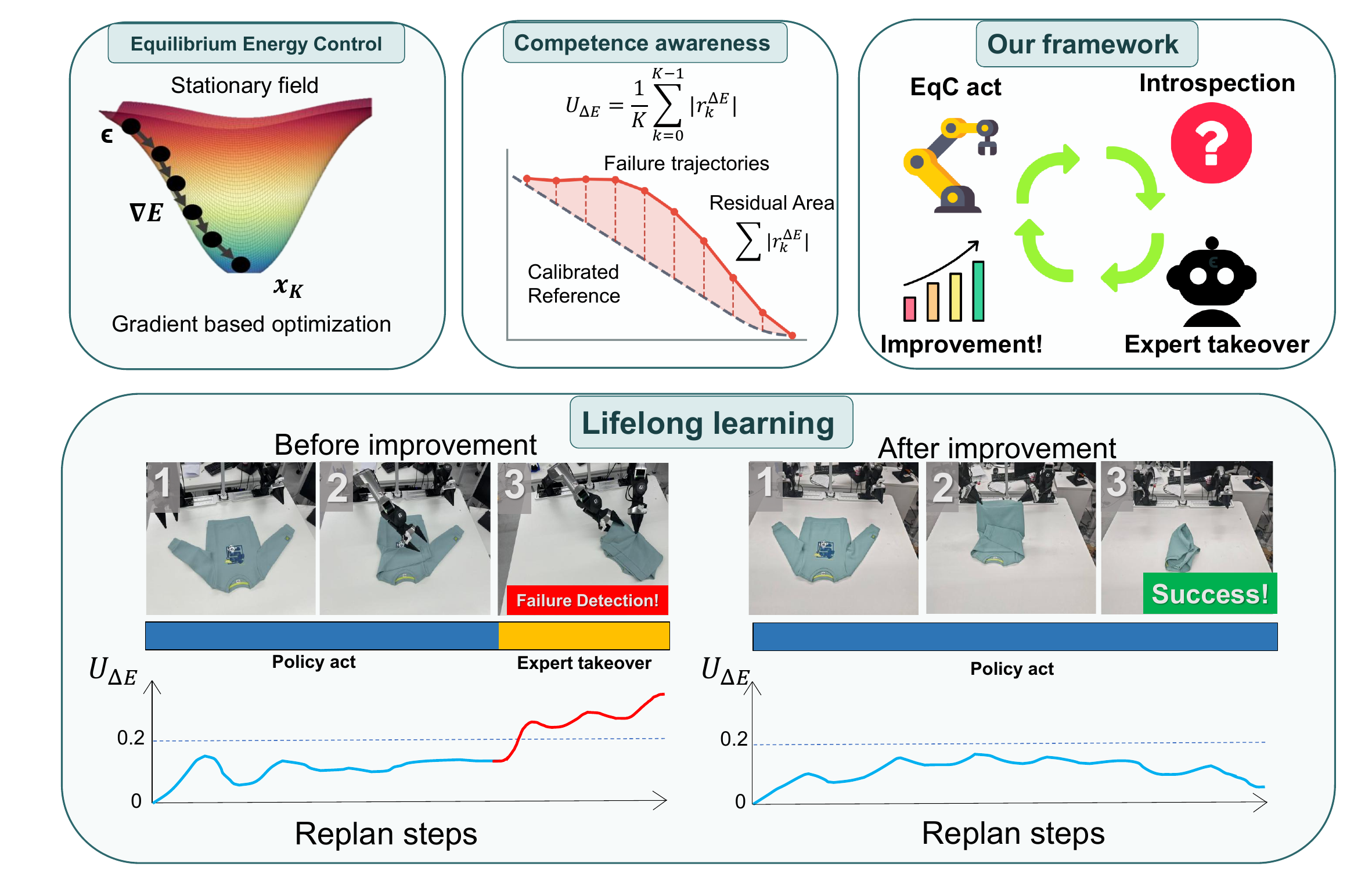}
    \captionof{figure}{\textbf{Overview of Equilibrium Energy Control (EqC).}
    \textbf{Top:} EqC generates actions through energy-based refinement, while the geometry of the refinement path reveals when autonomous execution is weakly supported. Energy-Gated DAgger uses this signal to trigger expert takeover and subsequent policy improvement.
    \textbf{Bottom:} In real-world deployment, persistent $U_{\Delta E}$ alarms prompt timely assistance before failure, progressively converting assisted interventions into fully autonomous execution through lifelong learning.}
    \label{fig:teaser}
  \end{minipage}
\end{center}
\vspace{-0.2em}

\begin{abstract}
Long-term robot autonomy depends on more than effective action generation: a policy must recognize when its behavior is no longer sufficiently supported and use assistance to improve. We introduce \textbf{Equilibrium Energy Control (EqC)}, a framework for introspective lifelong robot learning built around an explicit conditional energy landscape. EqC casts generative action refinement as equilibrium-seeking optimization, while the geometry of the refinement process reveals whether the generated behavior remains consistent with the learned landscape. This competence signal enables Energy-Gated DAgger, where persistent energy-consistency violations trigger expert takeover and the resulting recovery trajectories are aggregated for subsequent policy updates. EqC integrates naturally into both Diffusion Transformer policies and flow-matching-based vision-language-action (VLA) policies. Across simulation and real-world bimanual manipulation, EqC matches or improves control performance and robustness while enabling timely intervention near the policy's competence boundary. A real-world study further shows how these interventions progressively expand the range of situations handled autonomously, providing a practical path from static generative controllers toward self-improving robot systems. 
\end{abstract}

\section{Introduction}\label{introduction}

Robots deployed in open-ended environments must manage their autonomy as operating conditions change~\citep{farid2024task,filos2020can}. A policy may perform reliably when the current situation is well represented by prior experience, yet become unreliable as deployment approaches the limits of that experience. This transition defines the policy's context-dependent competence boundary. An introspective robot recognizes its proximity to this boundary and regulates autonomous execution before uncertainty develops into failure. Such awareness is also central to lifelong robot learning: assistance requested near the competence boundary protects deployment and reveals where corrective experience is most valuable. Modern generative robot policies, however, are commonly deployed as static action generators that provide little indication of whether their generated behavior remains reliable in the current context~\citep{chi2025diffusion,black2024pi0}. Competence awareness can instead turn these boundary cases into targeted opportunities for learning rather than avoidable failures.

We introduce \textbf{Equilibrium Energy Control (EqC)}, an equilibrium energy formulation inspired by Equilibrium Matching~\citep{wang2025equilibrium} that equips generative robot policies with an explicit conditional energy landscape. Given visual observations, language instructions, and robot state, EqC optimizes an action chunk on this landscape and executes a short prefix under receding-horizon control. The optimization trajectory also reveals whether refinement follows the energy-progress relation learned from expert demonstrations, providing an introspective view of the policy's competence during execution. Energy-Gated DAgger uses the resulting path-wise energy-consistency signal to trigger expert intervention and aggregate recovery trajectories near the policy's current competence boundary. This condition-dependent energy formulation therefore supports action generation, online competence monitoring, and targeted data acquisition, while integrating naturally with both Diffusion Transformer policies~\citep{peebles2023scalable} and flow-matching-based vision-language-action (VLA) policies~\citep{black2024pi0,physicalintelligence2025pi05}.

We evaluate EqC on simulation benchmarks including LIBERO~\citep{liu2023libero}, RoboTwin 2.0~\citep{chen2025robotwin}, VLABench~\citep{zhang2025vlabench}, and LIBERO-Plus~\citep{fei2025libero}, together with real-world bimanual manipulation. EqC matches or improves control performance and robustness across visual and task variation, while its geometry-derived competence signal enables online detection of failure-prone execution. In real-world lifelong learning experiments, the resulting requests for assistance concentrate recovery data near the policy's competence boundary, and subsequent updates convert previously assisted situations into autonomous capabilities. This process progressively enlarges the range of situations the robot can handle independently, which we refer to as \emph{autonomous competence expansion}. Code is available on the \href{https://github.com/hrh6666/EqC}{https://github.com/hrh6666/EqC}. Video is provided in the \href{https://equilibrium-energy-control.github.io/}{https://equilibrium-energy-control.github.io/}. 

Our contributions are summarized as follows:
\vspace{-2pt}
\begin{itemize}[itemsep=0.1em, topsep=0.2em]
    \item \textbf{Equilibrium energy control.} EqC formulates action generation as equilibrium-seeking optimization under an explicit conditional energy function. The formulation integrates with both Diffusion Transformer and flow-matching-based VLA policies, providing a common energy-based control mechanism across distinct policy families.
    \item \textbf{Introspection from energy geometry.} EqC characterizes finite energy progress across the action-refinement trajectory and measures its agreement with the learned equilibrium relation. The resulting signal detects when autonomous execution approaches the policy's competence boundary and provides an online criterion for timely assistance.
    \item \textbf{Energy-Gated DAgger for lifelong robot learning.} Persistent energy-consistency violations trigger expert recovery, and the resulting trajectories are aggregated to update the policy around its current competence boundary. Real-world deployment shows how this closed loop converts assisted situations into autonomous capabilities and progressively expands robot competence.
\end{itemize}

\vspace{-5pt}
\section{Related Work}\label{related-work}
\vspace{-5pt}
\paragraph{Generative models for robotics.}
\looseness=-1
Generative modeling has driven substantial progress in robot learning through diffusion-based
policies~\citep{janner2022planning,ajay2022conditional,chi2025diffusion,huang2024diffuseloco,huang2025flexible}, world
models~\citep{zhang2026world,chen2025large,yuan2026fast}, and vision-language-action policies~\citep{wang2026qwen}. Within
VLA policies, iterative action generation is commonly instantiated with diffusion, as in RDT-1B~\citep{liu2025rdt}, or
with flow-matching action experts, as in \(\pi_0\) and \(\pi_{0.5}\)~\citep{lipman2022flow,black2024pi0,
physicalintelligence2025pi05}. Recent analysis relates the effectiveness of these generative policies to noise-covered
supervision and iterative computation~\citep{pan2026much}. EqC builds on this structure by realizing action refinement
as optimization over an explicit conditional energy landscape.

\paragraph{Energy-based models for robotics.}
\looseness=-1
Energy-based models represent compatibility through a scalar function over candidate configurations, turning inference into optimization on a learned landscape~\citep{lecun2006tutorial}. In robot learning, observation--action energy functions have been learned for imitation and used to optimize actions at inference time~\citep{florence2022implicit}, with subsequent work improving their contrastive objectives, interpolation behavior, and scaling to transformer policies~\citep{singh2023revisiting,davies2025ebt}. Recent energy-based VLA work uses terminal action energy as a confidence measure for visual anomaly detection, while leaving its use for intervention-driven real-world data aggregation unexplored~\citep{naghiyev2026energy}. Equilibrium Matching (EqM)~\citep{wang2025equilibrium} learns stationary generative dynamics with an implicit energy structure, and GeCO~\citep{zhang2026generative} carries this implicit-energy formulation into robot control through iterative refinement and adaptive computation. \citet{feng2025guidance} investigate energy-guided flow matching, where an objective energy steers generation toward an energy-reweighted target distribution. EqC instead unifies energy-based action generation, introspective failure detection, and intervention-driven lifelong learning through an explicit conditional energy whose gradient directly drives action refinement and whose path-wise energy progression enables competence assessment and targeted expert intervention.

\paragraph{Failure detection and competence-aware autonomy.}
\looseness=-1
Failure detection and uncertainty estimation provide useful signals for deciding when autonomous execution should yield to
supervision~\citep{farid2024task,filos2020can}. DAgger establishes the iterative aggregation
of on-policy expert corrections~\citep{ross2011reduction}, while robot-gated variants use predicted action error,
ensemble disagreement, state novelty, and learned risk to decide when those corrections are
needed~\citep{zhang2016query,menda2019ensembledagger,hoque2021thriftydagger}. This principle also appears in
competence-aware autonomy, where a robot adjusts its level of autonomy according to its estimated ability in the current
situation~\citep{basich2020learning}. For generative manipulation policies, diffusion training loss has also served as a
failure signal for interactive data collection~\citep{lee2025diff}. EqC instead reads introspective information from the
consistency of energy progress across the full refinement path and uses persistent violations to gate recovery data
collection.

\vspace{-5pt}
\section{EqC for Introspective Lifelong Robot Learning}\label{eqc-method}
\vspace{-5pt}
\subsection{Overview}\label{overview}
\vspace{-3pt}
EqC organizes robot control and lifelong improvement around the condition-dependent energy function
\(E_\theta(x,z)\). Its conditional energy landscape refines actions through energy gradients, while finite energy
changes along the same refinement trajectory provide an introspective signal of policy competence. Energy-Gated DAgger uses this signal to
request expert intervention and aggregate targeted recovery experience. We first formulate energy-based action refinement,
then derive its introspective readout and use that signal for lifelong learning.
\vspace{-3pt}
\subsection{Equilibrium Energy Control}\label{equilibrium-energy-learning}
\vspace{-3pt}
We consider a robot policy that predicts a continuous action chunk \(a\in\mathbb{R}^D\) from a condition \(z\), which
may include images, a language instruction, robot proprioception, and recent observation history. Demonstrations are
drawn from an unknown conditional distribution \(p_{\mathrm{data}}(a\mid z)\). We parameterize the bridge from noise
to data by \(t\in[0,1]\), with \(t=0\) at noise and \(t=1\) at data, and sample \(t\) uniformly during training. For
\(\varepsilon\sim\mathcal{N}(0,I)\), the bridge displacement, intermediate state, and remaining schedule integral are
\begin{equation}
  d:=a-\varepsilon,
  \qquad x_t:=\varepsilon+t d,
  \qquad C(t):=\int_t^1 c(s)\,\mathrm ds.
  \label{eq:bridge-definitions}
\end{equation}
A nonnegative schedule \(c(t)\) controls the magnitude of the target field \(c(t)d\).

EqC represents action quality with a differentiable energy
\(E_\theta:\mathbb{R}^{D}\times\mathcal Z\to\mathbb R\). Its gradient
\(g_\theta(x,z)=\nabla_xE_\theta(x,z)\) is a stationary control field: it depends on the current action and condition,
but not on the bridge coordinate. At \(x_t\), the demonstrated displacement supplies the desired direction, and the
energy is trained with the gradient-matching objective
\begin{equation}
  \mathcal{L}_{\mathrm{grad}}
  = \mathbb{E}_{(a,z),\varepsilon,t}
  \left[
    \frac{1}{D}
    \left\|\nabla_xE_\theta(x_t,z)-c(t)d\right\|_2^2
  \right].
  \label{eq:gradient-matching}
\end{equation}

The bridge coordinate is used only to sample training targets; the deployed field depends on the current action and
condition alone. At inference, EqC initializes \(x^{(0)}\sim\mathcal N(0,I)\) and applies
\(x^{(k+1)}=x^{(k)}+\eta_k\nabla_xE_\theta(x^{(k)},z)\) for \(k=0,\ldots,K-1\), yielding \(\hat x=x^{(K)}\). The
controller executes a prefix and replans from the next observation. Standard smoothness conditions give monotonic
energy progress for sufficiently stable step sizes (Appendix~\ref{app:energy-ascent}).

\paragraph{Architecture compatibility.}\label{architecture-instantiations}
\looseness=-1
The framework requires only a differentiable, per-sample scalar over an action chunk. It may be produced by a dedicated
energy head or by a differentiable scalarization of action-conditioned features. In our implementations, a stationary
backbone representation \(f_\theta\) defines \(E_\theta(x,z)=\langle x,f_\theta(x,z)\rangle\) over active action
dimensions. The DiT model is conditioned on RGB observations, language, and robot state. The VLA model
retains the \(\pi_0\) multimodal prefix, action expert, action representation, and padding convention. Both expose the
same interface \(x\mapsto E_\theta(x,z)\mapsto\nabla_xE_\theta(x,z)\), so EqC changes the action objective and its
optimization through an interface shared across observation encoders and policy scales.

\vspace{-3pt}
\subsection{Introspection from Energy Geometry}
\vspace{-3pt}
\label{global-path-integral-consistency}
\label{energy-based-competence-awareness}
\label{equilibrium-energy-consistency-uncertainty}

The gradient-matching target also determines finite energy progress along a demonstration bridge. For
\(x_s=\varepsilon+sd\), integrating
\(\frac{d}{ds}E_\theta(x_s,z)=\nabla_xE_\theta(x_s,z)^\top d=c(s)\|d\|_2^2\) gives
\begin{equation}
E_\theta(a,z)-E_\theta(x_t,z)=\|d\|_2^2C(t).
\end{equation}
Violation of this relation along a realized refinement trajectory reveals disagreement between local optimization and
the energy geometry learned from demonstrations.

At deployment, we treat \(\hat x=x^{(K)}\) as the approximate equilibrium and use it as a self-anchor. Let
\(v=\hat x-x^{(0)}\) and \(D_v=\max(\|v\|_2^2,\delta)\). Projection onto the noise-to-equilibrium chord assigns each
iterate a progress coordinate \(\tau_k\), after which we compare its normalized energy-to-go with the bridge reference:
\begin{equation}
\begin{aligned}
\tau_k
&:=\operatorname{clip}\!\left(
\frac{\langle x^{(k)}-x^{(0)},v\rangle}{D_v},0,1
\right),\\
r_k^{\Delta E}
&:=\frac{E_\theta(\hat x,z)-E_\theta(x^{(k)},z)}{D_v}-C(\tau_k),\\
U_{\Delta E}(z)
&:=\frac{1}{K}\sum_{k=0}^{K-1}\left|r_k^{\Delta E}(z)\right|.
\end{aligned}
\label{eq:energy-consistency-score}
\end{equation}

The score is small when realized energy progress follows the bridge relation and increases when that relation is violated. Absolute aggregation prevents cancellation across refinement states, while averaging makes the scale comparable across different values of \(K\). It requires neither an expert action nor an auxiliary uncertainty model at deployment. Appendix~\ref{app:energy-consistency} relates the residual to field error and off-chord deviation.

\vspace{-3pt}
\subsection{Energy-Gated DAgger for Lifelong Robot Learning}\label{competence-aware-data-acquisition}
\vspace{-3pt}
We use EqC's energy-consistency score to gate a DAgger-style data aggregation loop~\citep{ross2011reduction}, an instantiation we call
Energy-Gated DAgger. Before each action prefix, the controller computes \(U_{\Delta E}(z_j)\); \(m\) consecutive
threshold crossings trigger expert takeover, and the recovery trajectory is retained for the next policy update. The
gate therefore responds to persistent path-wise inconsistency rather than raw energy magnitude.

\begin{algorithm}[t]
  \caption{Energy-Gated DAgger with EqC}
  \label{alg:eqc-lifelong}
  \footnotesize
  \begin{algorithmic}[1]
    \Require Initial demonstrations \(\mathcal D_0\), threshold \(\gamma\), persistence \(m\), rounds \(R\)
    \State \(\theta_0 \gets \Call{TrainEqC}{\mathcal D_0}\)
    \For{\(r=0,\ldots,R-1\)}
      \State \(\mathcal Q_r\gets\varnothing\), \(h\gets 0\)
      \For{each replan \(j\) in deployment round \(r\)}
        \State \(z_j\gets\Call{Observe}{}\)
        \State \((\hat x_j,\mathcal X_j)\gets\Call{EnergyRefine}{E_{\theta_r},z_j}\)
        \State \(u_j\gets\Call{EnergyConsistency}{\mathcal X_j}\)
        \State \(h\gets(h+1)\,\mathbf 1[u_j>\gamma]\)
        \If{\(h=m\)}
          \State \(\xi_j^\star\gets\Call{ExpertRecover}{z_j}\)
          \State \(\mathcal Q_r\gets\mathcal Q_r\cup\{\xi_j^\star\}\), \(h\gets0\)
        \Else
          \State \(\Call{ExecutePrefix}{\hat x_j}\)
        \EndIf
      \EndFor
      \State \(\mathcal D_{r+1}\gets\mathcal D_r\cup\mathcal Q_r\)
      \State \(\theta_{r+1}\gets\Call{TrainEqC}{\mathcal D_{r+1}}\)
    \EndFor
    \State \Return \(E_{\theta_R}\)
  \end{algorithmic}
\end{algorithm}

This deployment loop concentrates supervision on states where the current controller lacks support: the energy used to
generate behavior also determines when autonomy should yield, focusing new experience near the policy's competence
boundary.

\section{Why EqC Fits Generative Robot Policies}\label{analysis-and-core-insights}

\paragraph{A favorable generation regime.}
Image synthesis must cover diverse pixel realizations left unspecified by a class or text prompt. In control, rich
visual, language, and proprioceptive conditions narrow the action choice, success requires one competent action, and
receding-horizon execution replans after a short prefix. Action outputs are also much smaller than images. These
properties make conditional action generation a natural setting for a structured optimizer. This view agrees with the
finding that noise-covered supervised iteration explains more of generative-control performance than distinct action
modes on common benchmarks~\citep{pan2026much}, and suggests why the explicit-energy cost observed for EqM image
generation~\citep{wang2025equilibrium} need not transfer directly to control.

\paragraph{Approximation and optimization.}
Appendix~\ref{app:stationary-representation} compares EqC with matched stationary and time-conditioned regressors and
shows
\begin{equation}
R_{\mathrm{EqC}}^\Theta-R_{\mathrm{tc}}^\star
\leq
\left(\sigma_{\mathrm{act}}+\sigma_{\mathrm{scale}}+\delta_\Theta\right)^2.
\label{eq:concentration-summary}
\end{equation}
The terms measure residual action ambiguity, uncertainty about bridge progress after removing time, and finite energy
class approximation. When they are small, explicit energy has limited regression cost. It then exposes a property not
provided by an unconstrained field: sufficiently small gradient-ascent steps provably increase the scalar optimized at
inference. The same path calculus relates \(U_{\Delta E}\) to mean absolute integrated field error and trajectory deviation. Full
definitions and proofs are given in Appendix~\ref{app:theory}.

\FloatBarrier
\vspace{-5pt}
\section{Experiments}\label{experiments}
\vspace{-5pt}
The experiments evaluate the three capabilities developed by EqC. We first study its effectiveness as a generative controller across DiT and VLA policies in simulation and real-world manipulation. We then test path-wise energy consistency as an introspective signal for online failure detection. Finally, we close the loop with Energy-Gated DAgger to evaluate lifelong improvement on a physical robot. Complete benchmark and implementation protocols, failure-detection analyses, real-robot details, and ablations are provided in Appendices~\ref{app:evaluation-details},~\ref{app:competence-awareness-details},~\ref{app:real-robot-details}, and~\ref{app:ablations}, respectively.

\vspace{-3pt}
\subsection{Simulated and Real-World Benchmarks}\label{simulation-and-real-world-benchmarks}
\vspace{-3pt}
\subsubsection{Setup}\label{experimental-questions}
\vspace{-3pt}

We evaluate action generation and competence awareness on LIBERO~\citep{liu2023libero}, RoboTwin
2.0~\citep{chen2025robotwin}, VLABench~\citep{zhang2025vlabench}, and
LIBERO-Plus~\citep{fei2025libero}. All methods use each benchmark's binary success criterion. The DiT
experiments share one conditioning stack, action interface, and demonstration set across objectives and are built on a
modular diffusion-policy implementation~\citep{dong2024cleandiffuser}. The VLA
experiments use the same flow-matching VLA backbone within each comparison. EqC always refines a noise-initialized
action chunk and executes a prefix under receding-horizon control. Appendix~\ref{app:evaluation-details} gives the
benchmark definitions, rollout budgets, and implementation details.

\subsubsection{Diffusion Transformer}\label{libero-dit}

We adapt EqC to a DiT architecture and evaluate it on the four LIBERO suites. Table~\ref{tab:libero-dit} follows the
benchmark comparison reported by GeCO~\citep{zhang2026generative} and places EqC's multi-seed mean alongside those
published baselines. EqC achieves the highest aggregate result, with \(94.0\%\) average success across training seeds.
Appendix~\ref{app:dit-multiseed} reports the complete EqC statistics together with our independent multi-seed
reproductions of the matched rectified-flow and GeCO baselines; Appendix~\ref{app:benchmarks} describes the protocol.

\begin{table}[t]
  \centering
  \caption{Success rates (\%) on LIBERO with the DiT policy and VLABench with VLA policies.}
  \label{tab:compact-simulation-results}
  \begin{subtable}[t]{0.455\textwidth}
    \centering
    \phantomsubcaption\label{tab:libero-dit}
    \paneltitle{(a) LIBERO--DiT}\par\vspace{2pt}
    \footnotesize
    \setlength{\tabcolsep}{1.8pt}
    \renewcommand{\arraystretch}{1.08}
    \begin{tabular}{@{}lccccc@{}}
      \toprule
      \headrow Method & Goal & Spatial & Object & Long & Avg. \\
      \midrule
      Diffusion Policy & 82.9 & 91.4 & 88.9 & 82.7 & 86.5 \\
      Rectified Flow & 92.4 & 94.6 & 97.0 & 76.0 & 90.0 \\
      GeCO & 95.2 & 95.8 & \best{99.0} & 83.8 & 93.5 \\
      \multicolumn{6}{c}{} \\
      \eecrow \eqcmethod & \best{95.6} & \best{96.4} & 98.1 & \best{85.5} & \best{94.0} \\
      \bottomrule
    \end{tabular}
  \end{subtable}
  \hfill
  \begin{subtable}[t]{0.525\textwidth}
    \centering
    \phantomsubcaption\label{tab:vlabench}
    \paneltitle{(b) VLABench--VLA}\par\vspace{2pt}
    \footnotesize
    \setlength{\tabcolsep}{2.0pt}
    \renewcommand{\arraystretch}{1.08}
    \begin{tabular}{@{}lcccccc@{}}
      \toprule
      \headrow Method & T1 & T2 & T3 & T4 & T6 & Avg. \\
      \midrule
      $\pi_{0.5}$ & 40.6 & \best{22.6} & 18.0 & \best{16.1} & 25.6 & 24.5 \\
      \eecrow $\pi_{0.5}$ + \eqcmethod
        & \best{44.3} & 21.3 & \best{18.5} & 15.4 & \best{36.0} & \best{27.1} \\
      $\pi_0$ & 47.0 & 21.2 & 29.1 & 17.3 & 32.2 & 29.4 \\
      $\pi_0$ + GeCO & 61.0 & 22.0 & \best{34.0} & \best{20.0} & 45.0 & 36.0 \\
      \eecrow $\pi_0$ + \eqcmethod
        & \best{64.4} & \best{29.0} & 28.6 & 15.8 & \best{45.6} & \best{36.7} \\
      \bottomrule
    \end{tabular}
  \end{subtable}
  \vspace{-5pt}
\end{table}

\subsubsection{Vision-Language-Action Policies}\label{robotwin-and-vlabench}

We integrate EqC into a VLA backbone by adapting the $\pi_0$ policy while preserving its multimodal architecture. EqC
replaces flow matching with the energy objective $E_\theta(x,o) = \langle x, v_\theta(x,o,0) \rangle$, evaluated at the
fixed model time $t_{\mathrm{model}}=0$. Training matches energy gradients along interpolated action trajectories and
masks inactive padded dimensions; inference refines actions iteratively with the learned energy field. We evaluate the
resulting VLA policies on RoboTwin 2.0, VLABench, and LIBERO-Plus, with benchmark protocols in
Appendix~\ref{app:benchmarks}; supplementary LIBERO results are provided in Appendix~\ref{app:vla-libero}.

\begin{table}[t]
  \centering
  \caption{VLA policy success rates (\%) on RoboTwin 2.0 (Easy / Hard) and LIBERO-Plus.}
  \label{tab:vla-simulation-results}
  \begin{subtable}[t]{0.45\textwidth}
    \centering
    \phantomsubcaption\label{tab:robotwin}
    \paneltitle{(a) RoboTwin 2.0 (Easy / Hard)}\par\vspace{2pt}
    \scriptsize
    \setlength{\tabcolsep}{1.4pt}
    \renewcommand{\arraystretch}{1.08}
    \begin{tabular}{@{}l cccccc@{}}
      \toprule
      \headrow Method & Bottle & Hammer & Rank & Alarm & Bell & Avg. \\
      \midrule
      $\pi_0$         & \eh{90}{56}  & \eh{43}{21} & \eh{\best{7}}{\best{1}} & \eh{63}{11} & \eh{44}{3}  & \eh{49}{18} \\
      $\pi_0$ + GeCO  & \eh{88}{66}  & \eh{52}{\best{36}} & \eh{0}{0}          & \eh{72}{16} & \eh{42}{21} & \eh{50}{28} \\
      \eecrow $\pi_0$ + \eqcmethod 
        & \eh{\best{100}}{\best{76}} & \eh{\best{76}}{27} & \eh{6}{0} & \eh{\best{83}}{\best{32}} & \eh{\best{49}}{\best{30}} & \eh{\best{63}}{\best{33}} \\
      \bottomrule
    \end{tabular}
  \end{subtable}
  \hfill
  \begin{subtable}[t]{0.53\textwidth}
    \centering
    \phantomsubcaption\label{tab:libero-plus}
    \paneltitle{(b) LIBERO-Plus}\par\vspace{2pt}
    \scriptsize
    \setlength{\tabcolsep}{1.3pt}
    \renewcommand{\arraystretch}{1.08}
    \begin{tabular}{@{}l ccccccccc@{}}
      \toprule
      \headrow Method & Orig & Cam & Rob & Lang & Light & Back & Noise & Lay & Avg. \\
      \midrule
      OpenVLA & 76.5 & 0.8  & 3.5  & 23.0 & 8.1  & 34.8 & 15.2 & 28.5 & 15.6 \\
      $\pi_0$ & 94.2 & 15.8 & 6.6  & \best{61.0} & 79.6 & 78.5 & \best{79.4} & 70.4 & 53.6 \\
      \eecrow $\pi_0$ + \eqcmethod 
        & \best{94.3} & \best{61.7} & \best{39.0} & 54.8 & \best{84.6} & \best{83.0} & 74.3 & \best{73.1} & \best{65.8} \\
      \bottomrule
    \end{tabular}
  \end{subtable}
  \vspace{-5pt}
\end{table}

\vspace{-2pt}
\looseness=-1
\paragraph{RoboTwin 2.0.} RoboTwin 2.0 evaluates bimanual manipulation under structured randomization of clutter,
lighting, background, tabletop height, and language instructions. Under its clean Easy and domain-randomized Hard
conditions, Table~\ref{tab:robotwin} shows that EqC raises the \(\pi_0\) averages from \(49/18\) to \(63/33\), compared
with \(50/28\) for \(\pi_0\)+GeCO. The improvement is particularly pronounced on Adjust Bottle and Click Alarmclock,
including under the more challenging Hard condition.
\vspace{-3pt}
\looseness=-1
\paragraph{VLABench.} We evaluate on VLABench across its five standardized tracks (T1--T4 and T6), covering object generalization, commonsense reasoning, semantics, and appearance shifts. As shown in Table~\ref{tab:vlabench}, EqC consistently enhances generalization for both flow-matching backbones, boosting $\pi_0$ from $29.4\%$ to $36.7\%$ and $\pi_{0.5}$ from $24.5\%$ to $27.1\%$.
\vspace{-3pt}
\looseness=-1
\paragraph{LIBERO-Plus.} LIBERO-Plus tests robustness under seven controlled perturbation families. Table~\ref{tab:libero-plus} shows that EqC raises overall success from $53.6\%$ to $65.8\%$. The most prominent improvements occur under camera viewpoint ($15.8\% \to 61.7\%$) and robot initialization shifts ($6.6\% \to 39.0\%$), demonstrating superior out-of-distribution robustness.

\Needspace{0.48\textheight}
\subsubsection{Real-World Experiments}\label{real-world-experiment}

\begin{wrapfigure}{r}{0.44\textwidth}
  \vspace{-6pt}
  \centering
  \includegraphics[width=\linewidth]{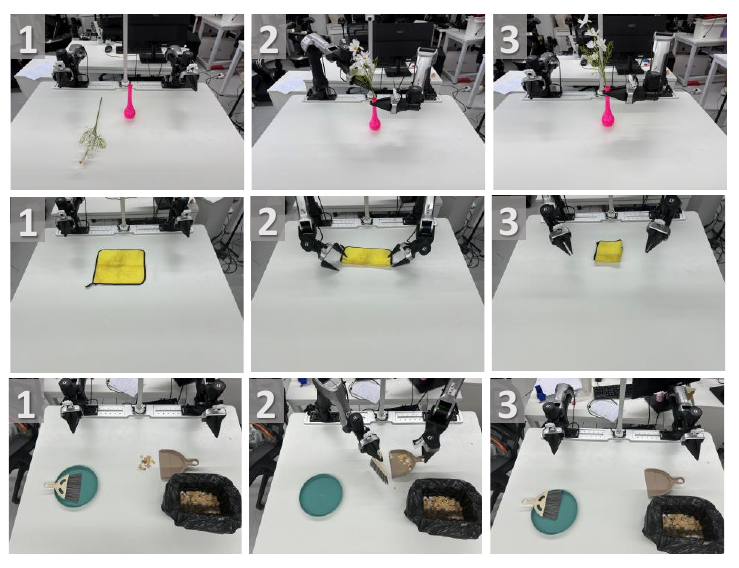} 
  \vspace{-6pt}
  \caption{\textbf{Real-world bimanual execution.} 
  \textit{Top (1--3):} Flower Insertion. 
  \textit{Middle (1--3):} Towel Folding. 
  \textit{Bottom (1--3):} Debris Sweeping.}
  \label{fig:real-world-benchmark}
  \vspace{-10pt}
\end{wrapfigure}
We evaluate EqC on a bimanual workstation with two Galaxea A1Z manipulators using the Galaxea G0.5~\citep{liu2026g0} flow-matching based VLA policy, with 20 independent trials per task under randomized object placements. Flower Insertion (Figure~\ref{fig:real-world-benchmark}, top) requires the robot to grasp and reorient a flexible stem before inserting it through a narrow vase opening; Towel Folding (middle) requires coordinated deformable-object manipulation through lateral folding, regrasping, and a final vertical fold. Debris Sweeping (bottom) evaluates multi-stage bimanual tool coordination, requiring the robot to sweep scattered paper scraps into a handheld dustpan using a brush and deposit them cleanly into a waste bin without spillage. Appendix~\ref{app:real-robot-details} provides task definitions, success criteria, and training settings.

\begin{table}[t]
  \centering
  \caption{Control and competence-awareness evaluation. \textbf{(a)} Success and latency on three real-world bimanual tasks (20 trials each). \textbf{(b)} Failure detection on pooled LIBERO and LIBERO-Plus rollouts. ``EqC+Diff-DAgger'' denotes the Diff-DAgger flow-loss detector applied to the EqC backbone; EqC uses path-wise energy consistency.}
  \label{tab:control-and-awareness-results}
  \begin{subtable}[t]{0.595\textwidth}
    \centering
    \phantomsubcaption\label{tab:real-world-results}
    \paneltitle{(a) Real-world bimanual manipulation}\par\vspace{2pt}
    \fontsize{6.4pt}{7.6pt}\selectfont
    \setlength{\tabcolsep}{1.1pt}
    \renewcommand{\arraystretch}{1.10}
    \begin{tabular}{@{}l cccc cccc@{}}
      \toprule
      \headrow & \multicolumn{4}{c}{Success (\%)} & \multicolumn{4}{c}{Latency (ms)} \\
      \cmidrule(lr){2-5} \cmidrule(lr){6-9}
      \headrow \multirow{-2}{*}{Method} & Flower & Towel & Sweep. & Avg. & Flower & Towel & Sweep. & Avg. \\
      \midrule
      Galaxea G0.5 & 45.0 & 80.0 & 65.0 & 63.3 & \best{421.9} & \best{435.6} & \best{439.6} & \best{432.4} \\
      \addlinespace[1pt]
      \eecrow Galaxea G0.5 + \eqcmethod
        & \best{80.0} & \best{95.0} & \best{100.0} & \best{91.7} & 526.4 & 556.5 & 548.7 & 543.9 \\
      \bottomrule
    \end{tabular}
  \end{subtable}
  \hfill
  \begin{subtable}[t]{0.395\textwidth}
    \centering
    \phantomsubcaption\label{tab:failure-detection}
    \paneltitle{(b) Online competence awareness}\par\vspace{2pt}
    \fontsize{6.4pt}{7.6pt}\selectfont
    \setlength{\tabcolsep}{1.1pt}
    \renewcommand{\arraystretch}{1.10}
    \begin{tabular}{@{}lcccc@{}}
      \toprule
      \headrow Method & AUROC & TPR/FPR & Acc. (\%) & Saved (\%) \\
      \midrule
      Diff-DAgger & 0.8430 & 61.3/10.2 & 82.2 & 33.3 \\
      EqC+Diff-DAgger & 0.8888 & 72.8/8.8 & 85.8 & 42.1 \\
      \addlinespace[1pt]
      \eecrow \eqcmethod & \best{0.9244} & \best{77.9/8.5} & \best{87.5} & \best{45.5} \\
      \bottomrule
    \end{tabular}
  \end{subtable}
\end{table}
Table~\ref{tab:real-world-results} shows that EqC raises average success from 63.3\% to 91.7\%. Success improves from 45.0\% to 80.0\% on \textit{Flower Insertion}, from 80.0\% to 95.0\% on \textit{Towel Folding}, and reaches 100.0\% on \textit{Debris Sweeping}, while introducing moderate inference latency overhead. This additional latency can be effectively hidden through Real-Time Chunking (RTC)~\citep{black2025training}(Appendix~\ref{app:rtc}).

\subsection{Competence Awareness}\label{competence-evaluation-and-ablations}
\label{sec:competence-eval}
\label{sec:ablations}

\subsubsection{Setup}\label{competence-awareness-setup}

EqC's energy-consistency signal measures whether inferred energy progress remains aligned with the relation induced by
expert demonstrations, enabling the policy to perform introspective competence assessment at inference time. We
evaluate this property as online failure detection: successful rollouts are negative examples and failed rollouts are
positive examples. An alarm requires the signal
to exceed a fixed threshold for \(m\) consecutive replanning steps, and detection timing is measured by the fraction of
the rollout remaining when an alarm is raised in a failed trial. 
Appendix~\ref{app:failure-detection-definition} formalizes the competence-aware task definition, while
Appendix~\ref{app:competence-awareness-split} provides the simulation failure-detection protocol, including the rollout partition, Receiver Operating Characteristic (ROC) curve construction, and threshold calibration.

\subsubsection{Simulated Rollout Analysis}\label{simulated-rollout-analysis}

In simulation, we evaluate online competence awareness on LIBERO and LIBERO-Plus using three configurations: Diff-DAgger's flow-loss detector on $\pi_0$~\citep{lee2025diff}, the same detector on the EqC backbone, and the full energy-consistency signal $U_{\Delta E}$. We pool rollouts from eight evaluation conditions:
the original LIBERO benchmark and the seven controlled perturbation families of LIBERO in LIBERO-Plus, which preserve the underlying LIBERO task semantics. Threshold-dependent metrics use a disjoint split:
10\% of the rollouts select and freeze \((m,\tau^*)\), and the remaining 90\% are used only for evaluation. Under this protocol, the calibration subset selects $(m,\tau^*)$ while constraining the FPR to approximately $10\%$; the pair remains frozen for held-out evaluation, while Area Under the Receiver Operating Characteristic (AUROC) summarizes the full score ranking. Table~\ref{tab:failure-detection} shows that EqC's path-wise energy-consistency signal attains the highest AUROC, TPR, accuracy, and saved inference time, together with the lowest FPR, among Diff-DAgger and EqC + Diff-DAgger baseline; its AUROC reaches $0.9244$, compared with $0.8430$ for Diff-DAgger. Appendix~\ref{app:competence-awareness-diffdagger-baseline-method} and ~\ref{app:diff-dagger-implementation-on-the-eqc-trained-backbone} detail the adapted Diff-DAgger and EqC + Diff-DAgger baseline, respectively, and Appendix~\ref{app:competence-awareness-detailed-results} reports calibration variants, window and Monte Carlo sensitivity, saved inference time, and inference latency.

\subsubsection{Real-World Rollout Analysis}\label{real-world-rollout-analysis}

We next use \textit{Hierarchical Block Stacking} as a real-world failure-detection testbed and collect 50 rollouts: 25 successful executions and 25 failed executions. Figure~\ref{fig:real-world-failure-curve-eqc} shows that the signal remains low for successful executions and rises
early for failures, whereas the Diff-DAgger trajectories in
Figure~\ref{fig:real-world-failure-curve-diffdagger} overlap through most of the horizon. The EqC signal therefore
preserves a more useful distinction between supported and unsupported execution under physical sensing and control
noise; Appendix~\ref{app:real-robot-details} provides the task details.

\begin{figure*}[t]
  \vspace{-8pt}
  \centering
  \begin{subfigure}[t]{0.445\textwidth}
    \centering
    \includegraphics[width=\linewidth]{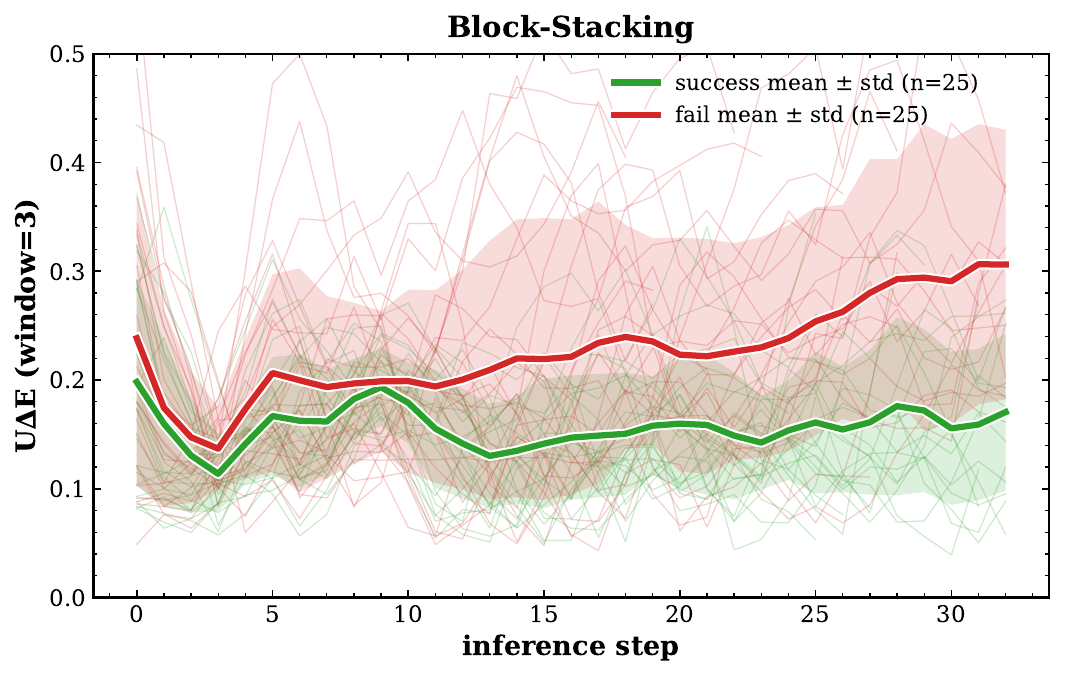}
    \caption{EqC (ours)}
    \label{fig:real-world-failure-curve-eqc}
  \end{subfigure}
  \hfill
  \begin{subfigure}[t]{0.445\textwidth}
    \centering
    \includegraphics[width=\linewidth]{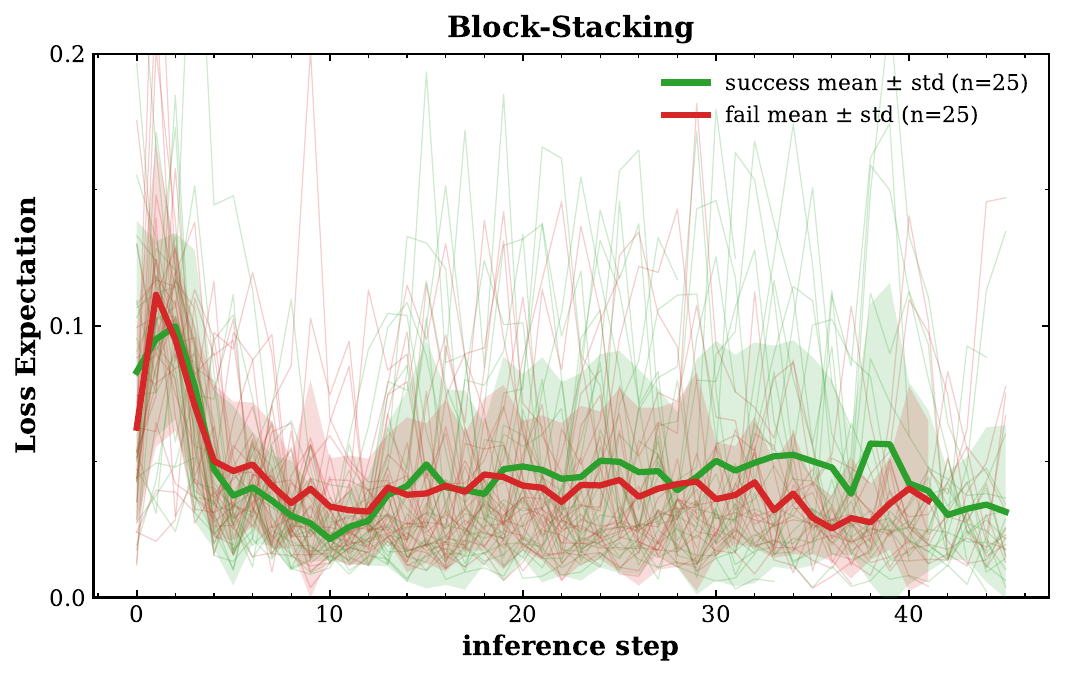}
    \caption{Diff-DAgger (baseline)}
    \label{fig:real-world-failure-curve-diffdagger}
  \end{subfigure}
  \caption{\textbf{Real-world failure detection} (mean $\pm 1\,\text{std}$, $n=25$). (a) EqC cleanly separates success and failure early for timely intervention, whereas (b) Diff-DAgger exhibits heavy overlap.}
  \label{fig:real-world-ood-curves-combined}
  \vspace{-1.0em}
\end{figure*}

\subsection{Energy-Gated DAgger for Lifelong Learning}\label{lifelong-learning}

\subsubsection{Experimental Setup and Protocol}\label{lifelong-learning-setup}
Real-world robot data are costly, and deployment conditions cannot be fully anticipated. We therefore test whether EqC recognizes its competence boundary and uses intervention data to expand it across two bimanual tasks on the Galaxea platform: 
(1) \textit{Hierarchical Block Stacking}, requiring the robot to construct a freestanding tower across $3! = 6$ color permutations under tight contact constraints; and 
(2) \textit{Garment Folding}, requiring sequential multi-stage folding of deformable fabric. 
Both policies start from limited seed demonstrations and follow Algorithm~\ref{alg:eqc-lifelong} ($\gamma = 0.2$), where persistent $U_{\Delta E}$ violations prompt expert recovery. The aggregated trajectories are used for subsequent policy fine-tuning (Appendix~\ref{app:real-robot-details}).

\subsubsection{Progressive Performance}\label{lifelong-results}

Figure~\ref{fig:lifelong-curves} (top) tracks competence acquisition over Energy-Gated DAgger deployment rounds across \textit{Hierarchical Block Stacking} and \textit{Garment Folding}. We evaluate against an automated Galaxea G0.5 + Diff-DAgger baseline and a skilled human-in-the-loop teleoperator. The human-intervention baseline uses the same EqC controller and initial checkpoint as our method,
differing only in the intervention trigger. Under identical data and training budgets, Energy-Gated DAgger substantially outperforms all baselines across three independent experimental runs. On \textit{Hierarchical Block Stacking}, success rises from an initial $0.0\%$ to an average of \best{96.7\%} after incorporating 30 targeted recovery episodes, outperforming Galaxea G0.5 + Diff-DAgger ($80.0\%$) and human-triggered intervention ($58.3\%$). 
On \textit{Garment Folding}, the policy progresses from $30.0\%$ to an average of \best{98.3\%} autonomous execution in just 2 rounds (12 targeted recovery episodes), surpassing Galaxea G0.5 + Diff-DAgger ($81.7\%$) and skilled human intervention ($55.0\%$). Full round-by-round statistics across all individual runs are reported in Table~\ref{tab:lifelong-detailed-stats} (Appendix~\ref{app:real-robot-lifelong-protocol}). These results confirm that EqC's path-wise energy-consistency signal identifies weakly supported states more reliably than Diff-DAgger's signal or manual monitoring, yielding targeted recovery data that systematically expands autonomous competence. As qualitatively shown in Figure~\ref{fig:real-robot-rollouts} (bottom), detected execution anomalies trigger timely $U_{\Delta E}$ alarms to prompt expert recovery in earlier rounds, progressively enabling robust autonomous execution under perturbed setups. Furthermore, under sequential cross-task shift, EqC-guided recovery outperforms matched cumulative replay across three independent runs. (Appendix~\ref{app:task-shift-cup-stacking}).

\begin{figure}[!htbp]
  \vspace{-6pt}
  \centering
  \includegraphics[width=0.76\linewidth]{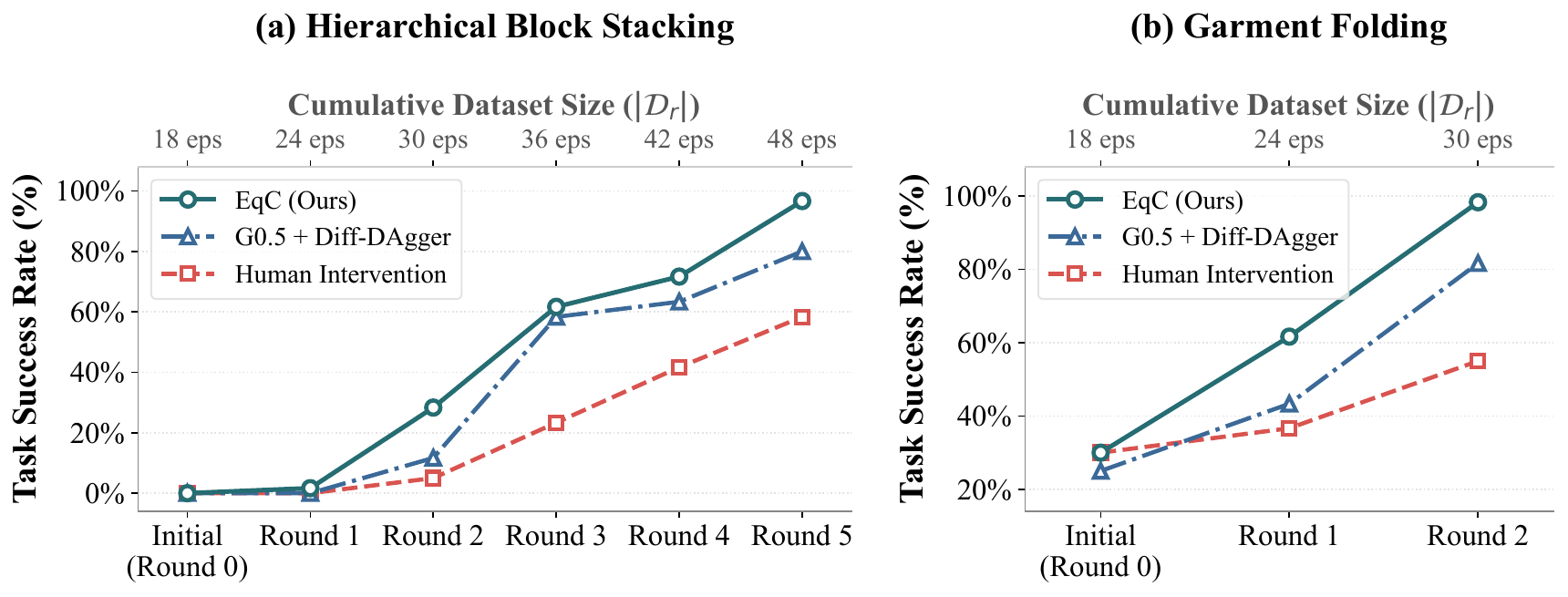}
  \vspace{2pt}
  \includegraphics[width=0.80\linewidth]{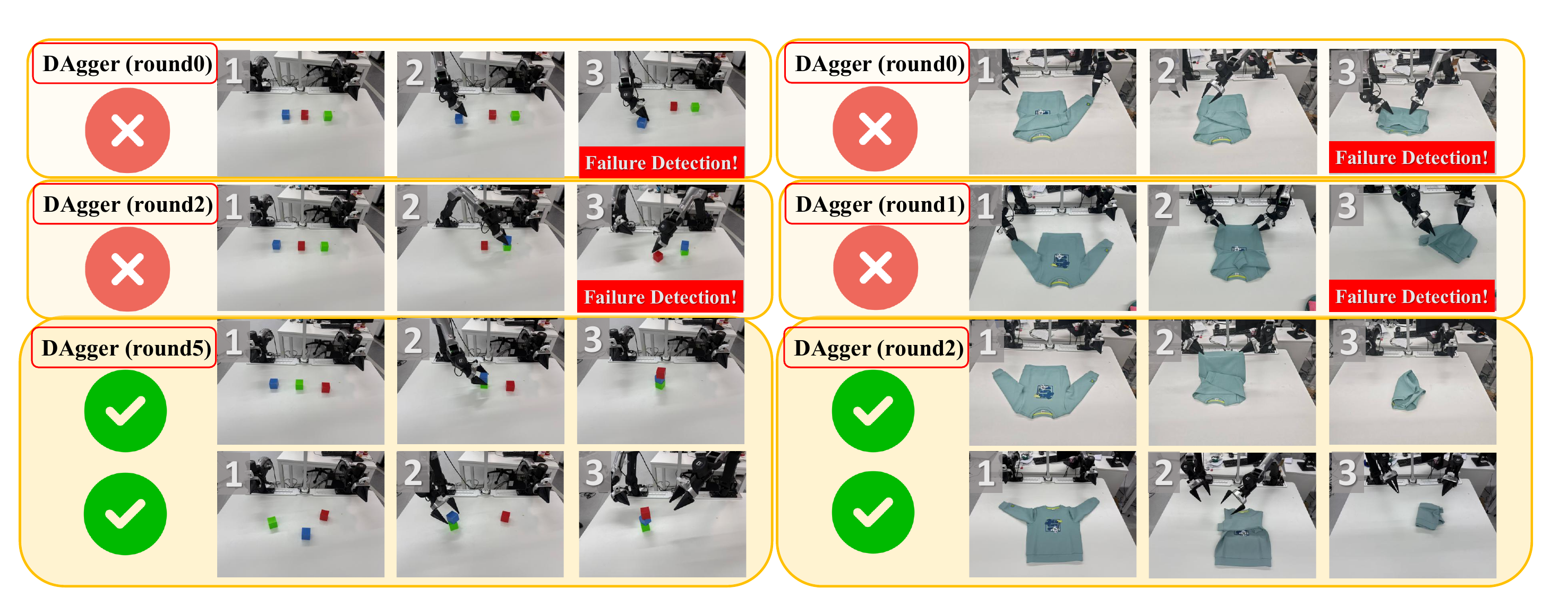}
  \vspace{-4pt}
  \caption{\textbf{Real-world lifelong learning evaluation with Energy-Gated DAgger.} 
    \textbf{Top:} Quantitative success curves across deployment rounds (averaged over 3 independent runs, with 20 autonomous evaluation trials per checkpoint per run). 
    \textbf{Bottom:} Qualitative rollouts: early drift triggers $U_{\Delta E}$ alarms (red banners) for expert recovery; fine-tuned policies achieve autonomous execution (Rows 3--4). 
    Across both \textbf{(a)}~\textit{Hierarchical Block Stacking} and \textbf{(b)}~\textit{Garment Folding}, EqC achieves \best{96.7\%} and \best{98.3\%} mean success rates, respectively, decisively outperforming Diff-DAgger and human intervention.}
  \label{fig:lifelong-combined}
  \label{fig:lifelong-curves}
  \label{fig:real-robot-rollouts}
  \vspace{-8pt}
\end{figure}

\FloatBarrier

\subsubsection{Agent-Powered Expert Takeover}
\label{lifelong-agent-expert}

\begin{wrapfigure}{R}{0.38\textwidth}
  \vspace{-12pt}
  \centering
  \includegraphics[width=\linewidth]{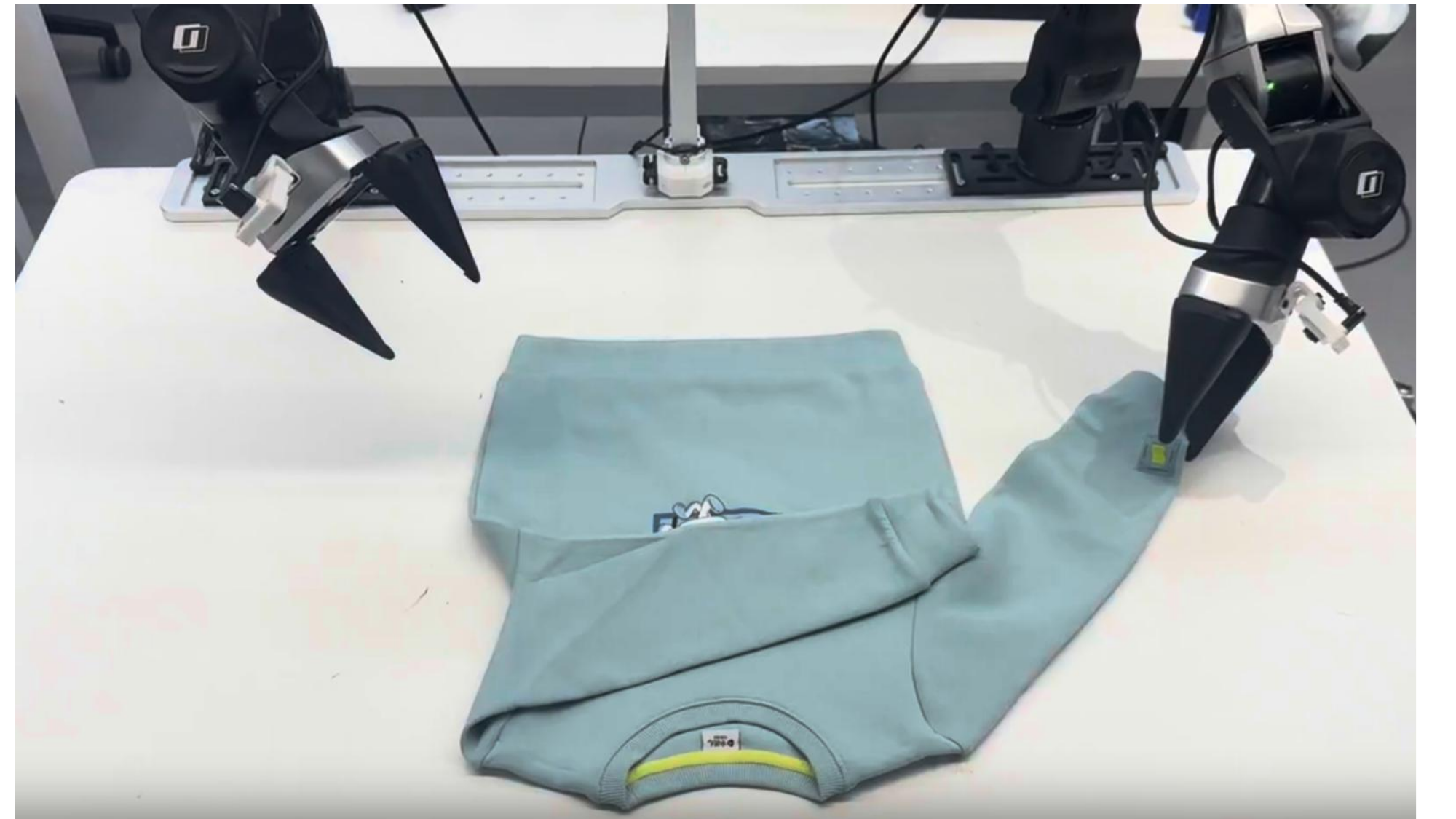}
  \vspace{-6pt}
  \caption{\textbf{Agent-powered expert takeover.}
   Astra takes over from an intermediate Garment Folding
  state following an EqC energy-consistency alarm.}
  \label{fig:agent-powered-takeover}
  \vspace{-3pt}
\end{wrapfigure}

To further explore automated lifelong robot learning, we
extend Energy-Gated DAgger by employing a robot agent
powered by \texttt{gpt-6-astra} as the recovery expert.
This experiment investigates whether EqC can autonomously
identify when assistance is needed, obtain targeted
recovery demonstrations from an agent, and
improve its performance.

\textbf{Agent preparation.}
We evaluate this setting on real-world
\textit{Garment Folding}, starting from an EqC policy
trained on 18 initial demonstrations.
The agent is prepared through two stages.
First, Astra analyzes the original demonstrations, constructs an executable robot-control
script, and iteratively refines it through full-task
practice on the physical robot.
Its execution experience is consolidated into reusable
skills. Second, takeover-specific practice teaches
the agent to adapt these skills to intermediate robot
and garment configurations encountered when EqC
requests assistance.
After preparation, the agent can adapt its control
script to each intervention state and execute the
remaining task.

\textbf{Automated takeover and policy improvement.}
During deployment, EqC continuously monitors
$U_{\Delta E}$ using the same intervention threshold and persistence criterion as in the
preceding experiments.
When a persistent energy-consistency alarm is triggered,
EqC suspends autonomous execution, and Astra takes
over from the current state to complete the task. 
Figure~\ref{fig:agent-powered-takeover} illustrates
the intermediate garment configuration at which
Astra begins expert takeover.
The resulting recovery trajectory is recorded as
a new expert demonstration.
After collecting only six agent-generated recovery
demonstrations, we combine them with the original
18 demonstrations and fine-tune EqC for one round.
The updated policy improves its autonomous success
rate from $30\%$ to $100\%$, successfully completing
all 20 evaluation trials.

\textbf{Implications.}
This proof-of-concept demonstrates the feasibility
of automating both intervention decisions and expert
recovery after agent preparation.
EqC identifies situations requiring assistance,
while the prepared agent supplies targeted recovery
experience that is subsequently incorporated into
the learned policy. The resulting framework offers a practical direction for extending competence-aware intervention toward
automated lifelong robot learning.
Additional implementation details, agent preparation
procedures, and the complete takeover algorithm are
provided in Appendix~\ref{app:robot-agent-as-expert}.
\par

\section{Limitations and Scope}\label{limitations-and-scope}

Our study focuses on continuous action-chunk policies with stationary refinement schedules. Future work should extend
EqC to hybrid, memory-based, and hierarchical policies; deepen the analysis of energy geometry, calibration, and
repeated data aggregation; and test how competence thresholds and intervention budgets transfer across embodiments
and longer deployments.

\section{Conclusion}\label{conclusion}

We presented Equilibrium Energy Control (EqC), an energy-based framework connecting generative robot control with
introspective competence awareness. Through an explicit conditional energy landscape, EqC refines actions toward
equilibrium and uses path-wise energy consistency to assess whether behavior remains supported; Energy-Gated DAgger
turns this introspective signal into targeted expert recovery and policy updates. Across simulation and physical robots,
EqC matches or improves control performance and robustness. Real-world lifelong learning experiments further show
that repeated recovery and updates expand autonomous competence. Its compatibility with diverse
robot-manipulation architectures also offers a path toward energy-based control and introspection in locomotion,
navigation, and autonomous driving. Embedding competence awareness into generative control may enable introspective,
self-evolving embodied agents whose autonomy grows through lifelong learning while retaining a mechanism for safe
intervention.

\clearpage

\subsection*{Reproducibility statement}

We provide the EqC formulation and theoretical derivations in the main text and Appendix~\ref{app:theory}. The appendices document model and training configurations, benchmark protocols, evaluation splits, failure-detection
calibration, real-robot tasks, and baseline implementations (Appendices~\ref{app:evaluation-details}, \ref{app:competence-awareness-details}, and~\ref{app:real-robot-details}). The code for the EqC implementations and evaluation procedures is available at the repository linked in the Introduction and Appendix~\ref{app:code-availability}.

\bibliographystyle{iclr2027_conference}
\bibliography{iclr2027_conference}

\appendix

\section{Evaluation Details}\label{app:evaluation-details}

\begin{table}[ht]
  \caption{Evaluation coverage and rollout budgets.}
  \label{tab:protocols}
  \centering
  \tabstyle
  \setlength{\tabcolsep}{4pt}
  \begin{tabularx}{\columnwidth}{@{}c>{\centering\arraybackslash}Xc@{}}
    \toprule
    \headrow Benchmark & Coverage & Episodes \\
    \midrule
    LIBERO--DiT & 4 suites; 40 tasks & 50/task \\
    LIBERO--VLA & 4 suites; 40 tasks & 50/task \\
    RoboTwin 2.0 & 5 tasks; Easy and Hard & 100/task/setting \\
    VLABench & Tracks 1, 2, 3, 4, and 6; 50 tasks & 50/task \\
    LIBERO-Plus & 7 perturbation families & 10030 total \\
    \bottomrule
  \end{tabularx}
\end{table}
\FloatBarrier

\subsection{Implementation Protocol}\label{app:implementation-protocol}

Both implementations use the convention \(t=0\) for noise and \(t=1\) for data, and query the stationary backbone at model time zero. The scalar energy is reduced independently for each sample over active action coordinates, \(E_\theta(x,z)=\sum_{i\in\mathcal A}x_i f_{\theta,i}(x,z)\); the VLA mask excludes padded embodiment dimensions. Training differentiates this scalar with respect to the complete action chunk and retains the higher-order graph needed to update \(\theta\). Inference performs gradient ascent and detaches the action between refinement steps. The reported models use \(c(t)=\min\{4,20(1-t)\}\), with \(C(t)\) evaluated analytically; the analysis allows any schedule satisfying the stated integrability conditions.

\subsection{Benchmarks}\label{app:benchmarks}

\paragraph{LIBERO.}
We evaluate the DiT policy on the LIBERO simulation benchmark,
which contains four task suites---Goal, Spatial, Object, and Long---with
ten tasks per suite. These suites evaluate different aspects of
language-conditioned robot manipulation, including changes in task goals,
spatial configurations, object categories, and long-horizon compositional
tasks. We use the shared DiT architecture and data pipeline
described in Section~\ref{experiments}.

\paragraph{RoboTwin 2.0.}
We evaluate the VLA policy on RoboTwin 2.0, a benchmark for bimanual
robot manipulation under both clean and domain-randomized conditions.
We consider five representative tasks: Adjust Bottle, Beat Block Hammer,
Blocks Ranking (Size), Click Alarmclock, and Click Bell. Each task is
evaluated under the Easy and Hard settings. The Easy setting uses clean
environments, whereas the Hard setting introduces additional domain
randomization. We report task-wise success rates and average performance
across the five tasks for both settings. All results are averaged over
100 evaluation episodes per task and setting using the benchmark's
success criteria.

\paragraph{VLABench.}
We follow the current official VLA evaluation protocol and report the five tracks with released standardized episode
configurations: in-distribution task learning (T1), cross-category and object generalization (T2), commonsense
understanding (T3), semantic-instruction understanding (T4), and unseen background or table textures (T6). Each task is
evaluated over 50 episodes, and we report track-wise success rates and their average. Cross-task skill transfer (T5) is
part of the broader VLABench design, but the official protocol defines it as an open setting in which users choose their
own training and evaluation tasks; it does not provide a fixed task split or standardized episode configuration.
We therefore omit T5 and use the publicly reproducible five-track configuration rather than introducing a
study-specific cross-task split.

\paragraph{LIBERO-Plus.}
We evaluate robustness under controlled distribution shifts using
LIBERO-Plus, which extends the original LIBERO tasks with perturbations
that preserve the underlying task semantics. We consider seven
perturbation families: camera viewpoint, robot initialization, language,
illumination, background, sensor noise, and object layout. We report
success rates for each perturbation family and the corresponding
robustness aggregate. 

\subsection{DiT Details}
\label{app:dit-training-details}

For LIBERO experiments with the Diffusion Transformer (DiT), we encode visual observations using a pretrained DINOv2-base (\texttt{dinov2-base}) model and language instructions using T5-base (\texttt{t5-base}), with an observation horizon of $T_o = 1$. The policy predicts continuous action chunks with horizon $T_a = 16$ and $D_a = 7$ control dimensions ($D = 112$), executing the first 8 steps under receding-horizon control.

\paragraph{Training.}
We train all DiT models for 70k steps with a batch size of 256 in \texttt{bf16-mixed} precision. Training uses the AdamW optimizer with cosine learning-rate decay, a peak learning rate of $1\times 10^{-4}$, 5,000 linear warmup steps, and weight decay of $1\times 10^{-4}$.

\subsection{Multi-Seed DiT Results}\label{app:dit-multiseed}

Table~\ref{tab:dit-multiseed-detail} reports the LIBERO results obtained across multiple EqC training seeds. The final
row gives the mean and sample standard deviation for each suite. Unlike the published baseline values retained in the
main comparison, Table~\ref{tab:dit-multiseed-summary} reports our independent multi-seed reproductions of the matched
rectified-flow and GeCO objectives alongside EqC, without listing individual baseline runs.

\begin{table}[ht]
  \caption{Detailed EqC success rates (\%) across multiple training seeds on LIBERO.}
  \label{tab:dit-multiseed-detail}
  \centering
  \small
  \setlength{\tabcolsep}{5pt}
  \renewcommand{\arraystretch}{1.10}
  \begin{tabular}{@{}lccccc@{}}
    \toprule
    \headrow Seed & Goal & Spatial & Object & Long & Avg. \\
    \midrule
    0 & 95.6 & 97.2 & 98.6 & 85.4 & 94.3 \\
    1 & 96.0 & 97.0 & 98.4 & 87.0 & 94.7 \\
    2 & 95.0 & 96.0 & 97.6 & 85.4 & 93.5 \\
    3 & 95.8 & 95.2 & 97.8 & 84.2 & 93.3 \\
    \midrule
    \eecrow Mean $\pm$ sample s.d.
      & $95.60 \pm 0.43$ & $96.35 \pm 0.93$ & $98.10 \pm 0.48$ & $85.50 \pm 1.15$ & $93.95 \pm 0.66$ \\
    \bottomrule
  \end{tabular}
\end{table}

\begin{table}[ht]
  \caption{Overall LIBERO success across multiple training seeds.}
  \label{tab:dit-multiseed-summary}
  \centering
  \small
  \setlength{\tabcolsep}{12pt}
  \renewcommand{\arraystretch}{1.10}
  \begin{tabular}{@{}lcc@{}}
    \toprule
    \headrow Method & Mean (\%) & Sample s.d. (\%) \\
    \midrule
    Fixed RF, 20 steps & 89.12 & 0.80 \\
    GeCO & 92.30 & 1.02 \\
    \eecrow \eqcmethod & \best{93.95} & \best{0.66} \\
    \bottomrule
  \end{tabular}
\end{table}

\subsection{VLA Training Details}
\label{app:vla-training-details}

We fine-tune $\pi_0$ and $\pi_{0.5}$ VLAs from OpenPI base checkpoints on LeRobot-format datasets using the AdamW optimizer with gradient clipping (\texttt{clip\_gradient\_norm=1.0}) and a cosine learning-rate schedule (10k warmup steps and a peak learning rate of $5\times 10^{-5}$). Models are trained for 50k steps on LIBERO and 30k steps on RoboTwin 2.0 and VLABench.

\paragraph{RoboTwin 2.0.}
We fine-tune a task-specific $\pi_0$ model for each of the five bimanual tasks (batch size 32, 30k steps). Inputs comprise three RGB camera views (\texttt{cam\_high}, \texttt{cam\_left\_wrist}, \texttt{cam\_right\_wrist}) and robot proprioception, predicting action chunks of horizon $H=50$ with 14 active control dimensions ($D = 700$).

\paragraph{LIBERO and LIBERO-Plus.}
For LIBERO, we fine-tune $\pi_0$ on the
\texttt{physical-intelligence/libero} dataset for 50k steps
with batch size 32. LIBERO-Plus extends the original LIBERO tasks through seven controlled perturbation families: camera viewpoint, robot initialization, language, illumination, background, sensor noise, and object layout. These perturbations preserve the underlying task semantics while introducing variations in observations and execution conditions. We follow the official LIBERO-Plus evaluation protocol and directly evaluate the LIBERO-trained checkpoint in a zero-shot manner across all seven perturbation families, without any benchmark-specific adaptation.

\paragraph{VLABench.}
For VLABench, models are fine-tuned for 30k steps on \texttt{vlabench/vlabench\_primitive\_ft\_lerobot} across five evaluation tracks, using a batch size of 256 for $\pi_0$ and 32 for $\pi_{0.5}$.

\subsection{Supplementary LIBERO Results for VLA Policies}\label{app:vla-libero}

We additionally evaluate EqC with VLA policies ($\pi_0$) on LIBERO to examine whether the formulation transfers to large multimodal action models. Table~\ref{tab:libero-vla-app} reports the full comparison across the four LIBERO task suites.

\begin{table}[ht]
  \caption{LIBERO success rates (\%) with the VLA backbone.}
  \label{tab:libero-vla-app}
  \centering
  \scriptsize
  \setlength{\tabcolsep}{5pt}
  \renewcommand{\arraystretch}{1.12}
  \begin{tabular}{@{}*{6}{c}@{}}
    \toprule
    \headrow Method & Goal & Spatial & Object & Long & Avg. \\
    \midrule
    $\pi_0$ + FAST~\citep{pertsch2025fast} & 88.6 & 96.4 & 96.8 & 60.2 & 85.5 \\
    $\pi_0$ & 95.8 & 96.8 & \best{98.8} & 85.2 & 94.2 \\
    $\pi_0$ + GeCO & \best{96.4} & 96.8 & 97.0 & 85.4 & 93.9 \\
    \addlinespace[1pt]
    \eecrow $\pi_0$ + \eqcmethod & 94.6 & \best{97.6} & \best{98.8} & \best{86.2} & \best{94.3} \\
    \bottomrule
  \end{tabular}
\end{table}

EqC achieves a 94.3\% aggregate success rate with the VLA policy, together with strong results on Spatial, Object, and
Long. Combined with the DiT experiments, these results demonstrate the applicability of the equilibrium-energy
formulation across both policy families. Sections~\ref{competence-evaluation-and-ablations}
and~\ref{lifelong-learning} further evaluate the introspective competence signal enabled by the VLA instantiation and
its use for lifelong learning.

\section{Competence Awareness Details}\label{app:competence-awareness-details}

\subsection{Task Definition: Online Failure Detection}\label{app:failure-detection-definition}

Our goal is online failure detection rather than detecting visual novelty or distribution shift. In robot manipulation, environmental shifts do not necessarily imply execution failure: an adaptable policy may handle unfamiliar scenes autonomously, where takeover is unnecessary and disruptive. Conversely, executions under nominal setups frequently fail due to contact slip, occlusions, or compounding control error. Intervention should therefore depend on whether the policy is actually failing the task, rather than whether the current observation looks unfamiliar.

Ground-truth labels are determined strictly by task outcomes: successful rollouts are labeled negative (autonomous completion, no intervention needed), while failed rollouts are labeled positive (intervention required). This formulation avoids penalizing successful generalization: standard novelty detectors flag unfamiliar yet competent executions as anomalies, prompting unnecessary human takeover. The reported metrics evaluate the ability to anticipate and catch execution failures rather than measuring observation density. Labels remain tied to final task success rather than the scene configuration.

\subsection{Simulation Failure Detection Protocol}
\label{app:competence-awareness-details-setup}
\label{app:competence-awareness-split}

We evaluate competence-aware failure detection using three methodological configurations:

\begin{enumerate}
    \item \textbf{EqC (ours):} Our primary approach quantifies policy competence using the energy-consistency signal $U_{\Delta E}$ derived from the EqC-adapted $\pi_0$ policy. As introduced in Section~\ref{equilibrium-energy-consistency-uncertainty}, this score projects each pre-update action iterate onto the noise-to-equilibrium chord using the floored denominator $D_v = \max(\|x^{(K)} - x^{(0)}\|_2^2, 10^{-12})$ and averages the absolute energy-consistency residuals according to~\eqref{eq:energy-consistency-score}.
    \item \textbf{Diff-DAgger (baseline):} The standard baseline using the original $\pi_0$ policy paired with Diff-DAgger's multi-pass flow-loss detector~\citep{lee2025diff}. Implementation details are provided in Section~\ref{app:competence-awareness-diffdagger-baseline-method}.
    \item \textbf{EqC + Diff-DAgger (hybrid):} A hybrid ablation variant that deploys our EqC-adapted $\pi_0$ policy for action generation while computing the failure-detection signal using Diff-DAgger's re-noising and denoising scoring mechanism. Implementation details are provided in Section~\ref{app:diff-dagger-implementation-on-the-eqc-trained-backbone}.
\end{enumerate}

All three methods are evaluated on the LIBERO and LIBERO-Plus benchmark suites using 50k-step training checkpoints. A
failure alarm is triggered when the detection signal exceeds a predefined threshold for $m$ consecutive replanning
steps.

The statistical unit of the analysis is a complete rollout rather than an individual replanning step. Before sampling either split, we stratify the evaluation pool across the $8$ dataset variants (the original LIBERO and the $7$ LIBERO-Plus perturbation families). Within every stratum, rollouts are assigned proportionally to a $10\%$ calibration split and a $90\%$ test split, ensuring that neither split is dominated by any particular task or perturbation type. A rollout is treated as an indivisible unit: its initial-state instance, outcome label, and complete sequence of replanning signals are assigned strictly to one split. No rollout or trajectory segment is shared across calibration and testing, and all compared detectors evaluate on the exact same split membership. Further details regarding the composition of the calibration split are provided in Section~\ref{app:calibration-split-composition}.

Policy checkpoints are fixed before detector calibration and remain unchanged throughout the analysis. The calibration split is used only to select the temporal window $m$ and operating threshold $\tau^*$; it is not used for checkpoint selection or detector adaptation. Once selected, $(m,\tau^*)$ is frozen and applied once to the
held-out test rollouts. Test outcomes and signal trajectories therefore cannot influence the alarm rule. By preserving
the task, perturbation, and outcome composition across the two splits, the protocol evaluates rollout-level failure
discrimination without confounding the comparison through a different benchmark mixture.

Each rollout produces a time-varying sequence of failure-detection signals. For a fixed window size $m$ and threshold
$\tau$, a rollout raises an alarm if its signal exceeds $\tau$ for any sequence of $m$ consecutive replanning steps;
otherwise, no alarm is raised. We then compute TPR and FPR from rollout-level success and failure labels and construct
the ROC curve by sweeping $\tau$. To examine whether performance remains stable across calibration procedures rather
than being tied to a particular threshold-selection rule, we evaluate two calibration strategies on the calibration
split:

\begin{enumerate}
  \item \textbf{FPR-Constrained Calibration:} As reported in the main text (Section~\ref{simulated-rollout-analysis}), all candidate methods are first evaluated on the $10\%$ calibration split. By performing a fine-grained sweep over candidate thresholds logged from these rollouts, we fix a given window size $m$ and select the optimal threshold $\tau^*$ that maximizes TPR while strictly constraining the FPR to approximately $10\%$. This calibrated parameter pair $(m, \tau^*)$ is then frozen and applied to evaluate the remaining $90\%$ test split.
  \item \textbf{Youden's $J$ Index Calibration:} Alternatively, for a given sliding window size $m$, the optimal threshold $\tau^*$ is selected on the $10\%$ calibration split by maximizing Youden's $J$ statistic ($J = \text{TPR} - \text{FPR}$). The calibrated threshold pair $(m, \tau^*)$ is subsequently deployed on the $90\%$ test split to trigger failure alarms.
\end{enumerate}

Across both FPR-constrained calibration and Youden's $J$ index calibration in Table~\ref{tab:window-size-calibration-comparison}, EqC consistently achieves the highest AUROC and delivers the most favorable overall trade-off among TPR, FPR, and Accuracy across the vast majority of window sizes $m$. This broader advantage indicates that our competence signal provides robust failure detection without depending on a specific threshold-selection rule.

AUROC is independent of the operating threshold $\tau$, so we report a single global ranking statistic obtained by sweeping $\tau$ over all labeled rollouts rather than a separate AUROC for each calibration rule. This descriptive quantity is not used to select $(m,\tau^*)$ or to compute the threshold-dependent operating metrics, which retain the calibration/test separation above. Regarding the time-saved metric, when an alarm is raised at elapsed horizon $h_{\mathrm{alarm}}$ in an episode capped at $H_{\max}$, the percentage of saved inference time is calculated as $100(1-h_{\mathrm{alarm}}/H_{\max})\%$.

\subsection{Real-World Failure Detection Protocol}
\label{app:real-world-failure-detection-protocol}

The real-world testbed evaluates the intermediate Round 3 policy checkpoint fine-tuned on $\mathcal{D}_3$ (36 cumulative demonstrations) on \textit{Hierarchical Block Stacking}. Across nominal block configurations , we collect 50 rollouts comprising 25 successful and 25 failed executions, with binary ground-truth labels assigned strictly from the final task outcome. To reduce short-timescale measurement noise, each detection signal is smoothed with a three-step moving average: starting from the third replanning step, the current value is averaged with the two preceding values. An alarm is raised when the smoothed signal remains above its threshold for $m$ consecutive replanning steps. Appendix~\ref{app:real-robot-details} provides the platform, task, and rollout-collection details.

\subsection{Composition of the Calibration Split}\label{app:calibration-split-composition}

The shared $10\%$ calibration split comprises $1{,}203$ rollouts in total. To ensure balanced coverage across environmental conditions, we perform stratified sampling across the $8$ dataset variants (the original LIBERO and $7$ LIBERO-Plus perturbation types). Within each dataset variant, rollouts are selected purely at random and shared identically across all evaluated methods (though ground-truth outcome labels depend on each policy's execution). As detailed in Table~\ref{tab:calibration-split-breakdown}, because rollouts are randomly sampled within each perturbation family, the distribution across the four suites naturally yields a nearly uniform breakdown ($\approx 25\%$ per suite) without requiring explicit suite-level stratification.

\begin{table}[htbp]
  \caption{Detailed composition of the shared $10\%$ calibration split ($1{,}203$ total rollouts).}
  \label{tab:calibration-split-breakdown}
  \centering
  \small
  \begin{subtable}[t]{0.48\linewidth}
    \centering
    \caption{Breakdown by Perturbation Variant}
    \label{tab:calibration-by-dataset}
    \begin{tabular}{@{} l c r @{}}
      \toprule
      Dataset & Count & Ratio (\%) \\
      \midrule
      Original LIBERO & 200 & 16.63 \\
      Camera          & 160 & 13.30 \\
      Noise           & 160 & 13.30 \\
      Robot           & 155 & 12.88 \\
      Language        & 154 & 12.80 \\
      Layout          & 152 & 12.64 \\
      Light           & 114 & 9.48 \\
      Background      & 108 & 8.98 \\
      \bottomrule
    \end{tabular}
  \end{subtable}
  \hfill
  \begin{subtable}[t]{0.48\linewidth}
    \centering
    \caption{Breakdown by Benchmark Suite}
    \label{tab:calibration-by-suite}
    \begin{tabular}{@{} l c r @{}}
      \toprule
      Suite & Count & Ratio (\%) \\
      \midrule
      LIBERO-Goal    & 307 & 25.52 \\
      LIBERO-Object  & 302 & 25.10 \\
      LIBERO-Spatial & 299 & 24.85 \\
      LIBERO-10      & 295 & 24.52 \\
      \bottomrule
    \end{tabular}
  \end{subtable}
\end{table}

\subsection{Diff-DAgger Baseline Implementation Adapted for \texorpdfstring{$\pi_0$}{pi0}}\label{app:competence-awareness-diffdagger-baseline-method}

To adapt Diff-DAgger as a failure-detection baseline for the $\pi_0$ backbone, we use $\pi_0$'s conditional flow-matching formulation. The detector evaluates whether an action sequence generated by $\pi_0$ is consistent with the flow-matching vector field learned during training. Because this inference-only procedure does not modify the original $\pi_0$ training objective, it preserves the $\pi_0$ architecture, including the PaliGemma vision-language backbone, action expert, observation preprocessing, and action projection head. The detector is attached directly to the original Euler sampler.

Given an observation $o$, $\pi_0$ first generates a predicted action sequence $\hat{a}$ through its standard Euler sampler:
\begin{equation}
    \hat{a} = S_{\theta}(o, \epsilon_0),
\end{equation}
where $\epsilon_0 \sim \mathcal{N}(0, I)$ is an initial Gaussian noise sample. The generated action sequence $\hat{a}$ is then treated as a pseudo-label, so the detector requires no expert demonstrations or additional supervision at test time.

To score this pseudo-label, we follow the timestep sampling and action corruption scheme used during training of the \(\pi_0\) backbone. For each Monte Carlo sample \(q \in \{1,\ldots,M\}\), we independently draw a Beta-distributed variable \(b_q\) and Gaussian noise \(\epsilon_q\), and obtain the timestep \(t_q\) through the training-time affine transformation:
\begin{equation}
    \begin{aligned}
        b_q &\sim \operatorname{Beta}(1.5,1),
        \qquad \epsilon_q \sim \mathcal{N}(0,I),\\
        t_q &= 0.999\,b_q + 0.001.
    \end{aligned}
\end{equation}
The resulting perturbed action is
\begin{equation}
    x_{t_q} = t_q\epsilon_q + (1-t_q)\hat{a},
\end{equation}
with the corresponding target velocity defined as
\begin{equation}
    u_{t_q} = \epsilon_q - \hat{a}.
\end{equation}
The $\pi_0$ model is then queried using the original observation $o$, the perturbed action $x_{t_q}$, and the sampled timestep $t_q$, producing the predicted velocity vector:
\begin{equation}
    v_\theta(x_{t_q}, t_q, o).
\end{equation}

The Diff-DAgger score for the \(q\)-th Monte Carlo sample is the normalized flow-matching error:
\begin{equation}
    \ell_q
    =
    \frac{1}{H D_{\mathrm{act}}}
    \left\|
        v_\theta(x_{t_q},t_q,o) - u_{t_q}
    \right\|_{\mathrm{act}}^2,
\end{equation}
where \(H\) is the action horizon, \(D_{\mathrm{act}}\) is the number of active action dimensions, and the squared norm sums over all action-horizon steps and active coordinates. Zero-padded dimensions are masked out in both the pseudo-label construction and the loss evaluation.

The final per-replan failure score is computed as the averaged flow loss over $M$ Monte Carlo evaluations:
\begin{equation}
    \mathcal{L}_{\mathrm{flow}} = \frac{1}{M} \sum_{q=1}^{M} \ell_q.
\end{equation}
We reserve $K$ for the number of EqC refinement updates and use $M$ for the number of Diff-DAgger Monte Carlo samples; the default EqC and Diff-DAgger configurations use $K=10$ and $M=10$, respectively.

\subsection{Diff-DAgger Baseline Implementation on the EqC-Trained Backbone}
\label{app:diff-dagger-implementation-on-the-eqc-trained-backbone}

We adapt Diff-DAgger's loss-based detection signal to the EqC-trained backbone by evaluating the energy-gradient matching objective on generated action chunks. Let \(o\) denote the current multimodal observation, comprising visual inputs, proprioceptive state, and the language instruction. Conditioned on \(o\), the EqC controller generates an action chunk \(\hat{a}\) through \(K\) refinement updates.

EqC defines the stationary energy and its associated gradient field as
\begin{equation}
    E_\theta(x,o)
    = \langle x, f_\theta(x,0,o) \rangle,
    \qquad
    g_\theta(x,o)
    = \nabla_x E_\theta(x,o).
\end{equation}
The scoring field is \(g_\theta\), obtained by differentiating the scalar energy function with respect to the action coordinates. The network timestep remains fixed at \(t_{\mathrm{model}}=0\). Action-space inner products and norms are evaluated over active coordinates, with padded dimensions masked out.

The generated action chunk \(\hat{a}\) serves as a pseudo-label for loss evaluation. For this baseline, we set \(M=K\), using \(K\) Monte Carlo samples to match the \(K\) refinement steps from which EqC's native detection signal aggregates information. For each sample \(k \in \{1,\ldots,K\}\), we independently draw the corruption time and Gaussian noise according to the sampling scheme used during EqC training:
\begin{equation}
    t_k \sim \mathcal{U}(0,1),
    \qquad
    \epsilon_k \sim \mathcal{N}(0,I).
\end{equation}
The corresponding perturbed action and gradient-matching target are
\begin{equation}
    x_{t_k}
    = (1-t_k)\epsilon_k + t_k\hat{a},
    \qquad
    u_{t_k}^{*}
    = c(t_k)(\hat{a}-\epsilon_k),
\end{equation}
where \(c(t)\) is the rescaling function used during EqC training. The sampled time \(t_k\) determines the interpolation and target rescaling only; it does not condition the stationary field \(g_\theta\).

The per-replan failure score is the normalized gradient-matching error averaged over the \(K\) Monte Carlo evaluations:
\begin{equation}
    s^{\mathrm{EqC}}(o)
    =
    \frac{1}{K H D_{\mathrm{act}}}
    \sum_{k=1}^{K}
    \left\|
        g_\theta(x_{t_k},o) - u_{t_k}^{*}
    \right\|_{\mathrm{act}}^2,
\end{equation}
where \(H\) is the action horizon, \(D_{\mathrm{act}}\) is the number of active action dimensions, and the squared norm sums over all action-horizon steps and active coordinates.

The detector uses EqC's gradient-matching loss. This construction preserves Diff-DAgger's loss-based scoring principle while defining the prediction error using EqC's stationary energy-gradient field and corresponding training target. Scoring is performed after action generation without updating either the generated action or the model parameters. Holding the EqC-trained backbone and action-generation procedure fixed provides a controlled baseline for assessing the contribution of EqC's path-wise energy-consistency signal.

\subsection{Detailed Results}\label{app:competence-awareness-detailed-results}

As presented in Table~\ref{tab:window-size-calibration-comparison} and Table~\ref{tab:time-saved-calibration-comparison}, we evaluate failure detection performance and early intervention capability across sliding window sizes $m \in [3, 13]$ under both calibration protocols. Note that AUROC reflects the threshold-independent detection capacity evaluated globally across all dataset rollouts (and is thus independent of split partitions and identical across both subtables for any given $m$), whereas TPR, FPR, Accuracy, and Time Saved reflect operational performance on the $90\%$ test split using operating thresholds calibrated from the $10\%$ calibration split.

A comprehensive comparative analysis highlights the clear advantages of EqC across window sizes, operating metrics, and thresholding strategies:

\begin{itemize}
    \item \textbf{Dominance in Theoretical Detection Capacity (AUROC):} Across all evaluated window sizes $m \in [3, 13]$, EqC consistently achieves the highest AUROC values among all compared methods, ascending monotonically from $0.8877$ ($m=3$) to $0.9300$ ($m=13$). It maintains a decisive lead over both the baseline Diff-DAgger ($0.8296\text{--}0.8576$) and the hybrid variant EqC + Diff-DAgger ($0.8778\text{--}0.8978$), confirming that energy-derived competence signals provide an inherently stronger and more reliable failure detection capability.
    
    \item \textbf{Superiority under FPR-Constrained Calibration:} Under strict false-alarm control (Table~\ref{tab:fpr-constrained-sweep}, target $\text{FPR} \approx 10\%$), EqC significantly outperforms Diff-DAgger in both TPR and overall Accuracy across almost all $m$. For example, at $m=13$, EqC reaches $78.9\%$ TPR and $87.4\%$ Accuracy, compared to Diff-DAgger's $63.8\%$ TPR and $83.1\%$ Accuracy. Furthermore, EqC maintains a lead over the hybrid variant across the majority of window sizes, demonstrating higher TPR.
    
    \item \textbf{Superiority under Youden's $J$ Calibration:} When optimizing the decision threshold via Youden's $J$ index ($J = \text{TPR} - \text{FPR}$, Table~\ref{tab:youden-j-sweep}), EqC consistently achieves higher $J$ statistic values, yielding a distinctly superior trade-off between TPR and FPR control. Under this objective, EqC attains higher peak TPR (reaching up to $87.5\%$ at $m=12$) while securing the highest overall classification Accuracy across the vast majority of window sizes (peaking at $86.4\%$ at $m=11$). This confirms that EqC's operational superiority stems from a fundamentally better signal trade-off profile rather than an artifact of a specific calibration rule.
    
    \item \textbf{Responsiveness in Early Intervention (Time Saved):} Beyond accurately classifying failures, a practical failure-detection system must intervene timely to allow trajectory recovery. Table~\ref{tab:time-saved-calibration-comparison} reports the execution time saved specifically on failed rollouts within the $90\%$ test split. Across both calibration protocols, EqC achieves the highest overall time saved across almost all window sizes $m$, consistently outperforming the baseline Diff-DAgger and surpassing the hybrid variant (EqC + Diff-DAgger) in the vast majority of cases. While the hybrid model also improves upon pure Diff-DAgger—further confirming the value of incorporating energy-based information—EqC provides the most responsive signal overall, timely avoiding rollouts from reaching an irrecoverable failure state.
\end{itemize}

Overall, these results confirm that EqC provides a robust, scalable, and highly responsive failure-detection signal regardless of the choice of threshold calibration, window size, or performance metric.

\begin{table}[htbp]
  \centering
  \small
  \caption{Sensitivity analysis of detection window size $m$ under two distinct threshold calibration strategies on the 90\% test split.}
  \label{tab:window-size-calibration-comparison}

  \begin{subtable}{\textwidth}
    \centering
    \caption{FPR-Constrained Calibration (Target $\text{FPR} \approx 10\%$ on calibration split)}
    \label{tab:fpr-constrained-sweep}
    \setlength{\tabcolsep}{3.5pt}
    \renewcommand{\arraystretch}{1.05}
    \begin{tabular}{@{} c ccc ccc ccc @{}}
      \toprule
      & \multicolumn{3}{c}{Diff-DAgger} & \multicolumn{3}{c}{EqC + Diff-DAgger} & \multicolumn{3}{c}{\eqcmethod} \\
      \cmidrule(lr){2-4} \cmidrule(lr){5-7} \cmidrule(lr){8-10}
      $m$ & AUROC & TPR / FPR (\%) & Acc (\%) & AUROC & TPR / FPR (\%) & Acc (\%) & AUROC & TPR / FPR (\%) & Acc (\%) \\
      \midrule
      3  & 0.8296 & 58.1 / 9.1 & 82.1 & 0.8778 & 71.9 / 9.9 & 84.8 & 0.8877 & 69.5 / 9.1 & 84.6 \\
      4  & 0.8307 & 56.3 / 8.3 & 82.2 & 0.8774 & 70.4 / 8.8 & 85.1 & 0.8947 & 70.8 / 8.9 & 85.1 \\
      5  & 0.8310 & 57.1 / 8.6 & 82.2 & 0.8802 & 70.9 / 9.3 & 84.9 & 0.9005 & 70.7 / 7.9 & 85.8 \\
      6  & 0.8328 & 55.7 / 8.1 & 82.2 & 0.8810 & 71.2 / 9.2 & 85.1 & 0.9070 & 74.7 / 9.0 & 86.2 \\
      7  & 0.8343 & 58.3 / 9.2 & 82.1 & 0.8818 & 71.7 / 8.9 & 85.4 & 0.9117 & 74.8 / 9.1 & 86.1 \\
      8  & 0.8345 & 59.8 / 9.9 & 82.0 & 0.8829 & 72.4 / 9.5 & 85.2 & 0.9147 & 76.1 / 9.3 & 86.3 \\
      9  & 0.8352 & 61.5 / 11.0 & 81.6 & 0.8849 & 73.5 / 9.8 & 85.3 & 0.9183 & 76.7 / 8.9 & 86.9 \\
      10 & 0.8391 & 60.3 / 10.1 & 82.0 & 0.8862 & 72.5 / 9.1 & 85.5 & 0.9222 & 77.7 / 8.9 & 87.1 \\
      11 & 0.8430 & 61.3 / 10.2 & 82.2 & 0.8888 & 72.8 / 8.8 & 85.8 & 0.9244 & 77.9 / 8.5 & 87.5 \\
      12 & 0.8517 & 62.3 / 9.2 & 83.2 & 0.8936 & 74.6 / 10.4 & 85.2 & 0.9274 & 78.4 / 8.4 & 87.6 \\
      13 & 0.8576 & 63.8 / 9.9 & 83.1 & 0.8978 & 74.0 / 9.2 & 85.9 & 0.9300 & 78.9 / 9.0 & 87.4 \\
      \bottomrule
    \end{tabular}
  \end{subtable}

  \vspace{10pt}

  \begin{subtable}{\textwidth}
    \centering
    \caption{Youden's $J$ Index Calibration (Maximizing $J = \text{TPR} - \text{FPR}$ on calibration split)}
    \label{tab:youden-j-sweep}
    \setlength{\tabcolsep}{3.5pt}
    \renewcommand{\arraystretch}{1.05}
    \begin{tabular}{@{} c ccc ccc ccc @{}}
      \toprule
      & \multicolumn{3}{c}{Diff-DAgger} & \multicolumn{3}{c}{EqC + Diff-DAgger} & \multicolumn{3}{c}{\eqcmethod} \\
      \cmidrule(lr){2-4} \cmidrule(lr){5-7} \cmidrule(lr){8-10}
      $m$ & AUROC & TPR / FPR (\%) & Acc (\%) & AUROC & TPR / FPR (\%) & Acc (\%) & AUROC & TPR / FPR (\%) & Acc (\%) \\
      \midrule
      3  & 0.8296 & 73.9 / 20.8 & 77.8 & 0.8778 & 78.3 / 16.9 & 81.7 & 0.8877 & 78.1 / 14.9 & 83.0 \\
      4  & 0.8307 & 65.2 / 13.3 & 80.9 & 0.8774 & 80.6 / 19.5 & 80.5 & 0.8947 & 82.4 / 20.0 & 80.7 \\
      5  & 0.8310 & 76.0 / 24.4 & 75.7 & 0.8802 & 79.2 / 16.8 & 82.0 & 0.9005 & 82.1 / 17.5 & 82.4 \\
      6  & 0.8328 & 73.0 / 19.9 & 78.2 & 0.8810 & 80.0 / 18.1 & 81.4 & 0.9070 & 79.4 / 13.0 & 84.7 \\
      7  & 0.8343 & 75.0 / 23.2 & 76.3 & 0.8818 & 78.6 / 15.9 & 82.5 & 0.9117 & 85.3 / 20.2 & 81.4 \\
      8  & 0.8345 & 70.4 / 18.9 & 78.2 & 0.8829 & 76.9 / 13.7 & 83.6 & 0.9147 & 81.9 / 14.5 & 84.4 \\
      9  & 0.8352 & 65.5 / 14.5 & 80.2 & 0.8849 & 73.8 / 10.1 & 85.2 & 0.9183 & 83.8 / 15.5 & 84.3 \\
      10 & 0.8391 & 71.3 / 17.6 & 79.4 & 0.8862 & 80.8 / 18.0 & 81.7 & 0.9222 & 84.0 / 14.7 & 85.0 \\
      11 & 0.8430 & 70.3 / 16.4 & 80.1 & 0.8888 & 78.1 / 14.2 & 83.6 & 0.9244 & 82.1 / 11.8 & 86.4 \\
      12 & 0.8517 & 65.1 / 11.0 & 82.6 & 0.8936 & 77.5 / 12.7 & 84.4 & 0.9274 & 87.5 / 18.3 & 83.4 \\
      13 & 0.8576 & 69.0 / 13.6 & 81.7 & 0.8978 & 74.6 / 9.6 & 85.8 & 0.9300 & 85.5 / 14.5 & 85.5 \\
      \bottomrule
    \end{tabular}
  \end{subtable}
\end{table}

\begin{table}[htbp]
  \centering
  \small
  \caption{Evaluation of early intervention capability (Time Saved \%) on the 90\% test split. The operational decision thresholds are calibrated using the 10\% calibration split under FPR-Constrained and Youden's $J$ index strategies.}
  \label{tab:time-saved-calibration-comparison}
  \setlength{\tabcolsep}{8pt}
  \renewcommand{\arraystretch}{1.05}
  \begin{tabular}{@{} c ccc ccc @{}}
    \toprule
    & \multicolumn{3}{c}{FPR-Constrained Calibration} & \multicolumn{3}{c}{Youden's $J$ Index Calibration} \\
    \cmidrule(lr){2-4} \cmidrule(lr){5-7}
    $m$ & Diff-DAgger & EqC + Diff & \eqcmethod & Diff-DAgger & EqC + Diff & \eqcmethod \\
    \midrule
    3  & 33.6 & 44.5 & 45.2 & 47.9 & 51.8 & 53.7 \\
    4  & 31.8 & 42.5 & 45.3 & 39.2 & 53.8 & 57.4 \\
    5  & 32.1 & 42.7 & 44.5 & 49.4 & 51.8 & 55.5 \\
    6  & 30.6 & 42.6 & 46.7 & 45.2 & 52.2 & 51.5 \\
    7  & 32.0 & 42.8 & 46.6 & 47.0 & 50.2 & 57.1 \\
    8  & 32.5 & 43.0 & 46.3 & 42.4 & 47.9 & 52.2 \\
    9  & 33.9 & 43.5 & 46.2 & 37.4 & 43.8 & 53.7 \\
    10 & 32.7 & 42.4 & 46.0 & 41.8 & 51.0 & 53.0 \\
    11 & 33.3 & 42.1 & 45.5 & 40.4 & 47.9 & 49.8 \\
    12 & 33.3 & 43.8 & 45.3 & 35.4 & 46.4 & 54.8 \\
    13 & 34.3 & 42.4 & 45.3 & 38.6 & 42.9 & 51.5 \\
    \bottomrule
  \end{tabular}
\end{table}

To further establish the theoretical upper bound on detection performance independent of split partitioning, Table~\ref{tab:full-m-k-experiments} presents an evaluation conducted directly on the entire dataset (without calibration/test splitting). In this setup, AUROC is calculated by sweeping the threshold globally, while TPR, FPR, and Accuracy are reported using the theoretical optimal threshold determined by maximizing Youden's $J$ index over the full dataset. Across all evaluated sliding-window sizes $m \in [3, 13]$, EqC consistently achieves the highest AUROC and the best optimal TPR/FPR trade-off compared to the Diff-DAgger baseline ($M=10$) and the hybrid variant. Furthermore, increasing Diff-DAgger's sample count from $M=10$ to $M \in \{16, 32, 64, 128\}$ fails to close this performance gap, indicating that merely scaling Monte Carlo dropout samples does not yield a fundamentally better failure-detection signal.

\begin{table*}[htbp]
  \caption{Evaluation of the theoretical upper bound of online failure detection across window sizes $m \in \{3,\ldots,13\}$ under Youden's optimal thresholding strategy (Best Youden). \textit{Top:} Comparison among the baseline Diff-DAgger, the EqC + Diff-DAgger hybrid, and our proposed EqC; EqC and Diff-DAgger use their default $K=10$ and $M=10$ samples, respectively. \textit{Bottom:} Performance of Diff-DAgger when scaling $M \in \{16, 32, 64, 128\}$. TPR and FPR are reported as percentages at Youden's optimal threshold.}
  \label{tab:full-m-k-experiments}
  \centering
  \small
  \textbf{(a) Primary Method Comparison}\\[0.4em]
  \setlength{\tabcolsep}{3.5pt}
  \begin{tabular}{@{} c ccc ccc ccc @{}}
    \toprule
    & \multicolumn{3}{c}{Diff-DAgger (baseline)} & \multicolumn{3}{c}{EqC + Diff-DAgger} & \multicolumn{3}{c}{\eqcmethod} \\
    \cmidrule(lr){2-4} \cmidrule(lr){5-7} \cmidrule(lr){8-10}
    $m$ & AUROC & TPR / FPR (\%) & Acc (\%) & AUROC & TPR / FPR (\%) & Acc (\%) & AUROC & TPR / FPR (\%) & Acc (\%) \\
    \midrule
    3  & 0.8296 & 72.9 / 19.4 & 78.5 & 0.8778 & 73.8 / 11.5 & \textbf{84.2} & \textbf{0.8877} & 78.0 / 15.0 & 82.9 \\
    4  & 0.8307 & 72.3 / 18.9 & 78.7 & 0.8774 & 76.7 / 14.8 & 82.7 & \textbf{0.8947} & 79.2 / 15.3 & \textbf{83.0} \\
    5  & 0.8310 & 67.8 / 15.1 & 80.3 & 0.8802 & 75.7 / 13.4 & \textbf{83.4} & \textbf{0.9005} & 81.1 / 15.9 & 83.2 \\
    6  & 0.8328 & 74.0 / 20.6 & 78.0 & 0.8810 & 75.4 / 12.6 & 83.9 & \textbf{0.9070} & 79.2 / 12.7 & \textbf{84.9} \\
    7  & 0.8343 & 69.5 / 16.0 & 80.1 & 0.8818 & 76.3 / 13.5 & 83.5 & \textbf{0.9117} & 78.2 / 11.7 & \textbf{85.3} \\
    8  & 0.8345 & 69.1 / 16.5 & 79.7 & 0.8829 & 76.6 / 13.3 & 83.7 & \textbf{0.9147} & 82.2 / 14.8 & \textbf{84.3} \\
    9  & 0.8352 & 72.1 / 19.5 & 78.2 & 0.8849 & 76.5 / 12.9 & 84.0 & \textbf{0.9183} & 83.0 / 14.4 & \textbf{84.8} \\
    10 & 0.8391 & 71.7 / 17.6 & 79.6 & 0.8862 & 78.2 / 14.1 & 83.6 & \textbf{0.9222} & 81.1 / 11.5 & \textbf{86.3} \\
    11 & 0.8430 & 71.8 / 16.9 & 80.1 & 0.8888 & 78.9 / 14.9 & 83.2 & \textbf{0.9244} & 82.7 / 12.5 & \textbf{86.1} \\
    12 & 0.8517 & 73.6 / 17.5 & 80.1 & 0.8936 & 79.1 / 14.3 & 83.8 & \textbf{0.9274} & 82.3 / 11.9 & \textbf{86.4} \\
    13 & 0.8576 & 75.9 / 19.2 & 79.5 & 0.8978 & 77.1 / 11.8 & 85.0 & \textbf{0.9300} & 86.6 / 15.4 & \textbf{85.2} \\
    \bottomrule
  \end{tabular}

  \vspace{1.2em}

  \textbf{(b) Diff-DAgger Scaling Across $M$}\\[0.4em]
  \setlength{\tabcolsep}{4.5pt}
  \begin{tabular}{@{} c cc cc cc cc @{}}
    \toprule
    & \multicolumn{2}{c}{Diff-DAgger ($M=16$)} & \multicolumn{2}{c}{Diff-DAgger ($M=32$)} & \multicolumn{2}{c}{Diff-DAgger ($M=64$)} & \multicolumn{2}{c}{Diff-DAgger ($M=128$)} \\
    \cmidrule(lr){2-3} \cmidrule(lr){4-5} \cmidrule(lr){6-7} \cmidrule(lr){8-9}
    $m$ & AUROC & TPR / FPR (\%) & AUROC & TPR / FPR (\%) & AUROC & TPR / FPR (\%) & AUROC & TPR / FPR (\%) \\
    \midrule
    3  & 0.8266 & 68.0 / 14.6 & 0.8302 & 74.5 / 21.1 & 0.8342 & 71.6 / 16.2 & 0.8264 & 71.3 / 18.2 \\
    4  & 0.8255 & 69.3 / 16.8 & 0.8263 & 70.2 / 17.1 & 0.8315 & 72.4 / 18.1 & 0.8216 & 68.3 / 16.3 \\
    5  & 0.8280 & 71.6 / 18.5 & 0.8250 & 67.9 / 15.4 & 0.8294 & 71.8 / 17.9 & 0.8210 & 69.6 / 18.1 \\
    6  & 0.8270 & 67.2 / 14.6 & 0.8261 & 70.9 / 18.8 & 0.8295 & 68.9 / 15.4 & 0.8218 & 67.4 / 16.1 \\
    7  & 0.8289 & 68.9 / 15.8 & 0.8295 & 71.0 / 18.0 & 0.8289 & 68.8 / 15.7 & 0.8210 & 65.5 / 13.8 \\
    8  & 0.8298 & 69.8 / 16.7 & 0.8297 & 69.2 / 17.2 & 0.8311 & 69.8 / 16.7 & 0.8209 & 66.1 / 14.3 \\
    9  & 0.8304 & 69.6 / 16.2 & 0.8306 & 67.9 / 15.2 & 0.8281 & 69.9 / 17.0 & 0.8203 & 67.9 / 15.8 \\
    10 & 0.8316 & 70.9 / 17.1 & 0.8313 & 69.2 / 16.3 & 0.8292 & 69.2 / 16.4 & 0.8219 & 68.0 / 16.2 \\
    11 & 0.8338 & 68.8 / 14.7 & 0.8355 & 67.5 / 14.2 & 0.8334 & 70.3 / 17.1 & 0.8248 & 67.7 / 15.5 \\
    12 & 0.8389 & 72.1 / 18.1 & 0.8414 & 73.8 / 19.8 & 0.8406 & 70.1 / 16.0 & 0.8317 & 65.6 / 12.6 \\
    13 & 0.8437 & 70.8 / 15.8 & 0.8470 & 70.7 / 15.8 & 0.8481 & 70.6 / 15.9 & 0.8385 & 72.6 / 19.0 \\
    \bottomrule
  \end{tabular}
\end{table*}

Beyond detection accuracy, Table~\ref{tab:time-saved-detailed} provides a detailed breakdown of the theoretical maximum inference time saved across various environmental perturbation suites in LIBERO and LIBERO-Plus. Standalone EqC achieves the highest overall time saved, outperforming both the baseline Diff-DAgger and the hybrid EqC + Diff-DAgger variant across the majority of evaluation setups and representative window sizes $m \in \{4, 11, 13\}$, reaching up to $54.0\%$ in weighted average time saved. Notably, even when scaling Diff-DAgger's sample budget to $M \in \{16, 32\}$, its performance in time saved does not improve, demonstrating that simply increasing trajectory sampling density cannot enhance its failure detection capability. While the hybrid model also demonstrates notable gains over pure Diff-DAgger—further underscoring the benefit of energy-derived features—standalone EqC provides the most effective and responsive signal overall for timely failure intervention.

\begin{table*}[t]
  \caption{Breakdown of theoretically optimal inference time saved (\%) across LIBERO and LIBERO-Plus suites under representative evaluation setups. \textit{Orig.} corresponds to the LIBERO tasks, while the remaining seven columns represent environmental perturbation variants in LIBERO-Plus. Diff-DAgger uses $M=10$ unless otherwise specified.}
  \label{tab:time-saved-detailed}
  \centering
  \tabstyle
  \setlength{\tabcolsep}{1.5pt}
  \begin{tabularx}{\linewidth}{@{} l *{9}{>{\centering\arraybackslash}X} @{}}
    \toprule
    \headrow Method & Orig. & Camera & Robot & Lang. & Light & Backgr. & Noise & Layout & W. Avg. \\
    \midrule
    Diff-DAgger ($m=4$) 
      & 27.7 & 50.5 & 50.1 & 46.4 & 30.0 & 34.0 & 44.1 & 47.9 & 46.1 \\
    Diff-DAgger ($m=11$) 
      & 26.5 & 44.2 & 46.0 & 39.9 & 32.5 & 32.3 & 37.3 & 44.9 & 41.5 \\
    Diff-DAgger ($m=13$) 
      & 28.9 & 47.2 & 48.6 & 44.1 & 33.8 & 34.5 & 43.1 & 46.1 & 44.6 \\
    \midrule
    Diff-DAgger ($M=16, m=4$) 
      & 27.0 & 45.0 & 47.5 & 44.1 & 33.8 & 28.0 & 39.0 & 44.0 & 42.7 \\
    Diff-DAgger ($M=16, m=11$) 
      & 23.5 & 41.3 & 42.4 & 37.4 & 28.5 & 30.2 & 35.1 & 40.6 & 38.4 \\
    Diff-DAgger ($M=16, m=13$) 
      & 25.9 & 42.8 & 43.5 & 39.0 & 30.7 & 30.5 & 37.1 & 41.8 & 39.8 \\
    \midrule
    Diff-DAgger ($M=32, m=4$) 
      & 28.1 & 48.5 & 48.3 & 42.8 & 34.9 & 34.4 & 38.5 & 45.6 & 43.9 \\
    Diff-DAgger ($M=32, m=11$) 
      & 28.3 & 40.4 & 42.9 & 34.8 & 31.1 & 28.7 & 36.1 & 38.9 & 38.0 \\
    Diff-DAgger ($M=32, m=13$) 
      & 28.6 & 42.9 & 44.1 & 37.4 & 32.3 & 31.3 & 38.6 & 40.6 & 40.0 \\
    \midrule
    EqC + Diff-DAgger ($m=4$) 
      & 30.7 & 50.3 & 51.8 & 56.9 & 44.1 & 27.0 & 44.5 & 48.9 & 49.1 \\
    EqC + Diff-DAgger ($m=11$) 
      & 30.3 & 51.4 & 50.8 & 53.7 & 42.6 & 29.1 & 46.3 & 48.3 & 48.5 \\
    EqC + Diff-DAgger ($m=13$) 
      & 26.8 & 48.4 & 47.0 & 50.2 & 40.8 & 26.4 & 43.5 & 45.2 & 45.3 \\
    \midrule
    \eecrow \eqcmethod ($m=4$)
      & 26.7 & 54.8 & 57.3 & 70.7 & 41.7 & 29.0 & 40.1 & 55.4 & \best{54.0} \\
    \eecrow \eqcmethod ($m=11$) 
      & 24.7 & 53.4 & 55.3 & 60.8 & 39.7 & 30.2 & 41.5 & 47.8 & \best{50.5} \\
    \eecrow \eqcmethod ($m=13$) 
      & 26.7 & 56.4 & 58.2 & 61.0 & 40.3 & 32.2 & 46.1 & 49.5 & \best{52.8} \\
    \bottomrule
  \end{tabularx}
\end{table*}

Finally, we evaluate the computational efficiency of EqC by benchmarking per-replan inference latency, encompassing both action candidate generation and failure-score computation at each step. As shown in Table~\ref{tab:inference-latency}, standalone EqC exhibits an inference latency comparable to that of low-budget Diff-DAgger variants. Crucially, as demonstrated in our earlier evaluations, EqC achieves significantly stronger and more reliable failure-detection performance than Diff-DAgger. While increasing Diff-DAgger's sample budget \(M\) monotonically inflates its inference latency, it yields no significant improvement in failure detection capability, rendering larger sample budgets computationally inefficient for real-time introspection. Furthermore, EqC + Diff-DAgger incurs a substantial latency overhead compared to either standalone EqC or Diff-DAgger, proving impractical for physical deployment. Consequently, EqC delivers the optimal balance, offering superior failure detection while maintaining the computational responsiveness required for real-time robotic control.

\begin{table}[htbp]
  \caption{Comparison of per-replan inference latency (ms) across methods. The latency includes both model action generation and failure-detection signal computation (evaluated at batch size = 1 after JAX warm-up). The optimization step budget $K$ for EqC is set to 10.}
  \label{tab:inference-latency}
  \centering
  \small
  \setlength{\tabcolsep}{12pt}
  \begin{tabular}{@{} l c @{}}
    \toprule
    \headrow Method & Mean Latency (ms) \\
    \midrule
    Diff-DAgger ($M=10$) & 79.44 \\
    Diff-DAgger ($M=16$) & 91.43 \\
    Diff-DAgger ($M=32$) & 122.11 \\
    EqC + Diff-DAgger ($M=10$) & 171.03 \\
    \eecrow \eqcmethod & 108.86 \\
    \bottomrule
  \end{tabular}
\end{table}

\subsection{Extended Studies}\label{app:competence-awareness-other-methods}

To evaluate the fundamental advantage of tracking temporal trajectory dynamics over static point-wise metrics, we compare our path-wise energy-consistency EqC signal against two alternative single-point energy-derived baselines: (1) \textbf{Terminal Energy}: the absolute energy value at the final step of the EqC action optimization; and (2) \textbf{Last-step Gradient Norm}: the gradient norm of the energy evaluated at the final optimization step. 

To benchmark these signals, performance is evaluated across sliding-window sizes $m \in [3, 13]$, where AUROC is computed globally in a threshold-independent manner, while decision thresholds for TPR, FPR, and Accuracy are calibrated on a $10\%$ split—using both FPR-constrained and Youden's $J$ Index regimes—and evaluated on the remaining $90\%$ test split. As shown in Table~\ref{tab:ablation-calibration-comparison}, comparing these signals reveals a clear progression: scalar \textbf{Terminal Energy} offers weak detection capability (AUROC $0.5237$--$0.6489$) because uncalibrated absolute energy values vary significantly across different environmental states, whereas \textbf{Last-step Gradient Norm} improves performance substantially (AUROC $\sim 0.87$) by measuring local gradient stationarity at the final optimization step. Building on this, \textbf{EqC} achieves the highest threshold-independent discriminability across all $m$, reaching a peak AUROC of $0.9300$. Under FPR-constrained calibration, EqC strictly dominates both baselines across all $m$, yielding higher TPR and Accuracy (up to $87.6\%$). Under Youden's $J$ Index, EqC achieves the optimal trade-off between TPR and FPR across all sliding window sizes $m$, while recovering higher peak Accuracy as $m$ grows. Overall, these results confirm that tracking path-wise temporal energy dynamics along the optimization trajectory provides a fundamentally richer and more reliable introspection signal than static, single-point evaluations.

\begin{table}[htbp]
  \centering
  \small
  \caption{Ablation analysis comparing single-point energy signals against full path-wise EqC across window sizes $m$ under two calibration strategies on the 90\% test split.}
  \label{tab:ablation-calibration-comparison}
  
  \begin{subtable}{\textwidth}
    \centering
    \caption{FPR-Constrained Calibration (Target $\text{FPR} \approx 10\%$ on calibration split)}
    \label{tab:ablation-fpr-constrained}
    \setlength{\tabcolsep}{3.5pt}
    \renewcommand{\arraystretch}{1.05}
    \begin{tabular}{@{} c ccc ccc ccc @{}}
      \toprule
      & \multicolumn{3}{c}{Terminal Energy} & \multicolumn{3}{c}{Last-step Gradient Norm} & \multicolumn{3}{c}{\eqcmethod} \\
      \cmidrule(lr){2-4} \cmidrule(lr){5-7} \cmidrule(lr){8-10}
      $m$ & AUROC & TPR / FPR (\%) & Acc (\%) & AUROC & TPR / FPR (\%) & Acc (\%) & AUROC & TPR / FPR (\%) & Acc (\%) \\
      \midrule
      3  & 0.5237 & 17.1 / 10.0 & 68.4 & 0.8734 & 68.0 / 9.3 & 84.0 & 0.8877 & 69.5 / 9.1 & 84.6 \\
      4  & 0.5404 & 19.3 / 10.1 & 68.9 & 0.8707 & 67.1 / 9.1 & 83.9 & 0.8947 & 70.8 / 8.9 & 85.1 \\
      5  & 0.5604 & 23.1 / 10.3 & 69.9 & 0.8713 & 65.5 / 8.7 & 83.6 & 0.9005 & 70.7 / 7.9 & 85.8 \\
      6  & 0.5754 & 26.0 / 10.8 & 70.5 & 0.8724 & 65.7 / 9.4 & 83.2 & 0.9070 & 74.7 / 9.0 & 86.2 \\
      7  & 0.5861 & 26.8 / 10.2 & 71.1 & 0.8737 & 66.0 / 9.3 & 83.3 & 0.9117 & 74.8 / 9.1 & 86.1 \\
      8  & 0.5974 & 28.7 / 9.8  & 71.9 & 0.8737 & 66.2 / 9.5 & 83.3 & 0.9147 & 76.1 / 9.3 & 86.3 \\
      9  & 0.6082 & 29.9 / 9.4  & 72.6 & 0.8744 & 64.7 / 9.0 & 83.2 & 0.9183 & 76.7 / 8.9 & 86.9 \\
      10 & 0.6214 & 32.5 / 10.0 & 72.9 & 0.8758 & 64.2 / 8.5 & 83.4 & 0.9222 & 77.7 / 8.9 & 87.1 \\
      11 & 0.6309 & 33.1 / 9.6  & 73.4 & 0.8769 & 65.4 / 9.3 & 83.1 & 0.9244 & 77.9 / 8.5 & 87.5 \\
      12 & 0.6410 & 33.8 / 9.6  & 73.6 & 0.8766 & 64.1 / 8.8 & 83.1 & 0.9274 & 78.4 / 8.4 & 87.6 \\
      13 & 0.6489 & 34.9 / 9.7  & 73.8 & 0.8766 & 64.6 / 8.8 & 83.3 & 0.9300 & 78.9 / 9.0 & 87.4 \\
      \bottomrule
    \end{tabular}
  \end{subtable}

  \vspace{10pt}

  \begin{subtable}{\textwidth}
    \centering
    \caption{Youden's $J$ Index Calibration (Maximizing $J = \text{TPR} - \text{FPR}$ on calibration split)}
    \label{tab:ablation-youden-j}
    \setlength{\tabcolsep}{3.5pt}
    \renewcommand{\arraystretch}{1.05}
    \begin{tabular}{@{} c ccc ccc ccc @{}}
      \toprule
      & \multicolumn{3}{c}{Terminal Energy} & \multicolumn{3}{c}{Last-step Gradient Norm} & \multicolumn{3}{c}{\eqcmethod} \\
      \cmidrule(lr){2-4} \cmidrule(lr){5-7} \cmidrule(lr){8-10}
      $m$ & AUROC & TPR / FPR (\%) & Acc (\%) & AUROC & TPR / FPR (\%) & Acc (\%) & AUROC & TPR / FPR (\%) & Acc (\%) \\
      \midrule
      3  & 0.5237 & 21.6 / 13.2 & 67.5 & 0.8734 & 67.6 / 9.0  & 84.0 & 0.8877 & 78.1 / 14.9 & 83.0 \\
      4  & 0.5404 & 22.5 / 12.4 & 68.3 & 0.8707 & 67.7 / 9.5  & 83.8 & 0.8947 & 82.4 / 20.0 & 80.7 \\
      5  & 0.5604 & 42.5 / 32.8 & 59.9 & 0.8713 & 64.3 / 7.8  & 83.9 & 0.9005 & 82.1 / 17.5 & 82.4 \\
      6  & 0.5754 & 24.0 / 9.0  & 71.1 & 0.8724 & 72.1 / 14.1 & 81.8 & 0.9070 & 79.4 / 13.0 & 84.7 \\
      7  & 0.5861 & 28.3 / 11.5 & 70.6 & 0.8737 & 68.6 / 11.0 & 83.0 & 0.9117 & 85.3 / 20.2 & 81.4 \\
      8  & 0.5974 & 34.5 / 17.4 & 68.4 & 0.8737 & 65.3 / 9.1  & 83.3 & 0.9147 & 81.9 / 14.5 & 84.4 \\
      9  & 0.6082 & 36.7 / 16.1 & 69.9 & 0.8744 & 76.0 / 16.4 & 81.4 & 0.9183 & 83.8 / 15.5 & 84.3 \\
      10 & 0.6214 & 33.6 / 11.0 & 72.5 & 0.8758 & 75.8 / 16.3 & 81.4 & 0.9222 & 84.0 / 14.7 & 85.0 \\
      11 & 0.6309 & 48.1 / 25.6 & 66.6 & 0.8769 & 72.5 / 13.4 & 82.4 & 0.9244 & 82.1 / 11.8 & 86.4 \\
      12 & 0.6410 & 49.7 / 24.0 & 68.2 & 0.8766 & 80.7 / 21.6 & 79.1 & 0.9274 & 87.5 / 18.3 & 83.4 \\
      13 & 0.6489 & 44.4 / 18.1 & 70.8 & 0.8766 & 80.4 / 20.9 & 79.5 & 0.9300 & 85.5 / 14.5 & 85.5 \\
      \bottomrule
    \end{tabular}
  \end{subtable}
\end{table}

\section{Theory: Assumptions and Guarantees}
\label{app:theory}

This appendix establishes the theoretical properties underlying Equilibrium Energy Control (EqC). Section~\ref{app:theory-setup} introduces the bridge construction, the energy-based control field, and the gradient-matching objective. Section~\ref{app:path-consistency} relates local gradient matching to finite-path energy consistency. Section~\ref{app:stationary-representation} characterizes when the time-dependent bridge target can be represented by a stationary energy field. Section~\ref{app:energy-ascent} establishes monotone energy ascent during inference. Section~\ref{app:energy-consistency} shows that the deployment-time consistency score measures integrated field mismatch along the realized refinement trajectory.

\subsection{Problem Setup}
\label{app:theory-setup}

Let $\mathcal{Z}$ denote the condition space and let
\[
(A,Z)\sim p_{\mathrm{data}}
\]
denote an action--condition pair, where $A\in\mathbb{R}^D$ is an expert action chunk and $Z\in\mathcal{Z}$ contains the corresponding observations, language instruction, robot state, or other policy inputs. Let
\[
\varepsilon\sim\mathcal{N}(0,I_D)
\]
be independent of $(A,Z)$. For a realization $(a,z,\varepsilon)$, define
\begin{equation}
d:=a-\varepsilon,
\qquad
x_t:=\varepsilon+td,
\qquad
t\in[0,1].
\label{eq:appendix-bridge}
\end{equation}
Thus, $x_0=\varepsilon$ and $x_1=a$.

Let
\[
c:[0,1]\to\mathbb{R}_{\geq 0}
\]
be a continuous schedule satisfying $c(1)=0$, and define the remaining schedule integral
\begin{equation}
C(t):=\int_t^1c(s)\,ds.
\label{eq:appendix-C-definition}
\end{equation}

EqC represents action quality by a differentiable scalar energy
\[
E_\theta:\mathbb{R}^D\times\mathcal{Z}\to\mathbb{R}.
\]
The induced stationary control field is
\begin{equation}
g_\theta(x,z):=\nabla_xE_\theta(x,z).
\label{eq:appendix-energy-field}
\end{equation}
The field is stationary in the sense that it depends on the current action state $x$ and condition $z$, but not explicitly on the bridge coordinate $t$.

At the bridge point $x_t$, the target field is
\[
c(t)d.
\]
The gradient-matching objective is
\begin{equation}
\mathcal{L}_{\mathrm{grad}}(\theta)
:=
\mathbb{E}_{(A,Z),\varepsilon,T}
\left[
\frac{1}{D}
\left\|
\nabla_xE_\theta(X_T,Z)-c(T)(A-\varepsilon)
\right\|_2^2
\right],
\label{eq:appendix-gradient-loss}
\end{equation}
where
\[
T\sim\mathrm{Uniform}[0,1],
\qquad
X_T:=\varepsilon+T(A-\varepsilon).
\]

The corresponding finite-path target follows from integrating the target field along the bridge:
\[
\int_t^1(c(s)d)^\top d\,ds
=
\|d\|_2^2C(t).
\]
This relation is induced by the gradient target and will provide the reference for the deployment-time consistency
score.

At inference time, EqC initializes
\[
x^{(0)}\sim\mathcal{N}(0,I_D)
\]
and applies gradient ascent:
\begin{equation}
x^{(k+1)}
=
x^{(k)}
+
\eta_k\nabla_xE_\theta(x^{(k)},z),
\qquad
k=0,\ldots,K-1.
\label{eq:appendix-inference-update}
\end{equation}
The final action prediction is
\[
\hat{x}:=x^{(K)}.
\]

\subsection{From Local Gradient Matching to Path Consistency}
\label{app:path-consistency}

Fix a bridge $(x_s)_{s\in[0,1]}$ defined by
\eqref{eq:appendix-bridge}. Define the local gradient-matching error
\begin{equation}
e_\theta(s)
:=
\nabla_xE_\theta(x_s,z)-c(s)d.
\label{eq:local-gradient-error}
\end{equation}

The following result relates finite energy inconsistency to local gradient-matching error accumulated along the bridge.

\begin{lemma}[Path-consistency identity]
\label{lem:path-consistency}
Suppose that $E_\theta(\cdot,z)$ is continuously differentiable on the bridge
\[
\{x_s:s\in[t,1]\}.
\]
Then, for every $t\in[0,1]$,
\begin{equation}
E_\theta(a,z)-E_\theta(x_t,z)-\|d\|_2^2C(t)
=
\int_t^1e_\theta(s)^\top d\,ds.
\label{eq:consistency-error-identity}
\end{equation}
\end{lemma}

\begin{proof}
Since
\[
x_s=\varepsilon+sd,
\]
we have
\[
\frac{d x_s}{ds}=d.
\]
Applying the chain rule,
\begin{equation}
\frac{d}{ds}E_\theta(x_s,z)
=
\nabla_xE_\theta(x_s,z)^\top d.
\label{eq:path-chain-rule}
\end{equation}
By \eqref{eq:local-gradient-error},
\[
\nabla_xE_\theta(x_s,z)
=
c(s)d+e_\theta(s).
\]
Substituting this identity into \eqref{eq:path-chain-rule} gives
\[
\frac{d}{ds}E_\theta(x_s,z)
=
c(s)\|d\|_2^2+e_\theta(s)^\top d.
\]
Integrating from $t$ to $1$ yields
\begin{align*}
E_\theta(x_1,z)-E_\theta(x_t,z)
&=
\int_t^1c(s)\|d\|_2^2\,ds
+
\int_t^1e_\theta(s)^\top d\,ds\\
&=
\|d\|_2^2C(t)
+
\int_t^1e_\theta(s)^\top d\,ds.
\end{align*}
Since $x_1=a$, rearranging proves
\eqref{eq:consistency-error-identity}.
\end{proof}

\begin{corollary}[Accumulated path-consistency error]
\label{cor:path-consistency-bound}
Under the assumptions of Lemma~\ref{lem:path-consistency},
\begin{equation}
\left|
E_\theta(a,z)-E_\theta(x_t,z)-\|d\|_2^2C(t)
\right|
\leq
\|d\|_2\sqrt{1-t}
\left(
\int_t^1\|e_\theta(s)\|_2^2\,ds
\right)^{1/2}.
\label{eq:consistency-error-bound}
\end{equation}
\end{corollary}

\begin{proof}
By Lemma~\ref{lem:path-consistency},
\[
\left|
E_\theta(a,z)-E_\theta(x_t,z)-\|d\|_2^2C(t)
\right|
=
\left|
\int_t^1e_\theta(s)^\top d\,ds
\right|.
\]
Cauchy--Schwarz first gives
\[
|e_\theta(s)^\top d|
\leq
\|e_\theta(s)\|_2\|d\|_2.
\]
Therefore,
\begin{align*}
\left|
\int_t^1e_\theta(s)^\top d\,ds
\right|
&\leq
\|d\|_2
\int_t^1\|e_\theta(s)\|_2\,ds\\
&\leq
\|d\|_2
\left(\int_t^1 1\,ds\right)^{1/2}
\left(
\int_t^1\|e_\theta(s)\|_2^2\,ds
\right)^{1/2}\\
&=
\|d\|_2\sqrt{1-t}
\left(
\int_t^1\|e_\theta(s)\|_2^2\,ds
\right)^{1/2}.
\end{align*}
\end{proof}

Lemma~\ref{lem:path-consistency} shows that the consistency residual is the integrated component of the local field
error in the bridge direction. Consequently, reducing the gradient-matching error also controls the finite energy
relation underlying the deployment-time consistency test.

\subsection{Stationary Representation of the Bridge Target}
\label{app:stationary-representation}

The bridge target $c(T)(A-\varepsilon)$ depends explicitly on $T$, whereas the deployed EqC field $\nabla_xE_\theta(X,Z)$ is stationary. We now characterize the approximation cost associated with this restriction.

Let
\[
X:=\varepsilon+T(A-\varepsilon),
\qquad
Y:=c(T)(A-\varepsilon).
\]
For $T<1$, define
\begin{equation}
\alpha(T):=\frac{c(T)}{1-T}.
\label{eq:alpha-definition}
\end{equation}
The value of $\alpha(1)$ may be assigned arbitrarily because
$\mathbb{P}(T=1)=0$.

Since
\[
A-X
=
A-\varepsilon-T(A-\varepsilon)
=
(1-T)(A-\varepsilon),
\]
the target satisfies
\begin{equation}
Y=\alpha(T)(A-X)
\qquad\text{almost surely}.
\label{eq:bridge-restoring-identity}
\end{equation}

Assume
\[
\mathbb{E}\|A\|_2^2<\infty,
\]
and define the conditional mean action
\begin{equation}
a_\star(Z):=\mathbb{E}[A\mid Z].
\label{eq:conditional-mean-action}
\end{equation}
Define the radius of the current state relative to this conditional reference by
\begin{equation}
R_X:=\|X-a_\star(Z)\|_2.
\label{eq:radial-state}
\end{equation}

\begin{definition}[Condition-dependent radial scale]
\label{def:radial-scale}
A condition-dependent radial scale is a measurable function
\[
\lambda:\mathcal{Z}\times[0,\infty)\to[0,\infty),
\]
written
\[
\lambda_z(r):=\lambda(z,r),
\]
such that $r\mapsto\lambda_z(r)$ is continuous for
$p_Z$-almost every $z$.

The random variable $\lambda_Z(R_X)$ represents a stationary approximation of the bridge coefficient $\alpha(T)$ using only the current state $X$ and condition $Z$.
\end{definition}

For a radial scale $\lambda$, define the per-dimension action-dispersion and scale-recovery errors by
\begin{align}
\sigma_{\mathrm{act}}^2
&:=
\frac{1}{D}
\mathbb{E}
\left[
\alpha(T)^2
\|A-a_\star(Z)\|_2^2
\right],
\label{eq:action-dispersion}\\
\sigma_{\mathrm{scale}}^2(\lambda)
&:=
\frac{1}{D}
\mathbb{E}
\left[
\bigl(\alpha(T)-\lambda_Z(R_X)\bigr)^2
R_X^2
\right].
\label{eq:scale-dispersion}
\end{align}

The first quantity measures the residual variation of the expert action after conditioning on $Z$. Since $T$ is independent of $(A,Z)$,
\begin{equation}
D\sigma_{\mathrm{act}}^2
=
\mathbb{E}[\alpha(T)^2]\,
\mathbb{E}\|A-a_\star(Z)\|_2^2
\label{eq:action-dispersion-factorization}
\end{equation}
whenever both factors are finite. The second quantity measures the error incurred when $\alpha(T)$ is represented as a function of $(R_X,Z)$.

\begin{lemma}[Optimal measurable radial scale]
\label{lem:optimal-radial-scale}
Among all measurable functions of $(R_X,Z)$, the weighted error
\[
\mathbb{E}
\left[
\bigl(\alpha(T)-\lambda_Z(R_X)\bigr)^2R_X^2
\right]
\]
is minimized, for $R_X>0$, by any version of
\begin{equation}
\lambda_z^\star(r)
:=
\mathbb{E}
\left[
\alpha(T)\mid R_X=r,Z=z
\right].
\label{eq:optimal-radial-scale}
\end{equation}
At $r=0$, the value of $\lambda_z^\star(0)$ does not affect the weighted objective.
\end{lemma}

\begin{proof}
Condition on $(R_X,Z)=(r,z)$. For any scalar $\ell$,
\begin{align*}
&\mathbb{E}
\left[
\bigl(\alpha(T)-\ell\bigr)^2R_X^2
\mid R_X=r,Z=z
\right]\\
&\hspace{2cm}
=
r^2
\mathbb{E}
\left[
\bigl(\alpha(T)-\ell\bigr)^2
\mid R_X=r,Z=z
\right].
\end{align*}
For $r>0$, multiplication by $r^2$ does not change the minimizer. The conditional squared error is minimized by the conditional mean
\[
\ell
=
\mathbb{E}
\left[
\alpha(T)\mid R_X=r,Z=z
\right].
\]
For $r=0$, the weighted conditional objective is identically zero for every $\ell$.
\end{proof}

If the conditional expectation in
\eqref{eq:optimal-radial-scale} has a continuous version, it satisfies Definition~\ref{def:radial-scale}. Otherwise, any continuous approximation may be used in the bound below.

Given a radial scale $\lambda$, define
\begin{equation}
E_\lambda(x,z)
:=
-\int_0^{\|x-a_\star(z)\|_2}s\lambda_z(s)\,ds.
\label{eq:radial-witness-energy}
\end{equation}

\begin{lemma}[Gradient of the radial witness]
\label{lem:radial-witness-gradient}
If $r\mapsto\lambda_z(r)$ is continuous, then
$E_\lambda(\cdot,z)$ is continuously differentiable and
\begin{equation}
\nabla_xE_\lambda(x,z)
=
\lambda_z(\|x-a_\star(z)\|_2)
\bigl(a_\star(z)-x\bigr).
\label{eq:radial-energy-witness}
\end{equation}
\end{lemma}

\begin{proof}
Fix $z$ and write
\[
r(x):=\|x-a_\star(z)\|_2.
\]
For $r(x)>0$, the chain rule gives
\begin{align*}
\nabla_xE_\lambda(x,z)
&=
-r(x)\lambda_z(r(x))\nabla_xr(x)\\
&=
-r(x)\lambda_z(r(x))
\frac{x-a_\star(z)}{r(x)}\\
&=
\lambda_z(r(x))(a_\star(z)-x).
\end{align*}

It remains to consider $x=a_\star(z)$. By continuity, $\lambda_z$ is bounded on a neighborhood of zero. Hence, for sufficiently small $h$, there exists $M<\infty$ such that
\begin{align*}
\left|
E_\lambda(a_\star(z)+h,z)
-
E_\lambda(a_\star(z),z)
\right|
&=
\left|
\int_0^{\|h\|_2}s\lambda_z(s)\,ds
\right|\\
&\leq
\frac{M}{2}\|h\|_2^2\\
&=
o(\|h\|_2).
\end{align*}
Therefore,
\[
\nabla_xE_\lambda(a_\star(z),z)=0,
\]
which agrees with \eqref{eq:radial-energy-witness}. Continuity of the gradient follows from continuity of $\lambda_z$.
\end{proof}

The field
\[
g_\lambda(x,z):=\nabla_xE_\lambda(x,z)
\]
is a stationary conservative field. We use it as an explicit witness for the approximation capacity of the EqC energy class.

For a measurable predictor
\[
h:\mathbb{R}^D\times[0,1]\times\mathcal{Z}
\to\mathbb{R}^D,
\]
define its normalized squared risk by
\begin{equation}
R(h)
:=
\frac{1}{D}
\mathbb{E}
\left[
\|Y-h(X,T,Z)\|_2^2
\right].
\label{eq:risk-definition}
\end{equation}
The optimal time-conditioned regressor is
\begin{equation}
m_{\mathrm{tc}}(X,T,Z)
:=
\mathbb{E}[Y\mid X,T,Z],
\label{eq:time-conditioned-regressor}
\end{equation}
with risk
\begin{equation}
R_{\mathrm{tc}}^\star
:=
R(m_{\mathrm{tc}}).
\label{eq:time-conditioned-risk}
\end{equation}

A stationary predictor depends only on $(X,Z)$. Define the optimal stationary risk by
\begin{equation}
R_{\mathrm{st}}^\star
:=
\inf_{h:\,h=h(X,Z)}
\frac{1}{D}
\mathbb{E}
\left[
\|Y-h(X,Z)\|_2^2
\right].
\label{eq:stationary-risk}
\end{equation}

For the EqC energy class
\[
\mathcal{E}_\Theta
:=
\{E_\theta:\theta\in\Theta\},
\]
define
\begin{equation}
R_{\mathrm{EqC}}^\Theta
:=
\inf_{\theta\in\Theta}
\frac{1}{D}
\mathbb{E}
\left[
\|Y-\nabla_xE_\theta(X,Z)\|_2^2
\right].
\label{eq:eqc-population-risk}
\end{equation}
Finally, define the approximation error of the EqC class relative to the radial witness:
\begin{equation}
\delta_\Theta(\lambda)
:=
\inf_{\theta\in\Theta}
\left(
\frac{1}{D}
\mathbb{E}
\left[
\left\|
\nabla_xE_\theta(X,Z)-g_\lambda(X,Z)
\right\|_2^2
\right]
\right)^{1/2}.
\label{eq:energy-class-approximation}
\end{equation}

\begin{proposition}[Stationary energy approximation]
\label{prop:stationary-energy-approximation}
Suppose that
\[
\sigma_{\mathrm{act}}<\infty,
\qquad
\sigma_{\mathrm{scale}}(\lambda)<\infty,
\qquad
\delta_\Theta(\lambda)<\infty.
\]
Then
\begin{equation}
0
\leq
R_{\mathrm{EqC}}^\Theta-R_{\mathrm{st}}^\star
\leq
R_{\mathrm{EqC}}^\Theta-R_{\mathrm{tc}}^\star
\leq
\left(
\sigma_{\mathrm{act}}
+
\sigma_{\mathrm{scale}}(\lambda)
+
\delta_\Theta(\lambda)
\right)^2.
\label{eq:conditional-concentration-bound}
\end{equation}
\end{proposition}

\begin{proof}
For a square-integrable random vector $W$, define
\[
\|W\|_{L_D^2}
:=
\left(
\frac{1}{D}
\mathbb{E}\|W\|_2^2
\right)^{1/2}.
\]

The orthogonality principle for conditional expectations gives, for every measurable $h(X,T,Z)$,
\begin{equation}
R(h)-R_{\mathrm{tc}}^\star
=
\|h-m_{\mathrm{tc}}\|_{L_D^2}^2.
\label{eq:regression-orthogonality}
\end{equation}
Indeed,
\begin{align*}
\|Y-h\|_{L_D^2}^2
&=
\|Y-m_{\mathrm{tc}}+m_{\mathrm{tc}}-h\|_{L_D^2}^2\\
&=
\|Y-m_{\mathrm{tc}}\|_{L_D^2}^2
+
\|m_{\mathrm{tc}}-h\|_{L_D^2}^2,
\end{align*}
because
\[
\mathbb{E}
\left[
(Y-m_{\mathrm{tc}})^\top(m_{\mathrm{tc}}-h)
\right]
=
0.
\]

Since $g_\lambda(X,Z)$ is measurable with respect to $(X,T,Z)$,
\begin{align}
\|m_{\mathrm{tc}}-g_\lambda\|_{L_D^2}
&=
\left\|
\mathbb{E}
\left[
Y-g_\lambda(X,Z)
\mid X,T,Z
\right]
\right\|_{L_D^2}
\nonumber\\
&\leq
\|Y-g_\lambda(X,Z)\|_{L_D^2},
\label{eq:conditional-jensen-bound}
\end{align}
where the inequality follows from conditional Jensen.

Using \eqref{eq:bridge-restoring-identity} and
\eqref{eq:radial-energy-witness},
\begin{align}
Y-g_\lambda(X,Z)
&=
\alpha(T)(A-X)
-
\lambda_Z(R_X)(a_\star(Z)-X)
\nonumber\\
&=
\alpha(T)(A-a_\star(Z))
+
\bigl(\alpha(T)-\lambda_Z(R_X)\bigr)
(a_\star(Z)-X).
\label{eq:witness-error-decomposition}
\end{align}
Applying Minkowski's inequality,
\begin{align*}
\|Y-g_\lambda(X,Z)\|_{L_D^2}
&\leq
\left\|
\alpha(T)(A-a_\star(Z))
\right\|_{L_D^2}\\
&\quad+
\left\|
\bigl(\alpha(T)-\lambda_Z(R_X)\bigr)
(a_\star(Z)-X)
\right\|_{L_D^2}\\
&=
\sigma_{\mathrm{act}}
+
\sigma_{\mathrm{scale}}(\lambda),
\end{align*}
where
\[
\|a_\star(Z)-X\|_2=R_X.
\]
Together with \eqref{eq:conditional-jensen-bound}, this gives
\begin{equation}
\|m_{\mathrm{tc}}-g_\lambda\|_{L_D^2}
\leq
\sigma_{\mathrm{act}}
+
\sigma_{\mathrm{scale}}(\lambda).
\label{eq:conditional-mean-witness-bound}
\end{equation}

For any $\theta\in\Theta$, Minkowski's inequality and
\eqref{eq:conditional-mean-witness-bound} imply
\begin{align*}
\|m_{\mathrm{tc}}-\nabla_xE_\theta\|_{L_D^2}
&\leq
\|m_{\mathrm{tc}}-g_\lambda\|_{L_D^2}
+
\|g_\lambda-\nabla_xE_\theta\|_{L_D^2}\\
&\leq
\sigma_{\mathrm{act}}
+
\sigma_{\mathrm{scale}}(\lambda)
+
\|g_\lambda-\nabla_xE_\theta\|_{L_D^2}.
\end{align*}
Taking the infimum over $\theta$ gives
\[
\inf_{\theta\in\Theta}
\|m_{\mathrm{tc}}-\nabla_xE_\theta\|_{L_D^2}
\leq
\sigma_{\mathrm{act}}
+
\sigma_{\mathrm{scale}}(\lambda)
+
\delta_\Theta(\lambda).
\]
By \eqref{eq:regression-orthogonality},
\[
R_{\mathrm{EqC}}^\Theta-R_{\mathrm{tc}}^\star
=
\inf_{\theta\in\Theta}
\|m_{\mathrm{tc}}-\nabla_xE_\theta\|_{L_D^2}^2.
\]
Therefore,
\[
R_{\mathrm{EqC}}^\Theta-R_{\mathrm{tc}}^\star
\leq
\left(
\sigma_{\mathrm{act}}
+
\sigma_{\mathrm{scale}}(\lambda)
+
\delta_\Theta(\lambda)
\right)^2.
\]

Every EqC gradient $\nabla_xE_\theta(X,Z)$ is a stationary predictor. Hence,
\[
R_{\mathrm{st}}^\star
\leq
R_{\mathrm{EqC}}^\Theta.
\]
Moreover, every stationary predictor is contained in the class of time-conditioned predictors, so
\[
R_{\mathrm{tc}}^\star
\leq
R_{\mathrm{st}}^\star.
\]
Combining these inequalities completes the proof.
\end{proof}

Proposition~\ref{prop:stationary-energy-approximation} separates the stationary approximation error into three terms: conditional action dispersion, recovery of the bridge scale from the current state, and approximation by the chosen neural energy class.

\begin{corollary}[Exact stationary representation]
\label{cor:exact-stationary-representation}
Suppose that the following conditions hold:
\begin{align}
&A=a_\star(Z)
&&\text{almost surely},
\label{eq:exact-action-concentration}\\
&\operatorname{Var}\!\left(
\alpha(T)\mid R_X,Z
\right)=0
&&\text{almost surely}.
\label{eq:exact-scale-recovery}
\end{align}
Let
\[
\lambda_z(r)
:=
\mathbb{E}
\left[
\alpha(T)\mid R_X=r,Z=z
\right]
\]
be a continuous version of the conditional expectation. If there exists
$\theta_\star\in\Theta$ such that
\begin{equation}
\nabla_xE_{\theta_\star}(x,z)
=
\lambda_z(\|x-a_\star(z)\|_2)
\bigl(a_\star(z)-x\bigr)
\label{eq:exact-energy-realizability}
\end{equation}
on the support of $(X,Z)$, then
\begin{equation}
\nabla_xE_{\theta_\star}(X,Z)=Y
\qquad\text{almost surely}.
\label{eq:exact-target-representation}
\end{equation}
Consequently,
\begin{equation}
R_{\mathrm{EqC}}^\Theta
=
R_{\mathrm{st}}^\star
=
R_{\mathrm{tc}}^\star
=
0.
\label{eq:exact-risk-equality}
\end{equation}
\end{corollary}

\begin{proof}
~\eqref{eq:exact-scale-recovery} implies that
$\alpha(T)$ is measurable with respect to $(R_X,Z)$. Therefore,
\[
\alpha(T)
=
\mathbb{E}
\left[
\alpha(T)\mid R_X,Z
\right]
=
\lambda_Z(R_X)
\qquad\text{almost surely}.
\]
Using \eqref{eq:bridge-restoring-identity} and
\eqref{eq:exact-action-concentration},
\begin{align*}
Y
&=
\alpha(T)(A-X)\\
&=
\lambda_Z(R_X)(a_\star(Z)-X).
\end{align*}
By \eqref{eq:exact-energy-realizability},
\[
\lambda_Z(R_X)(a_\star(Z)-X)
=
\nabla_xE_{\theta_\star}(X,Z).
\]
Hence,
\[
Y=\nabla_xE_{\theta_\star}(X,Z)
\]
almost surely. The corresponding squared risk is zero, which proves
\eqref{eq:exact-risk-equality}.
\end{proof}

In the concentrated case $A=a_\star(Z)$,
\begin{equation}
R_X
=
(1-T)\|\varepsilon-a_\star(Z)\|_2.
\label{eq:radius-progress-relation}
\end{equation}
Thus, whenever $(R_X,Z)$ determines $\alpha(T)$, a stationary energy can represent a nonconstant bridge schedule without receiving $T$ explicitly.

\subsection{Monotone Energy Ascent}
\label{app:energy-ascent}

We next establish the optimization property of the inference update
\eqref{eq:appendix-inference-update}.

\begin{assumption}[Local smoothness]
\label{ass:local-smoothness}
Fix $z$. There exists $L>0$ such that, for every
$k\in\{0,\ldots,K-1\}$,
\begin{equation}
\left\|
\nabla_xE_\theta(x,z)-\nabla_xE_\theta(y,z)
\right\|_2
\leq
L\|x-y\|_2
\label{eq:gradient-lipschitz}
\end{equation}
for all $x$ and $y$ on the segment
\[
[x^{(k)},x^{(k+1)}]
:=
\left\{
x^{(k)}+s(x^{(k+1)}-x^{(k)}):
s\in[0,1]
\right\}.
\]
\end{assumption}

\begin{theorem}[Monotone energy ascent]
\label{thm:monotone-ascent}
Suppose Assumption~\ref{ass:local-smoothness} holds and EqC uses the exact gradient-ascent update
\[
x^{(k+1)}
=
x^{(k)}
+
\eta_k\nabla_xE_\theta(x^{(k)},z)
\]
with
\begin{equation}
0<\eta_k<\frac{2}{L}.
\label{eq:stable-step-size}
\end{equation}
Then
\begin{equation}
E_\theta(x^{(k+1)},z)-E_\theta(x^{(k)},z)
\geq
w_k
\left\|
\nabla_xE_\theta(x^{(k)},z)
\right\|_2^2,
\label{eq:monotone-energy-ascent}
\end{equation}
where
\begin{equation}
w_k
:=
\eta_k
\left(
1-\frac{L\eta_k}{2}
\right)>0.
\label{eq:ascent-weight}
\end{equation}
Therefore,
\[
E_\theta(x^{(k+1)},z)
\geq
E_\theta(x^{(k)},z),
\]
with strict inequality whenever
\[
\nabla_xE_\theta(x^{(k)},z)\neq 0.
\]
\end{theorem}

\begin{proof}
Let
\[
g_k:=\nabla_xE_\theta(x^{(k)},z).
\]
The update is
\[
x^{(k+1)}=x^{(k)}+\eta_kg_k.
\]
Define the segment
\[
q_k(s):=x^{(k)}+s\eta_kg_k,
\qquad
s\in[0,1].
\]
By the fundamental theorem of calculus,
\begin{align}
E_\theta(x^{(k+1)},z)-E_\theta(x^{(k)},z)
&=
\int_0^1
\nabla_xE_\theta(q_k(s),z)^\top
(\eta_kg_k)\,ds
\nonumber\\
&=
\eta_k\|g_k\|_2^2
+
\eta_k
\int_0^1
\left(
\nabla_xE_\theta(q_k(s),z)-g_k
\right)^\top g_k\,ds.
\label{eq:ascent-integral-decomposition}
\end{align}
By Cauchy--Schwarz and Assumption~\ref{ass:local-smoothness},
\begin{align*}
\left(
\nabla_xE_\theta(q_k(s),z)-g_k
\right)^\top g_k
&\geq
-
\left\|
\nabla_xE_\theta(q_k(s),z)-g_k
\right\|_2
\|g_k\|_2\\
&\geq
-L\|q_k(s)-x^{(k)}\|_2\|g_k\|_2\\
&=
-Ls\eta_k\|g_k\|_2^2.
\end{align*}
Substituting this bound into
\eqref{eq:ascent-integral-decomposition},
\begin{align*}
E_\theta(x^{(k+1)},z)-E_\theta(x^{(k)},z)
&\geq
\eta_k\|g_k\|_2^2
-
L\eta_k^2\|g_k\|_2^2
\int_0^1s\,ds\\
&=
\eta_k
\left(
1-\frac{L\eta_k}{2}
\right)
\|g_k\|_2^2\\
&=
w_k\|g_k\|_2^2.
\end{align*}
The step-size condition \eqref{eq:stable-step-size} implies
$w_k>0$, proving the result.
\end{proof}

\begin{corollary}[Finite-step stationarity]
\label{cor:finite-step-stationarity}
Under the assumptions of Theorem~\ref{thm:monotone-ascent},
\begin{equation}
\min_{0\leq k<K}
\left\|
\nabla_xE_\theta(x^{(k)},z)
\right\|_2^2
\leq
\frac{
E_\theta(x^{(K)},z)-E_\theta(x^{(0)},z)
}{
\sum_{k=0}^{K-1}w_k
}.
\label{eq:finite-step-stationarity}
\end{equation}
If
\[
E_\theta(x^{(k)},z)\leq E_{\sup}
\qquad
\text{for }k=0,\ldots,K,
\]
then
\begin{equation}
\min_{0\leq k<K}
\left\|
\nabla_xE_\theta(x^{(k)},z)
\right\|_2^2
\leq
\frac{
E_{\sup}-E_\theta(x^{(0)},z)
}{
\sum_{k=0}^{K-1}w_k
}.
\label{eq:bounded-energy-stationarity}
\end{equation}
\end{corollary}

\begin{proof}
Summing \eqref{eq:monotone-energy-ascent} over
$k=0,\ldots,K-1$ gives
\begin{align*}
E_\theta(x^{(K)},z)-E_\theta(x^{(0)},z)
&=
\sum_{k=0}^{K-1}
\left(
E_\theta(x^{(k+1)},z)
-
E_\theta(x^{(k)},z)
\right)\\
&\geq
\sum_{k=0}^{K-1}
w_k
\left\|
\nabla_xE_\theta(x^{(k)},z)
\right\|_2^2\\
&\geq
\left(
\min_{0\leq k<K}
\left\|
\nabla_xE_\theta(x^{(k)},z)
\right\|_2^2
\right)
\sum_{k=0}^{K-1}w_k.
\end{align*}
Since each $w_k>0$, division by
$\sum_{k=0}^{K-1}w_k$ proves
\eqref{eq:finite-step-stationarity}. The bound
\eqref{eq:bounded-energy-stationarity} follows from
\[
E_\theta(x^{(K)},z)\leq E_{\sup}.
\]
\end{proof}

Theorem~\ref{thm:monotone-ascent} establishes that every EqC refinement step increases the same scalar energy under the stated smoothness and step-size conditions. Corollary~\ref{cor:finite-step-stationarity} quantifies the resulting approach to a first-order stationary point.

\subsection{Energy Consistency as a Path-Wise Mismatch Certificate}
\label{app:energy-consistency}

Let
\[
\hat{x}:=x^{(K)},
\qquad
v:=\hat{x}-x^{(0)}.
\]
For a fixed normalization constant $\delta>0$, define
\begin{equation}
D_v:=\max(\|v\|_2^2,\delta),
\qquad
\gamma_v:=\frac{\|v\|_2^2}{D_v}.
\label{eq:chord-normalization}
\end{equation}
For $u\in\mathbb{R}$, define
\[
\operatorname{clip}(u,0,1)
:=
\min\{1,\max\{0,u\}\}.
\]
The projected progress coordinate of iterate $x^{(k)}$ is
\begin{equation}
\tau_k
:=
\operatorname{clip}
\left(
\frac{
\langle x^{(k)}-x^{(0)},v\rangle
}{D_v},
0,1
\right).
\label{eq:projected-progress-coordinate}
\end{equation}
The corresponding point on the chord from $x^{(0)}$ to $\hat{x}$ is
\begin{equation}
\bar{x}^{(k)}
:=
x^{(0)}+\tau_kv,
\label{eq:chord-point}
\end{equation}
and the off-chord displacement is
\begin{equation}
h_k
:=
x^{(k)}-\bar{x}^{(k)}.
\label{eq:off-chord-displacement}
\end{equation}

Define the field mismatch along the endpoint chord by
\begin{equation}
e_v(s)
:=
\nabla_xE_\theta(x^{(0)}+sv,z)-c(s)v,
\qquad
s\in[0,1].
\label{eq:chord-field-error}
\end{equation}

The deployment-time energy residual is
\begin{equation}
r_k^{\Delta E}
:=
\frac{
E_\theta(\hat{x},z)-E_\theta(x^{(k)},z)
}{D_v}
-
C(\tau_k),
\label{eq:deployment-energy-residual}
\end{equation}
and the energy-consistency score is

\begin{equation}
U_{\Delta E}
:=
\frac{1}{K}
\sum_{k=0}^{K-1}
|r_k^{\Delta E}|.
\label{eq:deployment-energy-score}
\end{equation}

\begin{theorem}[Energy-consistency decomposition]
\label{thm:energy-consistency-decomposition}
Suppose that $E_\theta(\cdot,z)$ is continuously differentiable on the chord
\[
\{x^{(0)}+sv:s\in[0,1]\}.
\]
Then, for each $k\in\{0,\ldots,K-1\}$,
\begin{equation}
r_k^{\Delta E}
=
\frac{1}{D_v}
\int_{\tau_k}^1e_v(s)^\top v\,ds
+
\frac{
E_\theta(\bar{x}^{(k)},z)
-
E_\theta(x^{(k)},z)
}{D_v}
+
(\gamma_v-1)C(\tau_k).
\label{eq:energy-consistency-residual-decomposition}
\end{equation}
\end{theorem}

\begin{proof}
Define the chord
\[
q(s):=x^{(0)}+sv.
\]
Since
\[
q(\tau_k)=\bar{x}^{(k)},
\qquad
q(1)=\hat{x},
\qquad
\frac{dq(s)}{ds}=v,
\]
the fundamental theorem of calculus gives
\begin{align*}
E_\theta(\hat{x},z)-E_\theta(\bar{x}^{(k)},z)
&=
\int_{\tau_k}^1
\nabla_xE_\theta(q(s),z)^\top v\,ds\\
&=
\int_{\tau_k}^1
\bigl(c(s)v+e_v(s)\bigr)^\top v\,ds\\
&=
\|v\|_2^2
\int_{\tau_k}^1c(s)\,ds
+
\int_{\tau_k}^1e_v(s)^\top v\,ds\\
&=
\|v\|_2^2C(\tau_k)
+
\int_{\tau_k}^1e_v(s)^\top v\,ds.
\end{align*}
Adding
\[
E_\theta(\bar{x}^{(k)},z)-E_\theta(x^{(k)},z)
\]
to both sides gives
\begin{align*}
E_\theta(\hat{x},z)-E_\theta(x^{(k)},z)
&=
\|v\|_2^2C(\tau_k)
+
\int_{\tau_k}^1e_v(s)^\top v\,ds\\
&\quad+
E_\theta(\bar{x}^{(k)},z)
-
E_\theta(x^{(k)},z).
\end{align*}
Dividing by $D_v$ and subtracting $C(\tau_k)$ yields
\begin{align*}
r_k^{\Delta E}
&=
\frac{1}{D_v}
\int_{\tau_k}^1e_v(s)^\top v\,ds\\
&\quad+
\frac{
E_\theta(\bar{x}^{(k)},z)
-
E_\theta(x^{(k)},z)
}{D_v}\\
&\quad+
\left(
\frac{\|v\|_2^2}{D_v}-1
\right)C(\tau_k).
\end{align*}
Using
\[
\gamma_v=\frac{\|v\|_2^2}{D_v}
\]
proves \eqref{eq:energy-consistency-residual-decomposition}.
\end{proof}

The three terms in
\eqref{eq:energy-consistency-residual-decomposition} have distinct origins. The first is the integrated field mismatch along the endpoint chord. The second measures the energy discrepancy caused by deviation of the realized iterate from that chord. The third is the correction introduced by the normalization floor.

Define the mean integrated chord mismatch
\begin{equation}
A_K
:=
\frac{1}{KD_v}
\sum_{k=0}^{K-1}
\left|
\int_{\tau_k}^1e_v(s)^\top v\,ds
\right|.
\label{eq:mean-integrated-field-mismatch}
\end{equation}

\begin{assumption}[Bounded field on off-chord segments]
\label{ass:bounded-off-chord-field}
There exists $G<\infty$ such that
\[
\|\nabla_xE_\theta(y,z)\|_2\leq G
\]
for every $y$ on every segment
\[
[x^{(k)},\bar{x}^{(k)}],
\qquad
k=0,\ldots,K-1.
\]
\end{assumption}

Under Assumption~\ref{ass:bounded-off-chord-field}, define
\begin{equation}
B_K:=
\frac{G}{KD_v}\sum_{k=0}^{K-1}\|h_k\|_2
+
\frac{|\gamma_v-1|}{K}
\sum_{k=0}^{K-1}|C(\tau_k)|.
\label{eq:off-chord-floor-bound}
\end{equation}

\begin{proposition}[Bounds on the energy-consistency score]
\label{prop:energy-consistency-bounds}
Under Assumption~\ref{ass:bounded-off-chord-field},
\begin{equation}
[A_K-B_K]_+
\leq
U_{\Delta E}
\leq
A_K+B_K,
\label{eq:energy-consistency-absolute-bound}
\end{equation}
where
\[
[q]_+:=\max(q,0).
\]
Moreover,
\begin{equation}
A_K
\leq
\frac{\|v\|_2}{KD_v}
\int_0^1
N_K(s)\|e_v(s)\|_2\,ds,
\label{eq:energy-consistency-field-bound}
\end{equation}
where
\begin{equation}
N_K(s)
:=
\sum_{k=0}^{K-1}
\mathbf{1}\{\tau_k\leq s\}.
\label{eq:remaining-interval-count}
\end{equation}
\end{proposition}

\begin{proof}
For each $k$, write
\begin{align*}
a_k
&:=
\frac{1}{D_v}
\int_{\tau_k}^1e_v(s)^\top v\,ds,\\
b_k
&:=
\frac{
E_\theta(\bar{x}^{(k)},z)
-
E_\theta(x^{(k)},z)
}{D_v},\\
q_k
&:=
(\gamma_v-1)C(\tau_k).
\end{align*}
By Theorem~\ref{thm:energy-consistency-decomposition},
\[
r_k^{\Delta E}=a_k+b_k+q_k.
\]
The triangle and reverse-triangle inequalities give
\begin{equation}
|a_k|-|b_k|-|q_k|
\leq
|r_k^{\Delta E}|
\leq
|a_k|+|b_k|+|q_k|.
\label{eq:per-iterate-residual-bound}
\end{equation}

By the fundamental theorem of calculus on the segment from
$x^{(k)}$ to $\bar{x}^{(k)}$,
\begin{align*}
&\left|
E_\theta(\bar{x}^{(k)},z)
-
E_\theta(x^{(k)},z)
\right|\\
&\quad\leq
\sup_{y\in[x^{(k)},\bar{x}^{(k)}]}
\|\nabla_xE_\theta(y,z)\|_2
\,
\|\bar{x}^{(k)}-x^{(k)}\|_2\\
&\quad\leq
G\|h_k\|_2.
\end{align*}
Therefore,
\[
|b_k|
\leq
\frac{G}{D_v}\|h_k\|_2,
\qquad
|q_k|
=
|\gamma_v-1|\,|C(\tau_k)|.
\]
Averaging \eqref{eq:per-iterate-residual-bound} over
$k=0,\ldots,K-1$ gives
\[
A_K-B_K
\leq
U_{\Delta E}
\leq
A_K+B_K.
\]
Since $U_{\Delta E}\geq0$, the lower bound can be written as
$[A_K-B_K]_+$.

For the second result, Cauchy--Schwarz implies
\[
|e_v(s)^\top v|
\leq
\|e_v(s)\|_2\|v\|_2.
\]
Consequently,
\begin{align*}
A_K
&\leq
\frac{\|v\|_2}{KD_v}
\sum_{k=0}^{K-1}
\int_{\tau_k}^1
\|e_v(s)\|_2\,ds\\
&=
\frac{\|v\|_2}{KD_v}
\int_0^1
\left(
\sum_{k=0}^{K-1}
\mathbf{1}\{\tau_k\leq s\}
\right)
\|e_v(s)\|_2\,ds\\
&=
\frac{\|v\|_2}{KD_v}
\int_0^1
N_K(s)\|e_v(s)\|_2\,ds.
\end{align*}
\end{proof}

Proposition~\ref{prop:energy-consistency-bounds} shows that the observed score $U_{\Delta E}$ is controlled by the integrated chord-wise field error, up to explicit corrections for off-chord motion and normalization.

\begin{corollary}[Exact on-chord characterization]
\label{cor:exact-on-chord-characterization}
Suppose that the normalization floor is inactive,
\begin{equation}
D_v=\|v\|_2^2,
\label{eq:inactive-normalization-floor}
\end{equation}
and every refinement iterate lies on the endpoint chord:
\begin{equation}
x^{(k)}=\bar{x}^{(k)},
\qquad
k=0,\ldots,K-1.
\label{eq:on-chord-condition}
\end{equation}
Then
\begin{equation}
U_{\Delta E}
=
\frac{1}{K\|v\|_2^2}
\sum_{k=0}^{K-1}
\left|
\int_{\tau_k}^1e_v(s)^\top v\,ds
\right|.
\label{eq:energy-consistency-on-chord-exact}
\end{equation}
Furthermore,
\begin{equation}
\frac{1}{K\|v\|_2^2}
\left|
\int_0^1
N_K(s)e_v(s)^\top v\,ds
\right|
\leq
U_{\Delta E}
\leq
\frac{1}{K\|v\|_2}
\int_0^1
N_K(s)\|e_v(s)\|_2\,ds.
\label{eq:energy-consistency-on-chord-bounds}
\end{equation}
\end{corollary}

\begin{proof}
If $D_v=\|v\|_2^2$, then
\[
\gamma_v=1.
\]
If $x^{(k)}=\bar{x}^{(k)}$, then
\[
E_\theta(\bar{x}^{(k)},z)
-
E_\theta(x^{(k)},z)
=
0.
\]
The decomposition in
\eqref{eq:energy-consistency-residual-decomposition} therefore reduces to
\[
r_k^{\Delta E}
=
\frac{1}{\|v\|_2^2}
\int_{\tau_k}^1e_v(s)^\top v\,ds.
\]
Taking absolute values and averaging proves
\eqref{eq:energy-consistency-on-chord-exact}.

For the lower bound, observe that
\begin{align*}
\int_0^1N_K(s)e_v(s)^\top v\,ds
&=
\int_0^1
\sum_{k=0}^{K-1}
\mathbf{1}\{\tau_k\leq s\}
e_v(s)^\top v\,ds\\
&=
\sum_{k=0}^{K-1}
\int_{\tau_k}^1
e_v(s)^\top v\,ds.
\end{align*}
Hence, by the triangle inequality,
\[
\left|
\int_0^1N_K(s)e_v(s)^\top v\,ds
\right|
\leq
\sum_{k=0}^{K-1}
\left|
\int_{\tau_k}^1e_v(s)^\top v\,ds
\right|.
\]
Dividing by $K\|v\|_2^2$ proves the lower bound.

For the upper bound,
\begin{align*}
U_{\Delta E}
&=
\frac{1}{K\|v\|_2^2}
\sum_{k=0}^{K-1}
\left|
\int_{\tau_k}^1e_v(s)^\top v\,ds
\right|\\
&\leq
\frac{1}{K\|v\|_2}
\sum_{k=0}^{K-1}
\int_{\tau_k}^1\|e_v(s)\|_2\,ds\\
&=
\frac{1}{K\|v\|_2}
\int_0^1N_K(s)\|e_v(s)\|_2\,ds.
\end{align*}
\end{proof}

Corollary~\ref{cor:exact-on-chord-characterization} provides the exact interpretation of the EqC score in the ideal geometric regime: $U_{\Delta E}$ is the average normalized magnitude of the remaining integrated longitudinal field error along the refinement path.

\subsection{Scope of the auxiliary analysis}\label{app:scope}

\begin{itemize}
\item The empirical findings of \citet{pan2026much} motivate the conditional-concentration regime; they do not prove
that \(\sigma_{\mathrm{act}}\) or \(\sigma_{\mathrm{scale}}\) is small for EqC. These quantities require direct
diagnostics on the selected datasets and trained policies.
\item Proposition~\ref{prop:stationary-energy-approximation} is an existence upper bound based on a single-basin witness. Low conditional dispersion is
sufficient, not necessary: a nonlinear energy may represent several action basins. Its weighted moment conditions
also cover schedules for which \(c(T)/(1-T)\) grows near the endpoint.
\item Theorem~\ref{thm:monotone-ascent} concerns optimization of the learned scalar only. Neither stationarity nor increasing energy is a
correctness guarantee, and a generic stationary field can have a separately established Lyapunov function.
\item Theorem~\ref{thm:energy-consistency-decomposition} decomposes the mean absolute self-consistency residual. Absolute averaging prevents cancellation
across refinement iterates, but opposite field errors can still cancel within an individual remaining-work integral.
The normalization floor matters for short chords, and a wrong but internally coherent trajectory can have a small
score.
\end{itemize}


\section{Real-Robot Experimental Setup and Evaluation Protocol}\label{app:real-robot-details}

\subsection{Platform and Hardware Architecture}\label{app:real-robot-hardware}
We conduct all physical experiments on a bimanual manipulation workstation comprising two lightweight, force-controlled, 6-DOF \textbf{Galaxea A1Z} robotic arms. The policies are fine-tuned from the \textbf{Galaxea G0.5} VLA model. Each A1Z manipulator weighs $5\,\text{kg}$ and has an operating reach radius of $626.5\,\text{mm}$, an operating height of $722\,\text{mm}$, and a rated continuous payload of $3\,\text{kg}$ (with a peak payload of $5\,\text{kg}$ at full extension).

Both manipulators are equipped with parallel-jaw grippers, yielding 14 active control degrees of freedom (twelve arm joints and two gripper channels). Visual inputs are captured simultaneously by three synchronized RGB cameras: an overhead global camera and two wrist-mounted cameras (left and right). Policy inputs also incorporate 14-dimensional bimanual proprioception (joint angles and gripper states). The physical control frequency is set to $15\,\text{Hz}$, and the receding-horizon action chunk size is set to $T_a = 32$.

\subsection{Task Specifications and Success Criteria}\label{app:real-robot-tasks}

\paragraph{Benchmark 1: Flower Insertion.}
The \textit{Flower Insertion} task (Section~\ref{real-world-experiment}) evaluates bimanual coordination, 3D spatial alignment, and contact-rich insertion. To successfully complete the task, the dual-arm system must:
\begin{enumerate}
    \item Approach and grasp the flexible stem from the workspace surface;
    \item Lift and reorient the stem in 3D space to align its longitudinal axis with the vertical entry axis of the vase opening;
    \item Smoothly feed the stem through the aperture until fully seated, dynamically correcting minor contact misalignments at the rim without tipping the vase over.
\end{enumerate}
A trial is scored as a success if the stem is fully inserted and the gripper safely releases without disturbing the upright posture of the vase.

\paragraph{Benchmark 2: Towel Folding.}
The \textit{Towel Folding} task evaluates bimanual spatial coordination, deformable object manipulation, and sequential multi-stage execution. The robot must perform a two-stage quarter-fold (folding both laterally and longitudinally):
\begin{enumerate}
    \item \textbf{First Fold (Horizontal):} Simultaneously approach and pinch two adjacent corners/edges of the unfolded towel using both grippers, lift and fold the towel in half laterally, and lay it down smoothly;
    \item \textbf{Reorientation:} Rotate the laterally folded towel using both arms to reorient it for the longitudinal fold;
    \item \textbf{Second Fold (Vertical) and Release:} Reposition both arms to regrasp the newly formed folded edges, lift the towel, and fold it in half longitudinally to form a compact four-layer square; then smooth the top layer down and cleanly release both grippers without disturbing the folded creases or dragging the fabric.
\end{enumerate}
A trial is scored as a success if and only if: (1) both the lateral and longitudinal folds are successfully executed, (2) the final folded fabric is reduced to approximately a quarter of its original footprint with aligned edges, and (3) the towel remains stably and neatly folded without unfolding or flipping after gripper release.

\paragraph{Benchmark 3: Debris Sweeping.}
The \textit{Debris Sweeping} task evaluates multi-stage bimanual tool coordination and granular object collection. The dual-arm robot must coordinate a hand brush (or mini broom) and a dustpan to gather scattered paper scraps from the tabletop and dump them into a designated waste bin:
\begin{enumerate}
    \item \textbf{Dual-Tool Grasping and Positioning:} The left manipulator picks up the brush from the holder tray, while the right manipulator grasps the dustpan and places its front edge flat against the table surface near the perimeter of the debris cluster;
    \item \textbf{Contact-Rich Sweeping:} The left arm coordinates with the right arm to sweep scattered paper scraps cleanly into the cavity of the dustpan using multiple sweeping strokes, maintaining proper surface contact without crushing or scattering the debris beyond the reachable workspace;
    \item \textbf{Transport and Dumping:} The right arm lifts the loaded dustpan, transfers it over the rim of the desktop waste bin, tilts it to dump the collected scraps into the container, and returns to a neutral position.
\end{enumerate}
A trial is scored as a success if and only if: (1) all of the scattered paper scraps are cleanly collected and transferred into the waste bin, (2) neither the brush nor the dustpan is dropped or detached from the grippers, and (3) the waste bin remains standing upright without being tipped over during the dumping motion.

\paragraph{Benchmark 4: Hierarchical Block Stacking.}
The \textit{Hierarchical Block Stacking} task serves as the testbed for both online competence awareness (Section~\ref{real-world-rollout-analysis}) and closed-loop lifelong learning (Section~\ref{lifelong-learning}), evaluating multi-stage sequential execution, precise contact and balance, and tolerance to compounding errors. The task involves three color-coded blocks (green, blue, and red). In the nominal collinear configuration, the three blocks are placed collinearly on the tabletop in each of the $3! = 6$ possible color permutations (e.g., green-blue-red, blue-green-red, and red-blue-green). The green block serves as a stationary base, remaining at its initial position while the dual-arm system identifies the scene layout and manipulates the other blocks. In the disordered configuration, the blocks are initialized at arbitrary, non-collinear 2D positions with unaligned orientations. The robot must perform the following stages:
\begin{enumerate}
    \item \textbf{Base Identification and Blue Block Placement:} The dual-arm system identifies the scene layout, grasps the blue block, and balances it on the green base;
    \item \textbf{Red Block Placement:} The robot then grasps the red block and balances it on the blue block to construct a freestanding three-tier tower;
    \item \textbf{Release and Stabilization:} Both grippers release the assembly and clear the scene, leaving the tower upright.
\end{enumerate}
A trial is scored as a success if and only if: (1) the blue block is successfully balanced on the green base, (2) the red block is successfully balanced on the blue block, and (3) the final three-tier assembly (green $\to$ blue $\to$ red) remains upright and stable for at least 3 seconds after both grippers release and clear the scene.

\paragraph{Benchmark 5: Garment Folding.}
The \textit{Garment Folding} task evaluates long-horizon bimanual deformable-object manipulation. Starting from an unfolded garment laid flat on the table, the dual-arm system must:
\begin{enumerate}
    \item \textbf{Sleeve Folding:} Sequentially grasp, lift, and fold the left and right sleeves inward across the torso;
    \item \textbf{Longitudinal Fold:} Reposition both grippers to pinch the bottom hem and fold the lower half of the garment upward along the vertical centerline;
    \item \textbf{Lateral Fold:} Grasp the outer folded edge and fold the garment horizontally to form a compact rectangular profile;
    \item \textbf{Smoothing and Alignment:} Dynamically smooth out wrinkles and creases across the top layer with downward contact wiping.
\end{enumerate}
A trial is scored as a success if and only if: (1) all four folding and smoothing stages are sequentially completed, (2) the final footprint is reduced cleanly without unfolding or severe edge misalignments, and (3) both grippers safely disengage without dragging or wrinkling the finished garment.

\paragraph{Benchmark 6: Hierarchical Cup Stacking (Cross-Task Shift).}
To evaluate the policy's lifelong adaptation and catastrophic forgetting mitigation under explicit cross-task shift, we introduce the \textit{Hierarchical Cup Stacking} task. The setup employs three color-coded tapered cups (green, blue, and orange). Identical to the collinear setup in \textit{Hierarchical Block Stacking}, the three cups are initialized collinearly in front of the robot across all $3! = 6$ spatial permutations. The green cup serves as the fixed base. The dual-arm system must:
\begin{enumerate}
    \item Identify the target cup layout, approach, and securely grasp the blue cup, carefully placing it upside down on top of the green base cup with its rim facing downward, without knocking the base cup over;
    \item Reposition the manipulators to grasp the orange cup, align its orientation, and stack it neatly atop the blue cup to form a compact three-tier nested cup assembly.
\end{enumerate}
The success criterion strictly mirrors block stacking: a trial is deemed successful if and only if both stacking stages are completed without tipping, and the resulting three-cup tower remains stable and upright for at least 3 seconds after the grippers fully release and clear the scene.

\subsection{Training Configurations and Computational Overhead}\label{app:real-robot-training}

All models are fine-tuned from base Galaxea G0.5 checkpoints on a GPU cluster with 768GB of total GPU memory and approximately 1.2 PFLOPS of BF16 compute, using the AdamW optimizer with a cosine learning-rate decay schedule. Table~\ref{tab:real-robot-params} lists the hyperparameters shared across the real-world experiments.

During physical deployment, EqC performs $K = 3$ energy gradient ascent steps starting from Gaussian noise $\mathcal{N}(0, I)$ to iteratively refine the action chunk, employing a decaying step-size schedule $[\eta_0, \eta_1, \eta_2] = [0.12, 0.10, 0.05]$. This decaying schedule facilitates rapid initial progress toward the high-compatibility energy basin followed by fine-grained local stabilization, while keeping the overall inference latency within practical limits.

\begin{table}[ht]
    
  \caption{Unified training, inference, and intervention
  hyperparameters for real-robot experiments.}
  \label{tab:real-robot-params}
  \centering
  \tabstyle
  \setlength{\tabcolsep}{14pt}
  \begin{tabular}{@{}ll@{}}
    \toprule
    \headrow Hyperparameter & Value \\
    \midrule
    Base VLA Backbone & Galaxea G0.5 \\
    Optimizer & AdamW \\
    Batch Size & 16 \\
    Peak Learning Rate & $8 \times 10^{-5}$ \\
    Warmup Schedule & Linear warmup (500 steps) \\
    Decay Schedule & Cosine decay (\texttt{warmup\_constant\_cosine}) \\
    Action Chunk Horizon ($T_a$) & 32 \\
    Executed Action-Prefix Length & 16 steps \\
    Control Frequency & 15\,Hz \\
    \midrule
    EqC Refinement Steps ($K$) & 3 \\
    EqC Step Sizes ($\eta_k$) & $[0.12, 0.10, 0.05]$ \\
    \midrule
    Intervention Threshold ($\gamma$) & 0.2 \\
    Persistence Window ($m$) & 3 consecutive replans \\
    \midrule
    Training Epochs (Flower Insertion) & 6 epochs (83 demonstrations) \\
    Training Epochs (Towel Folding) & 4 epochs (36 demonstrations) \\
    Training Epochs (Debris Sweeping) & 4 epochs (142 demonstrations) \\
    Training Epochs (Lifelong Learning) & 2 epochs per update session \\
    \bottomrule
  \end{tabular}
\end{table}

\paragraph{Wall-Clock Fine-Tuning Overhead.}
To assess the computational cost of learning an explicit energy function, Table~\ref{tab:training-time} reports wall-clock fine-tuning times for \textit{Flower Insertion} (6 epochs, 83 demonstrations), \textit{Towel Folding} (4 epochs, 36 demonstrations), and \textit{Debris Sweeping} (4 epochs, 142 demonstrations). Computing explicit energy gradients during backpropagation introduces only a modest overhead, increasing training time by 22.7\% on Flower Insertion, 22.1\% on Towel Folding, and 17.5\% on Debris Sweeping.

\begin{table}[ht]
  \caption{Wall-clock fine-tuning time across real-robot tasks.}
  \label{tab:training-time}
  \centering
  \tabstyle
  \setlength{\tabcolsep}{12pt}
  \begin{tabular}{@{}lcc@{}}
    \toprule
    \headrow Method & Time (min) & Overhead \\
    \midrule
    \multicolumn{3}{l}{\paneltitle{Flower Insertion}} \\
    Galaxea G0.5 & \best{111.12} & --- \\
    \addlinespace[1pt]
    \eecrow Galaxea G0.5 + \eqcmethod & 136.30 & +22.7\% \\
    \midrule
    \multicolumn{3}{l}{\paneltitle{Towel Folding}} \\
    Galaxea G0.5 & \best{40.48} & --- \\
    \addlinespace[1pt]
    \eecrow Galaxea G0.5 + \eqcmethod & 49.42 & +22.1\% \\
    \midrule
    \multicolumn{3}{l}{\paneltitle{Debris Sweeping}} \\
    Galaxea G0.5 & \best{115.30} & --- \\
    \addlinespace[1pt]
    \eecrow Galaxea G0.5 + \eqcmethod & 135.47 & +17.5\% \\
    \bottomrule
  \end{tabular}
\end{table}

\subsection{Energy-Gated DAgger Protocol}\label{app:real-robot-lifelong-protocol}

\paragraph{Initial Demonstration Dataset ($\mathcal{D}_0$).}
For \textit{Hierarchical Block Stacking}, the base policy $E_{\theta_0}$ is initialized by fine-tuning on a small seed demonstration set $\mathcal{D}_0$ comprising $N_0 = 18$ teleoperated episodes collected exclusively under nominal, collinear block initializations (spanning all $3! = 6$ color permutations with 3 demonstrations per configuration), yielding a $0\%$ initial success rate on perturbed setups. For \textit{Garment Folding}, the base policy is initialized on $N_0 = 18$ demonstrations on standard unfolded garments, yielding an initial autonomous success rate of $30\%$.

\paragraph{Human Intervention Baseline.}
To isolate the benefit of automated competence-aware intervention, we additionally evaluate a Human Intervention DAgger baseline under the same data and training budget as Energy-Gated DAgger. The human baseline is initialized from exactly the same seed demonstration dataset $\mathcal{D}_0$ containing $N_0=18$ teleoperated episodes. Intervention decisions are made by skilled teleoperators with more than three months of experience collecting VLA training data. At each round, the current checkpoint is deployed on the same nominal collinear configurations used for autonomous evaluation, spanning all $3! = 6$ possible color permutations for Block Stacking, while the operator continuously monitors whether autonomous execution is likely to fail or has moved beyond the policy's apparent competence. When the operator identifies such a failure-prone state, autonomous execution is suspended and the operator intervenes to provide a recovery trajectory. For Block Stacking, we collect exactly $6$ targeted recovery episodes per round over $R=5$ rounds, yielding a total of $30$ additional recovery demonstrations and expanding the dataset from $18$ to $48$ episodes. For Garment Folding, we follow the same procedure over $R=2$ rounds, collecting exactly $6$ targeted recovery episodes per round, yielding $12$ additional recovery demonstrations and expanding the dataset from $18$ to $30$ episodes. The human-in-the-loop baseline therefore follows the same task-specific schedule and collection budget as Energy-Gated DAgger: $R=5$ rounds and $30$ recovery demonstrations for Block Stacking, and $R=2$ rounds and $12$ recovery demonstrations for Garment Folding. The two methods share the same initial dataset, collection budget, deployment schedule, and policy-update procedure, and differ in how intervention decisions are triggered: external human judgment for the baseline and the policy-internal energy-consistency signal $U_{\Delta E}$ for our method.

\paragraph{Galaxea G0.5 + Diff-DAgger Baseline.}
To provide a rigorous automated baseline for lifelong learning, we adapt Diff-DAgger~\citep{lee2025diff} to the physical Galaxea G0.5 policy. Its online failure-detection scoring is identical to the flow-matching formulation detailed in Appendix~\ref{app:competence-awareness-diffdagger-baseline-method}, evaluating $M=10$ Monte Carlo samples per replan across active embodiment dimensions. Analogous to EqC, persistent threshold exceedances over $m$ consecutive replanning steps trigger expert recovery intervention. Crucially, G0.5 + Diff-DAgger operates under the exact same experimental conditions: it starts from the identical seed dataset, collects the exact same budget of targeted recovery demonstrations per round, and follows the identical iterative fine-tuning hyperparameters as Energy-Gated DAgger and the human baseline.

\paragraph{Active Deployment and Intervention Mechanism.}
During autonomous deployment in round $r$, EqC performs
$K=3$ gradient-ascent steps to generate a 32-step action
chunk at a control frequency of 15~Hz. The policy executes
a prefix of 16 action steps before replanning, while
monitoring the energy-consistency signal $U_{\Delta E}$.
If $U_{\Delta E}$ exceeds the fixed alarm threshold
$\gamma=0.2$ for $m=3$ consecutive replanning steps,
autonomous execution is suspended, and an expert
intervenes to provide a recovery trajectory $\xi_j^\star$
that guides the robot to successful task completion.
Each round collects $|\mathcal{Q}_r|=6$ targeted
recovery episodes.

\paragraph{Threshold Calibration.}
For physical deployment with $K=3$, we calibrate the energy-consistency threshold on 40 real-world rollouts from Flower Insertion and Towel Folding, comprising 10 successful and 10 failed rollouts per task. The selected threshold $\gamma=0.2$ provides useful separation between successful and failed executions and remains fixed across subsequent lifelong-learning experiments without task-specific retuning.

\paragraph{Iterative Fine-Tuning.}
Following Energy-Gated DAgger, each round aggregates newly collected recovery trajectories into the historical dataset ($\mathcal{D}_{r+1} \gets \mathcal{D}_r \cup \mathcal{Q}_r$). The policy is re-initialized from the same initial checkpoint and independently trained for 2 epochs on $\mathcal{D}_{r+1}$ using the hyperparameters in Table~\ref{tab:real-robot-params}. For Block Stacking, across $R = 5$ rounds, the robot incorporates 30 active recovery demonstrations, expanding the dataset from $|\mathcal{D}_0| = 18$ to $|\mathcal{D}_5| = 48$ episodes. For Garment Folding, the policy completes lifelong adaptation in $R = 2$ rounds, incorporating 12 targeted recoveries and expanding the dataset from $|\mathcal{D}_0| = 18$ to $|\mathcal{D}_2| = 30$ episodes. To ensure statistical consistency, all lifelong learning experiments across both tasks and all compared methods are repeated over 3 independent end-to-end runs. After each fine-tuning round, the resulting policy checkpoint is evaluated over 20 independent autonomous trials, and we report the mean success rate and standard deviation across the 3 runs in Table~\ref{tab:lifelong-detailed-stats} and Figure~\ref{fig:lifelong-curves}.

\begin{table}[t]
  \caption{Detailed progressive success rates (\%) across 3 independent lifelong learning runs on the physical Galaxea bimanual workstation (20 autonomous trials per checkpoint per run). Both individual run outcomes and summary statistics ($\text{Mean} \pm \text{Std}$) are reported across deployment rounds.}
  \label{tab:lifelong-detailed-stats}
  \centering
  \fontsize{7.5pt}{8.8pt}\selectfont
  \setlength{\tabcolsep}{3.5pt}
  \renewcommand{\arraystretch}{0.98}

  \paneltitle{(a) Hierarchical Block Stacking ($|\mathcal{D}_0|=18 \to |\mathcal{D}_5|=48$)}\par\vspace{3pt}
  \resizebox{\linewidth}{!}{%
  \begin{tabular}{@{} l l cccccc @{}}
    \toprule
    \headrow Method & Run & R0 (18) & R1 (24) & R2 (30) & R3 (36) & R4 (42) & R5 (48) \\
    \midrule
    \multirow{4}{*}{Human Intervention} 
      & Run 1 & 0.0 & 0.0 & 5.0  & 15.0 & 20.0 & 50.0 \\
      & Run 2 & 0.0 & 0.0 & 10.0 & 35.0 & 60.0 & 70.0 \\
      & Run 3 & 0.0 & 0.0 & 0.0  & 20.0 & 45.0 & 55.0 \\
      & \textit{Mean$\pm$Std} & $0.0 \pm 0.0$ & $0.0 \pm 0.0$ & $5.0 \pm 5.0$ & $23.3 \pm 10.4$ & $41.7 \pm 20.2$ & $58.3 \pm 10.4$ \\
    \midrule
    \multirow{4}{*}{Galaxea G0.5 + Diff-DAgger} 
      & Run 1 & 0.0 & 0.0 & 20.0 & 60.0 & 70.0 & 90.0 \\
      & Run 2 & 0.0 & 0.0 & 10.0 & 80.0 & 70.0 & 70.0 \\
      & Run 3 & 0.0 & 0.0 & 5.0  & 35.0 & 50.0 & 80.0 \\
      & \textit{Mean$\pm$Std} & $0.0 \pm 0.0$ & $0.0 \pm 0.0$ & $11.7 \pm 7.6$ & $58.3 \pm 22.5$ & $63.3 \pm 11.5$ & $80.0 \pm 10.0$ \\
    \midrule
    \multirow{4}{*}{\eqcmethod} 
      & Run 1 & 0.0 & 0.0 & 25.0 & 60.0 & 50.0 & 100.0 \\
      & Run 2 & 0.0 & 5.0 & 30.0 & 55.0 & 85.0 & 95.0 \\
      & Run 3 & 0.0 & 0.0 & 30.0 & 70.0 & 80.0 & 95.0 \\
      \eecrow
      & \textit{Mean$\pm$Std} 
      & $\mathbf{0.0 \pm 0.0}$ 
      & $\mathbf{1.7 \pm 2.9}$ 
      & $\mathbf{28.3 \pm 2.9}$ 
      & $\mathbf{61.7 \pm 7.6}$ 
      & $\mathbf{71.7 \pm 18.9}$ 
      & $\mathbf{96.7 \pm 2.9}$ \\
    \bottomrule
  \end{tabular}%
  }

  \vspace{8pt}

  \paneltitle{(b) Garment Folding ($|\mathcal{D}_0|=18 \to |\mathcal{D}_2|=30$)}\par\vspace{3pt}
  \begin{tabular}{@{} l l ccc @{}}
    \toprule
    \headrow Method & Run & R0 (18) & R1 (24) & R2 (30) \\
    \midrule
    \multirow{4}{*}{Human Intervention} 
      & Run 1 & 30.0 & 35.0 & 40.0 \\
      & Run 2 & 30.0 & 40.0 & 65.0 \\
      & Run 3 & 30.0 & 35.0 & 60.0 \\
      & \textit{Mean$\pm$Std} & $30.0 \pm 0.0$ & $36.7 \pm 2.9$ & $55.0 \pm 13.2$ \\
    \midrule
    \multirow{4}{*}{Galaxea G0.5 + Diff-DAgger} 
      & Run 1 & 25.0 & 35.0 & 90.0 \\
      & Run 2 & 25.0 & 50.0 & 80.0 \\
      & Run 3 & 25.0 & 45.0 & 75.0 \\
      & \textit{Mean$\pm$Std} & $25.0 \pm 0.0$ & $43.3 \pm 7.6$ & $81.7 \pm 7.6$ \\
    \midrule
    \multirow{4}{*}{\eqcmethod} 
      & Run 1 & 30.0 & 45.0 & 100.0 \\
      & Run 2 & 30.0 & 80.0 & 100.0 \\
      & Run 3 & 30.0 & 60.0 & 95.0 \\
      \eecrow
      & \textit{Mean$\pm$Std} 
      & $\mathbf{30.0 \pm 0.0}$ 
      & $\mathbf{61.7 \pm 17.6}$ 
      & $\mathbf{98.3 \pm 2.9}$ \\
    \bottomrule
  \end{tabular}
\end{table}

\paragraph{Expert-data costs.}
To further characterize the amount of expert supervision,
Table~\ref{tab:expert-data-costs} reports the total number
of expert-labeled frames collected during lifelong
learning, averaged across three independent runs.
All methods follow the same demonstration-count budget,
with 30 additional demonstrations over five rounds for Hierarchical Block Stacking and 12 additional demonstrations over
two rounds for Garment Folding.

\begin{table}[htbp]
    \caption{
    \textbf{Expert-data costs in real-world lifelong learning.}
    Total expert-labeled frames collected during lifelong
    learning, averaged across three independent runs.
    }
    \label{tab:expert-data-costs}
    \centering
    \tabstyle
    \setlength{\tabcolsep}{10pt}
    \begin{tabular}{@{}lcc@{}}
        \toprule
        \headrow Method
        & Hierarchical Block Stacking
        & Garment Folding \\
        \midrule
        Human Intervention
        & 13,141 & 7,384 \\
        G0.5 + Diff-DAgger
        & 12,128 & 6,946 \\
        \eecrow EqC (ours)
        & 13,883 & 7,848 \\
        \bottomrule
    \end{tabular}
\end{table}

With comparable amounts of expert supervision,
EqC achieves substantially higher final success rates
on both tasks, highlighting the effectiveness of
competence-aware recovery-data acquisition.

\subsection{Real-Time Deployment with Asynchronous Action Chunking}
\label{app:rtc}

\paragraph{Inference latency and asynchronous control.}
EqC generates action chunks through iterative gradient-based
refinement on the learned conditional energy landscape. In real-world deployment, the standard flow-matching Galaxea G0.5 policy requires approximately 400~ms per replan, while EqC-adapted G0.5 requires approximately 500~ms.

To improve continuous execution efficiency, we integrate
EqC with Real-Time Chunking (RTC)~\citep{black2025training},
which enables asynchronous generation of the next action
chunk while the current chunk is being executed.
By overlapping model inference with robot execution,
RTC can effectively hide the additional inference
latency of EqC during continuous robot motion.

\paragraph{RTC integration.}
We adapt EqC to asynchronous action chunking through
prefix-conditioned training.
During training, we sample delays of 0--12 action steps
within each 32-step action chunk, condition EqC on a
clean action prefix, and apply the energy-gradient
matching objective to the remaining action suffix.

During inference, EqC asynchronously generates the next
action chunk while the robot executes the current one.
The action prefix corresponding to the anticipated
inference delay provides conditioning for the refinement
of subsequent actions.
This allows the robot to continue executing available
actions while the policy generates the next chunk,
reducing interruptions caused by synchronous replanning.

\paragraph{Real-world evaluation.}
We integrate RTC into both the standard flow-matching
Galaxea G0.5 policy and EqC-adapted G0.5, and evaluate
them on the \textit{Towel Folding} task using the
Galaxea A1Z dual-arm system operating at 15~Hz.

Table~\ref{tab:rtc-results} reports the average task
completion time with and without RTC.
RTC substantially reduces task completion time for both
methods. In particular, EqC achieves an average
completion time of 37.28~s with RTC, compared with
60.34~s without RTC.
The EqC policy with RTC achieves a 95\% success rate
(19 out of 20 trials), with an average per-replan
inference latency of 525~ms.

\begin{table}[htbp]
    \caption{
    \textbf{Real-world task completion time with and
    without RTC.}
    Results are reported on the \textit{Towel Folding}
    task at a control frequency of 15~Hz.
    Lower completion time is better.
    }
    \label{tab:rtc-results}
    \centering
    \tabstyle
    \setlength{\tabcolsep}{12pt}
    \begin{tabular}{@{}lcc@{}}
        \toprule
        \headrow Method
        & Without RTC (s) $\downarrow$
        & With RTC (s) $\downarrow$ \\
        \midrule
        Galaxea G0.5
        & 57.88
        & 38.68 \\
        \eecrow Galaxea G0.5 + EqC (ours)
        & 60.34
        & \best{37.28} \\
        \bottomrule
    \end{tabular}
\end{table}

These results demonstrate that EqC is compatible with
efficient asynchronous real-time control.
By overlapping action generation with robot execution,
RTC enables EqC to achieve task execution efficiency
comparable to that of the standard flow-matching policy
while maintaining strong manipulation performance. The same asynchronous execution mechanism can support
higher control frequencies, provided that
the remaining action execution window covers the replanning
latency with sufficient timing margin.

\subsection{Same-Controller Comparison of Intervention Gates}
\label{app:same-controller-gates}

\paragraph{Experimental setup.}
To isolate the contribution of EqC's path-wise
energy-consistency signal to closed-loop lifelong
learning, we conduct additional real-world experiments
on Garment Folding using the same EqC controller
with different intervention gates.
All methods start from the same initial EqC checkpoint
and follow the same lifelong-learning protocol and
demonstration-count budget, collecting six expert
recovery demonstrations per round over two rounds.
Each method is evaluated across three independent
runs, with 20 autonomous trials per checkpoint per run.

\paragraph{Intervention gates.}
We compare the following three automated
intervention mechanisms:

\begin{itemize}
    \item \textbf{EqC + Diff-DAgger.}
    We adapt Diff-DAgger's loss-based failure detection
    to EqC's stationary energy-gradient matching
    objective, while keeping EqC's action generation
    unchanged.

    \item \textbf{EqC + Final-Step Gradient Gate.}
    We use the energy-gradient norm at the final
    refined action candidate as the intervention
    signal, evaluating local stationarity at the
    endpoint of the refinement trajectory.

    \item \textbf{EqC + Energy-Gated DAgger (ours).}
    We use the proposed path-wise energy-consistency
    signal $U_{\Delta E}$ to trigger expert intervention,
    incorporating information from the entire
    refinement trajectory.
\end{itemize}

\begin{table}[htbp]
  \caption{
  \textbf{Same-controller comparison of intervention
  gates on real-world Garment Folding.}
  Autonomous success rates (\%) are reported as
  mean $\pm$ sample standard deviation across
  three independent runs, with 20 evaluation
  trials per checkpoint per run.
  All methods use the same EqC controller,
  initial checkpoint, and demonstration-count budget.
  }
  \label{tab:same-controller-gates}
  \centering
  \tabstyle
  \setlength{\tabcolsep}{6pt}
  \begin{tabular}{@{}lccc@{}}
    \toprule
    \headrow Method
    & R0 (18) & R1 (24) & R2 (30) \\
    \midrule
    EqC + Diff-DAgger
    & $30.0 \pm 0.0$
    & $48.3 \pm 11.5$
    & $81.7 \pm 5.8$ \\
    EqC + Final-Step Gradient Gate
    & $30.0 \pm 0.0$
    & $43.3 \pm 7.6$
    & $61.7 \pm 7.6$ \\
    \eecrow EqC + Energy-Gated DAgger (ours)
    & $30.0 \pm 0.0$
    & $\best{61.7} \pm 17.6$
    & $\best{98.3} \pm 2.9$ \\
    \bottomrule
  \end{tabular}
\end{table}

\paragraph{Results.}
As shown in Table~\ref{tab:same-controller-gates},
Energy-Gated DAgger achieves the highest mean
autonomous success rates after both learning rounds.
At the final round, our method achieves
$98.3 \pm 2.9\%$, compared with
$81.7 \pm 5.8\%$ for EqC + Diff-DAgger and
$61.7 \pm 7.6\%$ for EqC + Final-Step Gradient Gate.
The three independent runs of our method achieve
final success rates of 100\%, 100\%, and 95\%,
respectively.

These results demonstrate that path-wise
energy-consistency gating achieves the highest
final mean success rate and the lowest observed
final-round standard deviation among the three
intervention strategies under the matched
controller and demonstration budget.

\section{Robot Agent as Expert}
\label{app:robot-agent-as-expert}

Energy-Gated DAgger connects competence detection, expert recovery, and policy updates through a common intervention interface. Inspired by recent work on agentic robotics~\citep{huang2026scaling}, we extend this framework by using a robot agent powered by \texttt{gpt-6-astra} as the recovery expert on \textit{Garment Folding}. Once the expert agent is prepared, both the decision to request assistance and the subsequent recovery are automated: EqC detects when its execution becomes unreliable, and the agent takes over to complete the task without human intervention. Incorporating these recoveries into policy training closes an automated intervention-and-learning loop. This also provides a route to \emph{behavioral distillation}, through which a compact robot policy acquires capabilities from a larger model through the recovery demonstrations it generates.

\paragraph{Experimental setup.}
We use the bimanual platform and garment-folding task described in Appendix~\ref{app:real-robot-details}. We initialize from the seed-trained EqC checkpoint of Run~1 at Round~0 in Table~\ref{tab:lifelong-detailed-stats}(b), obtained using the original 18 garment-folding demonstrations. During deployment, EqC generates actions and monitors its energy-consistency signal $U_{\Delta E}$. The alarm threshold remains fixed at $\gamma=0.2$, identical to the preceding lifelong learning experiments, and we retain the same persistence criterion for triggering intervention. Upon an alarm, autonomous execution is suspended and the robot agent assumes the expert-recovery role in Algorithm~\ref{alg:eqc-lifelong}.

\paragraph{Expert agent preparation.}

We prepare the agent through full-task practice followed by
takeover-specific practice (Algorithm~\ref{alg:agent-preparation-and-takeover}). We first provide Astra with the original 18 teleoperation
demonstrations. The agent inspects these demonstrations offline to understand the folding procedure and writes a robot-control script that approximates the recorded manipulation trajectories. It then performs repeated complete garment-folding attempts on the physical robot. After each attempt, Astra reviews the execution outcome, revises the script, and consolidates its spatial reasoning and manipulation experience into reusable executable skills. This process continues until the agent achieves a success rate of approximately 100\% on complete garment-folding attempts.

Because EqC has already performed part of the task when an intervention is triggered, recovery requires adaptation to intermediate garment and robot configurations. We therefore conduct several rounds of takeover-specific practice. Following an EqC alarm, the agent assesses the current scene, estimates spatial relationships and distances, and adapts its control script to complete the task from the intervention state. After execution, it reviews the recovery, refines the reusable control code, and incorporates the new experience into its skill library. Once preparation is complete, the agent enters routine takeover: it uses the prepared code and accumulated skills to adapt an execution script to each intervention state, without further post-execution reflection or skill consolidation. The adapted script is executed directly on the robot, avoiding a separate language-model response for each low-level action.

\begin{algorithm}[t]
\caption{Expert Agent Preparation and Routine Takeover}
\label{alg:agent-preparation-and-takeover}
\footnotesize

\noindent
\textit{Notation.}
$P$: control script;
$\mathcal{S}$: skill library;
$z$: observed intervention state;
$\tau$: execution record, including the observed outcome.
\par\smallskip

\begin{algorithmic}[1]
\Require Initial demonstrations $\mathcal{D}_0$;
         full-task practice rounds $R_{\mathrm{full}}$;
         takeover practice rounds $R_{\mathrm{takeover}}$
\State $P \gets \Call{AstraFitScript}{\mathcal{D}_0}$
\State $\mathcal{S} \gets \varnothing$

\Statex
\Statex \textbf{Stage 1: Full-task practice}
\For{$r = 1,\ldots,R_{\mathrm{full}}$}
    \State $\tau \gets \Call{ExecuteFromTaskStart}{P}$
    \State $P \gets \Call{AstraRefineScript}{P,\mathcal{S},\tau}$
    \State $\mathcal{S} \gets
        \Call{AstraConsolidateSkills}{\mathcal{S},P,\tau}$
\EndFor

\Statex
\Statex \textbf{Stage 2: Takeover practice}
\For{$r = 1,\ldots,R_{\mathrm{takeover}}$}
    \State $z \gets \Call{WaitForEqCAlarm}{}$
        \Comment{EqC pauses at the intervention state}
    \State $\widetilde{P} \gets
        \Call{AstraAdaptScript}{P,\mathcal{S},z}$
    \State $\tau \gets \Call{ExecuteRecovery}{\widetilde{P},z}$
    \State $P \gets
        \Call{AstraRefineScript}{\widetilde{P},\mathcal{S},\tau}$
    \State $\mathcal{S} \gets
        \Call{AstraConsolidateSkills}{\mathcal{S},\widetilde{P},\tau}$
\EndFor

\Statex
\Statex \textbf{Stage 3: Routine takeover}
\State $(P^\star,\mathcal{S}^\star) \gets (P,\mathcal{S})$
\For{each subsequent EqC alarm}
    \State $z \gets \Call{ObserveInterventionState}{}$
    \State $\widetilde{P} \gets
        \Call{AstraAdaptScript}{P^\star,\mathcal{S}^\star,z}$
    \State $\tau \gets \Call{ExecuteRecovery}{\widetilde{P},z}$
    \State \Call{RecordRecoveryDemonstration}{$\tau$}
    \Comment{No post-execution reflection or skill updates}
\EndFor
\end{algorithmic}
\end{algorithm}

\paragraph{Takeover and policy learning.}
After expert preparation, we deploy the agent for routine takeover, following Stage~3 of Algorithm~\ref{alg:agent-preparation-and-takeover}.
For each EqC alarm, Astra uses the prepared code and accumulated skills to adapt a recovery script to the current intervention state. The adapted code is then executed to complete the remaining operations. 

We collect six agent-generated recovery demonstrations and combine them with the original demonstration set:
\begin{equation}
    |\mathcal{D}_0| = 18,
    \quad
    |\mathcal{Q}_{\mathrm{Astra}}| = 6,
    \quad
    \mathcal{D}_{\mathrm{1}}
    =
    \mathcal{D}_0 \cup \mathcal{Q}_{\mathrm{Astra}},
    \quad
    |\mathcal{D}_{\mathrm{1}}| = 24.
\end{equation}
After fine-tuning on this combined dataset, the updated EqC policy successfully completes all 20 autonomous garment-folding evaluation trials, achieving 100\% success without agent takeover.

\paragraph{Implications.}
This proof-of-concept experiment demonstrates how EqC can support a fully automated intervention-and-learning loop after initialization and expert preparation. Its introspective signal determines when assistance is needed, the robot agent supplies recovery behavior, and policy updates incorporate the resulting experience into subsequent autonomous execution. The system also supports behavioral distillation at two levels: the larger model's reasoning is consolidated into executable robot skills, and the recoveries produced by those skills become supervision for the compact EqC policy. By connecting autonomous competence detection, agent-directed recovery, and policy improvement, this framework provides a promising path toward lifelong robot learning: situations beyond the policy's current competence become opportunities for targeted learning, allowing the robot to progressively expand its autonomous capabilities through continued deployment.

\section{Cross-Task Continual Learning and Catastrophic Forgetting}
\label{app:task-shift-cup-stacking}

\paragraph{Protocol.}
Real-world lifelong deployment often requires agents to acquire
novel manipulation skills sequentially while retaining competence
on previously mastered behaviors. To evaluate this capability,
we conduct cross-task continual learning experiments across all
three independent EqC training runs. For each run, we initialize
from its Round~5 Block Stacking checkpoint in
Table~\ref{tab:lifelong-detailed-stats}(a), trained on
$|\mathcal{D}_5|=48$ demonstrations, and deploy it directly onto
the novel Hierarchical Cup Stacking benchmark without any prior
cup interactions.

\begin{itemize}
    \item \textbf{Round 0:}
    The three initial policies achieve an average success rate
    of 96.7\% on Block Stacking and 0.0\% on Cup Stacking.
    As visualized in Figure~\ref{fig:cup-stacking-rollouts}
    (top row), trajectory drift and grasp slippage during
    zero-shot cup manipulation trigger persistent
    $U_{\Delta E}$ violations, prompting expert assistance.

    \item \textbf{Round 1:}
    Using EqC's Energy-Gated DAgger with the frozen threshold
    $\gamma$, we collect six intervention-triggered Cup Stacking
    recovery demonstrations for each run
    ($|\mathcal{Q}_1^{\mathrm{cup}}|=6$).
    After fine-tuning on
    $\mathcal{D}_5 \cup \mathcal{Q}_1^{\mathrm{cup}}$,
    the average Cup Stacking success rate increases to 35.0\%,
    while Block Stacking performance decreases to 60.0\%,
    indicating forgetting of previously acquired skills.

    \item \textbf{Round 2:}
    Starting from the same Round~1 deployment policy and
    cumulative dataset within each run, we compare EqC-guided
    recovery with ordinary cumulative replay.
    Both approaches collect 12 additional expert demonstrations,
    including six for Cup Stacking and six for Block Stacking.
    Cumulative replay collects complete expert demonstrations
    from the beginning of each task, while EqC collects
    intervention-triggered recovery trajectories from
    competence-critical states. Both approaches therefore use
    a matched demonstration-count budget, resulting in a
    cumulative dataset of 66 demonstrations per run.
\end{itemize}

\begin{figure}[htbp]
  \centering
  \includegraphics[width=\linewidth]{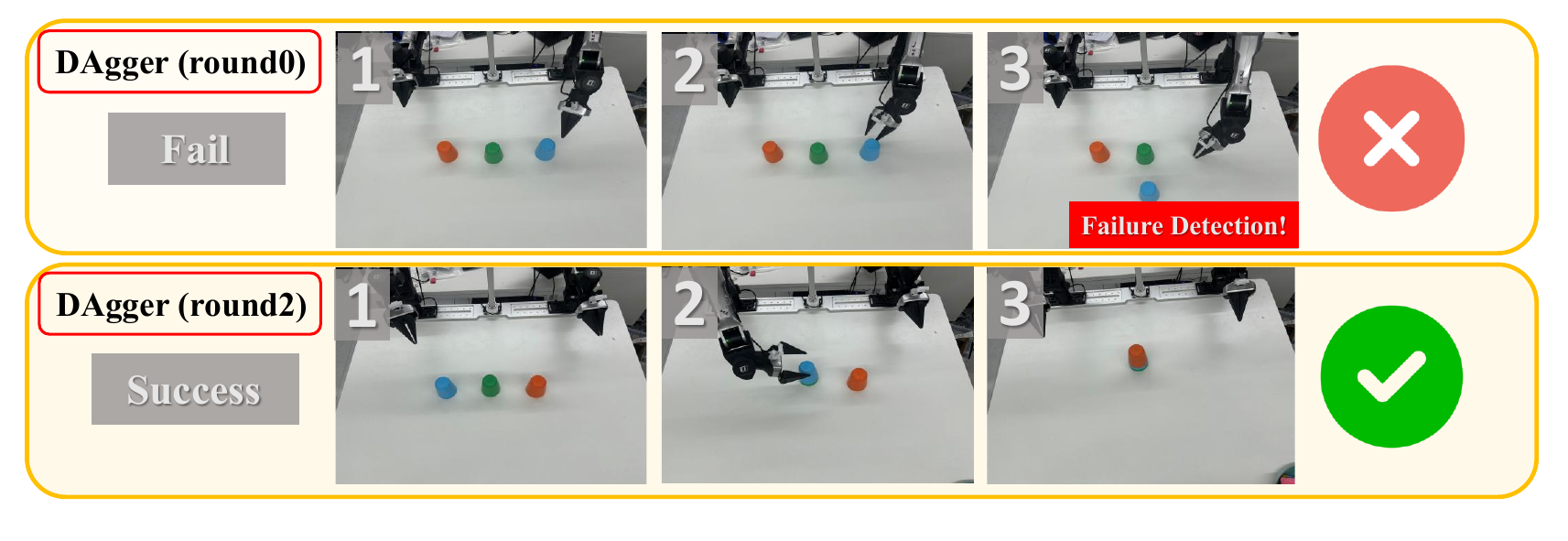}
  \vspace{-0.6em}
  \caption{
  \textbf{Cross-task adaptation on Hierarchical Cup Stacking. }
  \textbf{Top (Round 0):} Zero-shot deployment causes grasp
  slippage and drift, promptly caught by a persistent
  $U_{\Delta E}$ failure alarm.
  \textbf{Bottom (Round 2):} After two rounds of Energy-Gated
  DAgger, the policy achieves fully autonomous execution,
  cleanly assembling the three-tier cup tower.
  }
  \label{fig:cup-stacking-rollouts}
  \vspace{-0.5em}
\end{figure}

\begin{table}[htbp]
  \caption{
  \textbf{Cross-task adaptation and forgetting recovery across
  three independent runs.}
  Each entry reports Block Stacking / Cup Stacking success
  rates (\%).
  Round~0 uses 48 Block Stacking demonstrations.
  Round~1 adds six Cup Stacking recovery demonstrations.
  Both Round~2 methods collect 12 additional demonstrations
  (six per task), yielding a cumulative dataset of
  66 demonstrations.
  }
  \label{tab:task-shift-results}
  \centering
  \tabstyle
  \setlength{\tabcolsep}{6pt}
  \begin{tabular}{@{}lcccc@{}}
    \toprule
    & & & \multicolumn{2}{c}{Round 2 (+12 Demos)} \\
    \cmidrule(lr){4-5}
    \headrow Run
    & Round 0
    & Round 1 (+6 Cup)
    & Cumulative Replay
    & EqC (Ours) \\
    \midrule
    Run 1
    & 100 / 0
    & 60 / 50
    & 95 / 65
    & \best{100 / 100} \\
    Run 2
    & 95 / 0
    & 50 / 35
    & 70 / 65
    & \best{90 / 85} \\
    Run 3
    & 95 / 0
    & 70 / 20
    & 80 / 50
    & \best{95 / 70} \\
    \midrule
    \eecrow Mean
    & 96.7 / 0.0
    & 60.0 / 35.0
    & 81.7 / 60.0
    & \best{95.0 / 85.0} \\
    \bottomrule
  \end{tabular}
\end{table}

\paragraph{Results.}
As shown in Table~\ref{tab:task-shift-results}, EqC-guided recovery outperforms cumulative replay on both tasks across all three independent runs. After Round~2, EqC-guided recovery achieves average success rates of 95.0\% on Block Stacking and 85.0\% on Cup Stacking, compared with 81.7\% and 60.0\% for cumulative replay, corresponding to improvements of 13.3 and 25.0 percentage points, respectively.

These results demonstrate the effectiveness of competence-aware recovery-data acquisition for both new-task adaptation and recovery of previously acquired skills under matched demonstration-count budgets. By actively gating data collection for degraded and novel skills, EqC provides a practical mechanism for recovering from task-interference-induced forgetting during lifelong deployment. 

\section{Ablation Study} \label{app:ablations}

\subsection{Robustness to Optimization Step Schedules}
\label{subsec:opt-schedule-robustness}

To further evaluate the sensitivity of EqC to the underlying optimization step-size schedule, we conduct an ablation study across the LIBERO benchmark. Specifically, we compare our default Cosine decay schedule against uniform step schedules (with 15 and 20 optimization steps). Table~\ref{tab:opt-schedule} presents the performance breakdown across the four standard LIBERO task suites (Goal, Spatial, Object, and Long). The results demonstrate that EqC is highly robust to variations in the optimization scheduling strategy. While the Cosine decay schedule yields the best average performance (achieving $94.3$), the uniform schedules maintain competitive results, with the 20-step uniform schedule reaching $94.2$ and the 15-step schedule maintaining $93.1$. This consistent performance indicates that our energy-based failure detection mechanism relies on the fundamental energy landscape rather than highly tuned optimization step sizes, ensuring reliable and stable operation across different scheduling choices.

\begin{table}[htbp]
  \arrayrulecolor{black} 
  \caption{Success rate (\%) of EqC under different optimization step schedules across the LIBERO suites. The highest values in each column are highlighted in bold.}
  \label{tab:opt-schedule}
  \centering
  \small
  \setlength{\tabcolsep}{8pt}
  \begin{tabular}{@{} l c c c c c @{}}
    \toprule
    \headrow Methods & Goal & Spatial & Object & Long & Average \\
    \midrule
    EqC (uniform 15 steps) & \textbf{95.8} & 95.2 & 98.4 & 83.0 & 93.1 \\
    EqC (uniform 20 steps) & 95.4 & 96.2 & 98.0 & \textbf{87.2} & 94.2 \\
    \eecrow EqC (Cosine decay) & 95.6 & \textbf{97.2} & \textbf{98.6} & 85.4 & \textbf{94.3} \\
    \bottomrule
  \end{tabular}
\end{table}

\subsection{Adaptive Inference}\label{subsec:adaptive-inference}

To investigate whether our action optimization converges to an equilibrium during inference, we adapt EqC to DiT backbones and analyze the optimization dynamics across the LIBERO benchmark. During inference procedure, we observe that candidate actions rapidly reach a stationary regime, as evidenced by two main aspects in Table~\ref{tab:grad-energy-convergence-dit}:
\begin{itemize}
    \item \textbf{Rapid Gradient Decay}: the mean gradient norm $\|\nabla E\|$ decreases almost monotonically from an initial noise scale ($49.3$) at step $0$ to around $1.0$ by step $14$, ultimately flattening to a floor of $0.3$ at the final step.
    \item \textbf{Early Energy Saturation}: the normalized energy convergence ratio—defined as $(E_k - E_0)/(E_{\mathrm{final}} - E_0)$—saturates remarkably early, reaching $99.0\%$ by step $7$ and exceeding $99.9\%$ by step $9$.
\end{itemize}
These empirical observations indicate that our method rapidly reaches an equilibrium, as action candidate refinement enters a stable energy state within the first $10$--$15$ steps.

Motivated by this steady-state behavior, we introduce an adaptive early-stopping mechanism that terminates action chunk optimization once the gradient norm falls below a threshold $g_{\mathrm{min}} \in \{1.0, 3.0\}$. As summarized in Table~\ref{tab:early-stopping-ablation}, adaptive early stopping significantly enhances inference efficiency without incurring any notable performance degradation. Specifically, setting $g_{\mathrm{min}} = 1.0$ reduces the average optimization steps to $16.8$ (a $47.5\%$ reduction in computation compared to the full budget) while maintaining an equivalent average success rate ($93.7$ vs. $94.0$ baseline). Setting the threshold to $g_{\mathrm{min}} = 3.0$ further truncates the average optimization steps down to $11.1$ (saving $65.3\%$ of inference computation) while achieving a success rate of $94.2$. These results confirm that our proposed method---Equilibrium Energy Control---can effectively eliminate redundant computation while preserving policy execution quality through an adaptive gradient-based stopping strategy.

\begin{table}[htbp]
  \caption{Convergence profile of mean gradient norm $\|\nabla E\|$ and normalized energy saturation ratio across optimization steps $k$ on LIBERO.}
  \label{tab:grad-energy-convergence-dit}
  \centering
  \small
  \setlength{\tabcolsep}{5pt}
  \begin{tabular}{@{} l c c c c c c c c c c @{}}
    \toprule
    \headrow Step $k$ & $0$ & $5$ & $6$ & $7$ & $8$ & $9$ & $14$ & $19$ & $32$ \\
    \midrule
    Mean Gradient Norm $\|\nabla E\|$ & 49.3 & 44.4 & 29.9 & 12.9 & 6.0 & 3.6 & 1.2 & 0.7 & 0.3 \\
    $E$ Saturation Ratio (\%) & 0.00 & 82.51 & 95.18 & 99.22 & 99.79 & 99.91 & 99.99 & -- & -- \\
    \bottomrule
  \end{tabular}
\end{table}

\begin{table}[htbp]
  \caption{Performance and inference efficiency under adaptive early stopping across LIBERO task suites. Metrics are reported as Success Rate (\%) / Average Optimization Steps per action chunk.}
  \label{tab:early-stopping-ablation}
  \centering
  \small
  \setlength{\tabcolsep}{8pt}
  \begin{tabular}{@{} l c c c @{}}
    \toprule
    \headrow Task Suite & Baseline ($K=32$) & $g_{\mathrm{min}}=1.0$ & $g_{\mathrm{min}}=3.0$ \\
    \midrule
    libero\_goal & $95.6$ / $32.0$ & $97.0$ / $16.9$ & $97.0$ / $11.1$ \\
    libero\_spatial & $96.4$ / $32.0$ & $95.2$ / $17.3$ & $98.0$ / $11.3$ \\
    libero\_object & $98.1$ / $32.0$ & $98.0$ / $16.2$ & $97.6$ / $10.9$ \\
    libero\_long & $85.5$ / $32.0$ & $84.6$ / $16.8$ & $84.2$ / $10.9$ \\
    \midrule
    \eecrow \textbf{Average} & $\mathbf{94.0}$ / $\mathbf{32.0}$ & $\mathbf{93.7}$ / $\mathbf{16.8}$ & $\mathbf{94.2}$ / $\mathbf{11.1}$ \\
    \bottomrule
  \end{tabular}
\end{table}

\section{Case Study: Detection Timeliness and Failure Modes}\label{app:case-study}

A practical introspective robot system must trigger assistance \emph{at the competence boundary}: alarming neither prematurely on benign execution, nor belatedly after irreversible physical damage has occurred. In this section, we present real-world qualitative case studies on the physical Galaxea bimanual manipulation platform to evaluate detection timeliness across two tasks: Hierarchical Block Stacking and Garment Folding, which contain rigid-body multi-stage contact and deformable manipulation, respectively. We first highlight EqC's proactive detection capability, followed by an analysis of the baseline's failure modes and a discussion of EqC's own missed failures and limitations.

\subsection{Proactive Detection by EqC at the Competence Boundary}

As illustrated in Figure~\ref{fig:case-study-eqc-timing}, EqC's path-wise energy consistency residual $U_{\Delta E}$ demonstrates sharp sensitivity to the earliest emergence of off-manifold execution drift:

\begin{itemize}
    \item \textbf{Hierarchical Block Stacking (Top row):}
    When the block is accidentally pushed away (Top-left) or angular misalignment occurs during contact placement (Top-right), $U_{\Delta E}$ persistently exceeds the threshold $\gamma$, arresting autonomous motion before blocks are knocked over or the tower collapses.

    \item \textbf{Garment Folding (Bottom row):} EqC reliably flags failure to secure the right sleeve during grasping (Bottom-left) and failure to reach the target folding position (Bottom-right), requesting expert takeover while
    the fabric remains completely recoverable.
\end{itemize}



{\centering
\includegraphics[
    width=0.72\linewidth,
    height=0.24\textheight,
    keepaspectratio
]{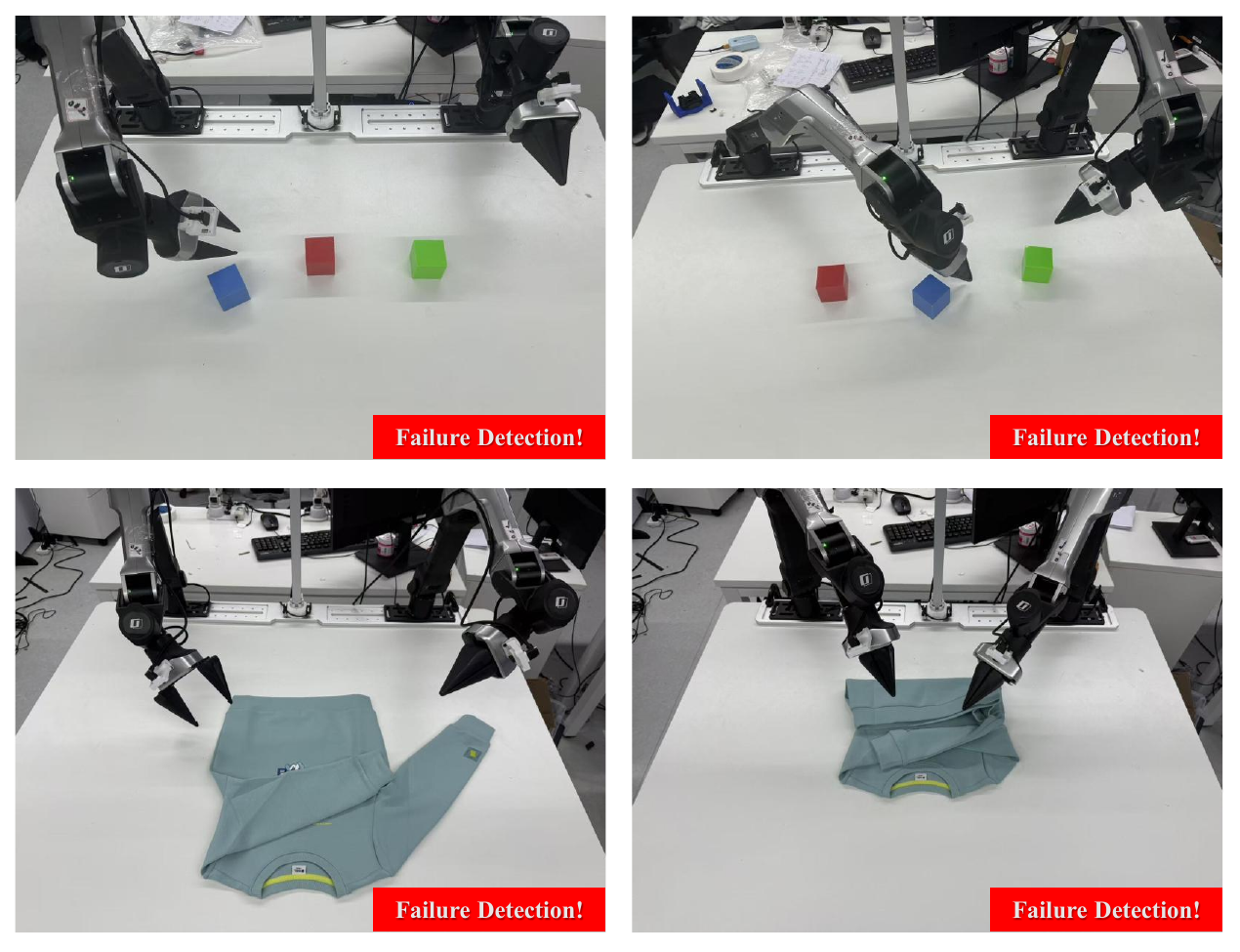}
\par}

\captionof{figure}{
    \textbf{Proactive competence-boundary detection of EqC
    across rigid and deformable manipulation.}
    Persistent $U_{\Delta E}$ violations trigger timely
    expert recovery during execution drift.
    \textbf{Top:} Block Stacking, including displaced blocks
    and angular misalignment.
    \textbf{Bottom:} Garment Folding, including failed
    sleeve grasping and incorrect folding positions.
}
\label{fig:case-study-eqc-timing}

\vspace{0.8em}

\noindent\textbf{Human versus EqC Intervention Timing.}
We provided experienced data collectors with videos of both successful and failed rollouts, while providing no other information (such as the EqC signal curves). Based solely on the video observations, they were asked to annotate the moments at which they would choose to intervene. Meanwhile, we determined EqC's intervention timing from its signal curves, enabling a direct comparison of intervention timing between EqC and human operators on the exact same rollouts. Figure~\ref{fig:human_eqc_comparison} illustrates four cases.
\textbf{Premature human intervention (Rows 1 and 3):}
During imperfect sleeve folding or after a block is knocked
over, humans prefer intervention, whereas EqC raises no
alarm and the robot successfully recovers.
\textbf{Delayed human intervention (Rows 2 and 4):}
When a gripper pins the sleeve or fails to release a block,
humans tend to wait, whereas EqC triggers an alarm before
the eventual failure. These cases highlight the distinction between external intervention judgments and policy-internal competence
assessment. Human decisions rely on observable behavior
and prior experience, which may not align with the
executing policy's actual recovery capabilities.
By monitoring its own action-generation consistency,
EqC can request assistance based on its estimated
competence boundary, enabling more targeted recovery
data collection.

\vspace{0.6em}


{
\centering
\includegraphics[
    width=0.68\linewidth,
    height=0.46\textheight,
    keepaspectratio
]{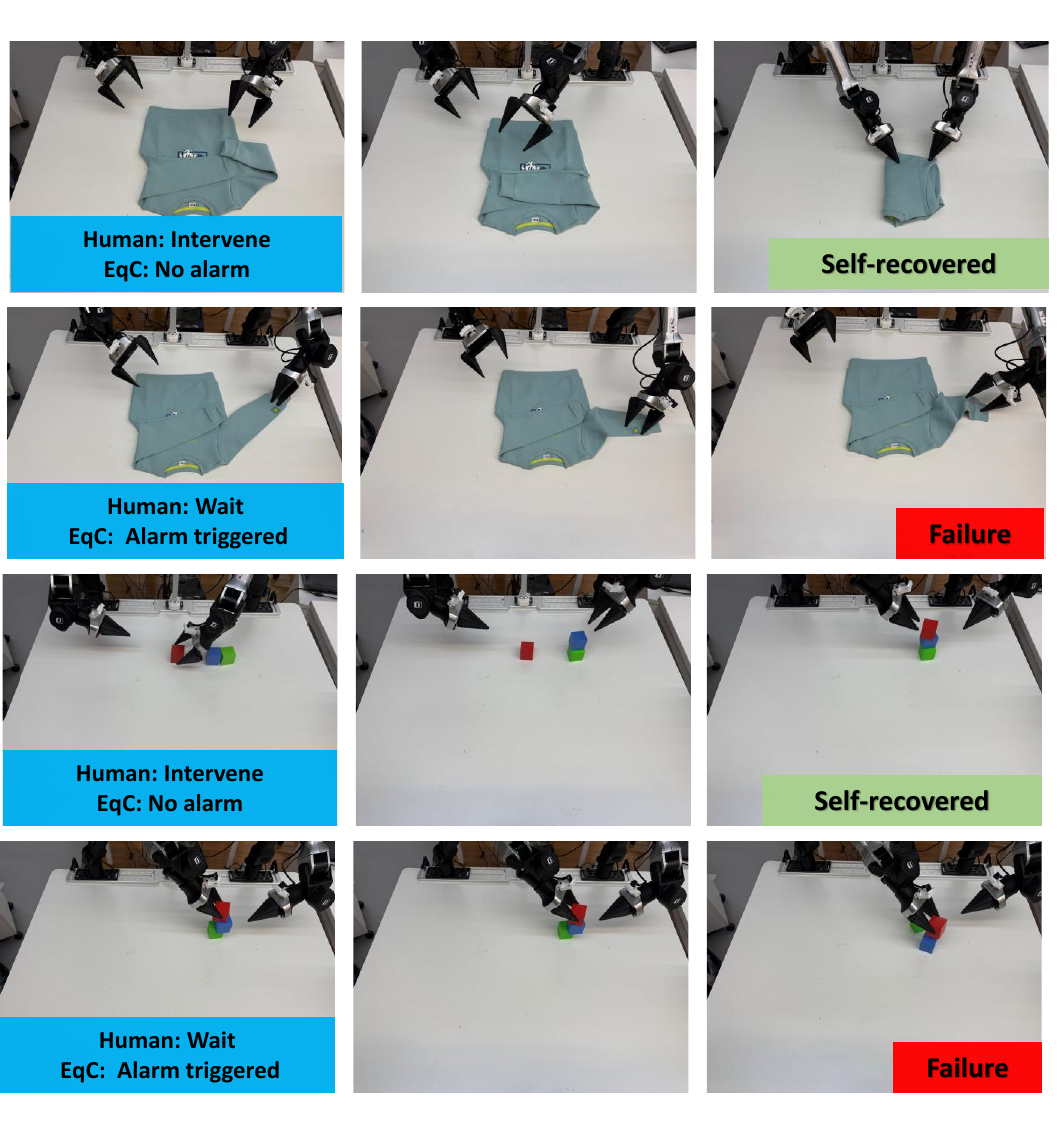}
\par

\captionof{figure}{
    \textbf{Comparison of human-annotated intervention
    decisions and EqC alarms on identical real-world
    trajectories.}
    Experienced operators annotate decisions using
    rollout videos only.
    In rows 1 and 3, humans prefer intervention,
    whereas EqC raises no alarm and the policy
    recovers autonomously.
    In rows 2 and 4, EqC triggers an alarm at the
    first displayed frame while humans prefer to wait.
    Without intervention, these trajectories
    eventually fail.
}
\label{fig:human_eqc_comparison}
\par}



\subsection{Pathology Diagnostics of Flow-Loss Detection (Diff-DAgger)}
In contrast to EqC's proactive warnings, point-wise velocity error in flow-matching policies displays significant variance and poor temporal sensitivity to compounding geometric error. As illustrated in Figure~\ref{fig:case-study-snapshots}, evaluating \textbf{Galaxea G0.5 + Diff-DAgger} across both rigid stacking (Rows 1--2) and deformable garment folding (Row 3) reveals three distinct, pervasive pathology modes across multiple independent trials:

\begin{figure}[!htbp]
    \centering
    \includegraphics[
        width=0.92\linewidth,
        keepaspectratio
    ]{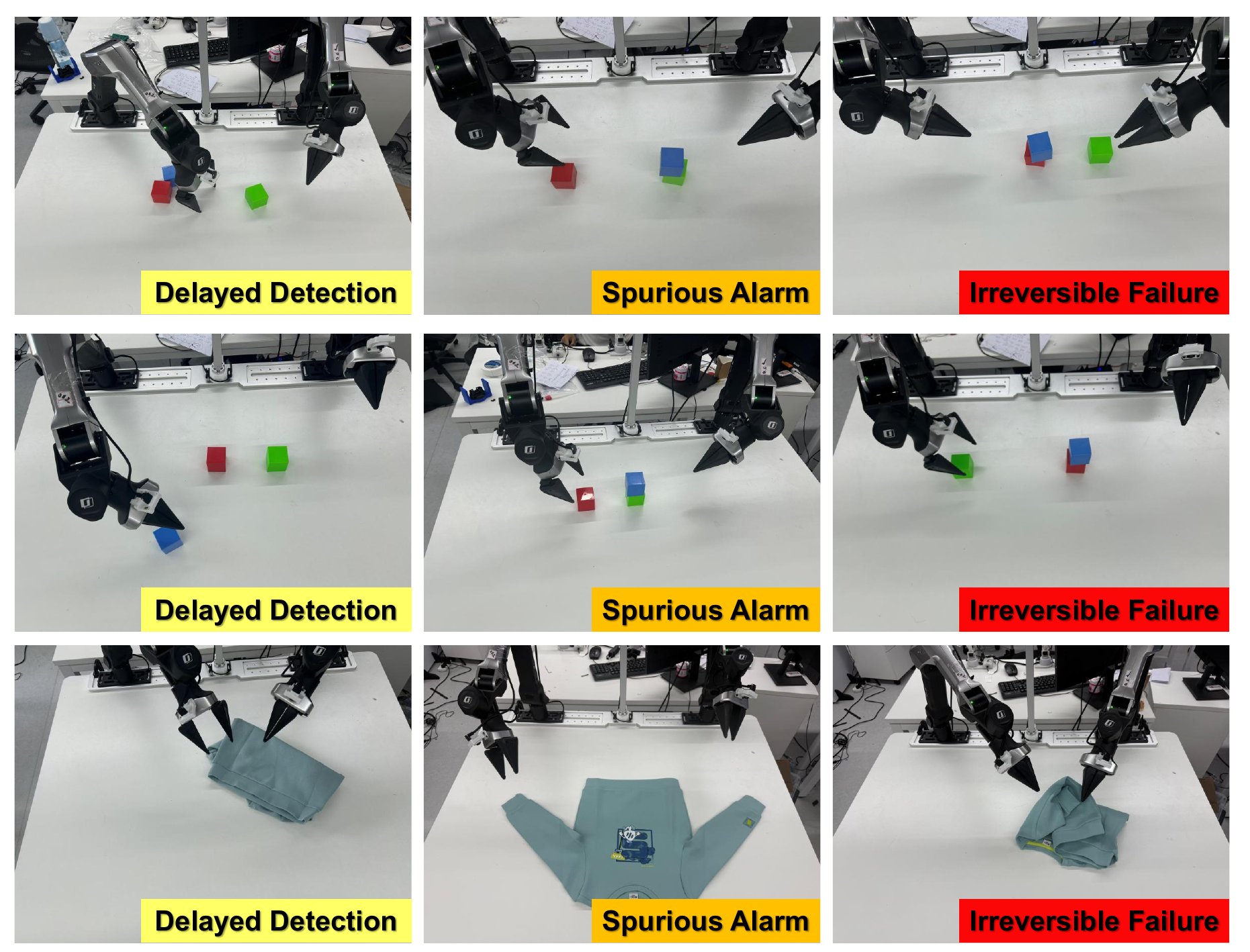}
    \vspace{-0.4em}
    \caption{
        \textbf{Qualitative failure modes of Galaxea G0.5
        + Diff-DAgger across rigid and deformable manipulation.}
        Rows 1--2 evaluate \textit{Hierarchical Block Stacking}
        under nominal setups, while Row 3 evaluates
        \textit{Garment Folding}.
        \textbf{Left column (Delayed Detection):}
        The detector alarms only after significant
        execution drift has accumulated.
        \textbf{Middle column (Spurious Alarm):}
        Noisy flow-loss estimates trigger false-positive
        alarms during otherwise valid execution.
        \textbf{Right column (Irreversible Failure):}
        Alarms are raised only after severe manipulation
        failures, leaving no viable in-place recovery path.
    }
    \label{fig:case-study-snapshots}
    \vspace{-0.5em}
\end{figure}

\begin{enumerate}
    \item \textbf{Delayed Detection during Cumulative Drift (Figure~\ref{fig:case-study-snapshots}, Left column):} 
    When the manipulators suffer from progressive spatial drift, Diff-DAgger persistently fails to raise an alarm throughout the approach trajectory. Across both block stacking (Rows 1--2) and garment folding (Row 3), it alarms only at a belated stage---after the block cluster is disrupted or the fabric has already suffered a severe angular tilt---severely compressing or completely missing the reaction window required for human teleoperator intervention.
    \item \textbf{Spurious False Alarm on Supported Execution (Figure~\ref{fig:case-study-snapshots}, Middle column):} 
    During nominal, error-free execution, Diff-DAgger's noisy loss expectation frequently induces spurious spikes. This occurs not only during mid-trajectory block transit, but even on completely flat, undisturbed garments (Row 3, Middle), mistakenly triggering false-positive alarms that prematurely halt valid execution.
    \item \textbf{Irreversible Failure Triggering (Figure~\ref{fig:case-study-snapshots}, Right column):} 
    When execution diverges critically (e.g., stacking onto the wrong block or severely bunching fabric), Diff-DAgger fails to catch the deviation during transit. It alerts only post-contact or after the garment is balled into an entangled knot (Row 3, Right). At this terminal stage, in-place teleoperation is impossible, requiring a full manual reset of the environment.
\end{enumerate}

\subsection{Missed Failures and Limitations of EqC}
\label{app:failure-cases}

\paragraph{Real-world failure-detection analysis.}
To characterize the scope of EqC's introspection signal,
we analyzed 44 failed real-world rollouts across different
policy checkpoints: Garment Folding (R0 and R1, 10 each)
and Hierarchical Block Stacking (R3 and R4, 12 each).
Following our deployment protocol, detection requires
$U_{\Delta E}$ to exceed the threshold for $m=3$
consecutive replanning steps.

\begin{table}[!htbp]
    \centering
    \caption{
        \textbf{Real-world failure-detection statistics.}
        Detection requires three consecutive threshold
        crossings.
    }
    \label{tab:real-world-missed-failures}
    \small
    \setlength{\tabcolsep}{7pt}
    \begin{tabular}{lcccc}
        \toprule
        Task & Failed & Detected & Missed & Rate (\%) \\
        \midrule
        Garment Folding & 20 & 18 & 2 & 90.0 \\
        Block Stacking & 24 & 24 & 0 & 100.0 \\
        \midrule
        \textbf{Total} & \textbf{44} & \textbf{42}
        & \textbf{2} & \textbf{95.5} \\
        \bottomrule
    \end{tabular}
\end{table}

\paragraph{Detected and missed failure modes.}
We further categorize the 44 failed rollouts by their
underlying failure mechanisms, as summarized in
Table~\ref{tab:real-world-failure-categories}.

\begin{table}[!htbp]
    \centering
    \caption{
        \textbf{Detected and missed real-world failure modes.}
    }
    \label{tab:real-world-failure-categories}
    \footnotesize
    \setlength{\tabcolsep}{4pt}
    \renewcommand{\arraystretch}{1.10}
    \begin{tabularx}{\linewidth}{
        @{}l>{\raggedright\arraybackslash}Xcc@{}
    }
        \toprule
        Task & Failure type & Detected & Missed \\
        \midrule
        Garment Folding
        & Gripper pins down sleeve, preventing grasping
        & 3 & 0 \\

        Garment Folding
        & Left gripper fails to secure fabric during
          lateral folding
        & 12 & 0 \\

        Garment Folding
        & Corner slips
        & 1 & 0 \\

        Garment Folding
        & Garment caught by gripper
        & 2 & 0 \\

        Garment Folding
        & Excessive upward folding
        & 0 & 2 \\

        \midrule

        Block Stacking
        & Block pushed away during grasping
        & 12 & 0 \\

        Block Stacking
        & Gripper fails to release the stacked block
        & 5 & 0 \\

        Block Stacking
        & Block displaced or knocked over
        & 7 & 0 \\

        \midrule
        \textbf{Total}
        & & \textbf{42} & \textbf{2} \\
        \bottomrule
    \end{tabularx}
\end{table}

Figure~\ref{fig:eqc-detected-failure-modes}
illustrates representative successfully detected
failure modes across both manipulation tasks,
following the category ordering in
Table~\ref{tab:real-world-failure-categories}.

\begin{figure}[htbp]
    \centering
    \includegraphics[
        width=\linewidth
    ]{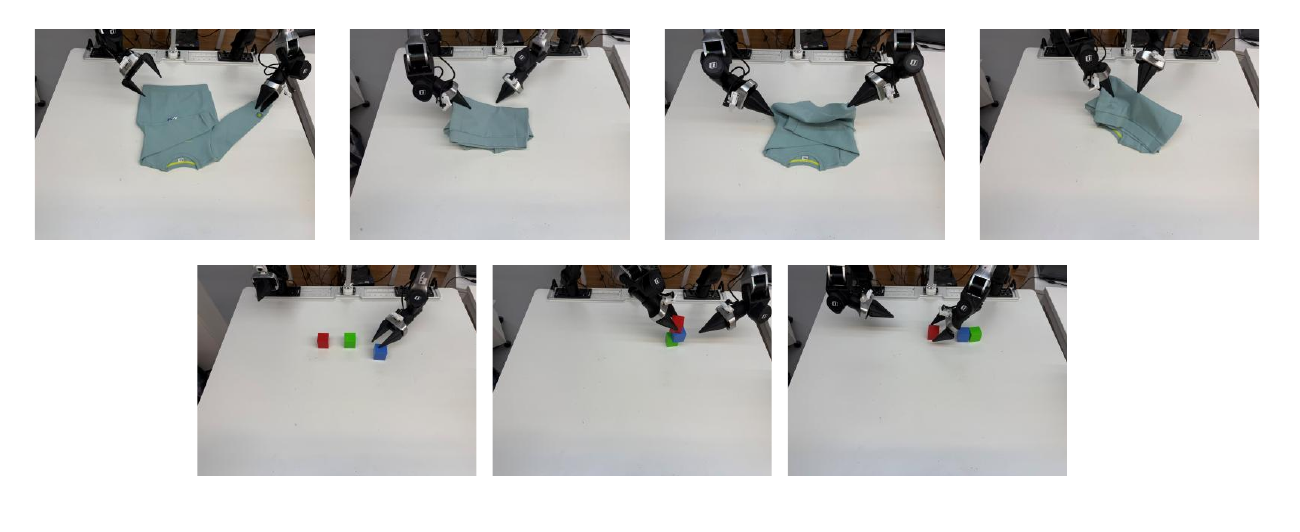}
    \caption{
        \textbf{Representative real-world failure modes
        successfully detected by EqC.}
        \textbf{Top row (Garment Folding), left to right:}
        (1) gripper pins down the sleeve, preventing grasping;
        (2) left gripper fails to secure fabric during lateral folding;
        (3) garment corner slips;
        and (4) garment caught by gripper.
        \textbf{Bottom row (Block Stacking), left to right:}
        (5) block pushed away during grasping;
        (6) gripper fails to release the stacked block;
        and (7) block displaced or knocked over.
    }
    \label{fig:eqc-detected-failure-modes}
\end{figure}

\FloatBarrier

\paragraph{Interpretation and limitations.}
EqC detects 42 of 44 failed rollouts (95.5\%) under
the persistent-alarm criterion. The two missed cases
both involve excessive upward folding, nearly
inverting the garment, as illustrated in
Figure~\ref{fig:eqc-missed-failure}.
Although execution produces an incorrect outcome,
action generation may remain internally consistent
and thus fail to trigger an alarm. Importantly, EqC's competence signal reflects inconsistency in the policy's action-generation process, thereby providing a basis for assessing whether the task has failed, rather than directly verifying task success or action correctness. Its competence awareness is therefore grounded in deviations from the learned energy-consistency relation and does not guarantee detection of every execution failure.

\begin{figure}[htbp]
    \centering
    \includegraphics[
        width=\linewidth
    ]{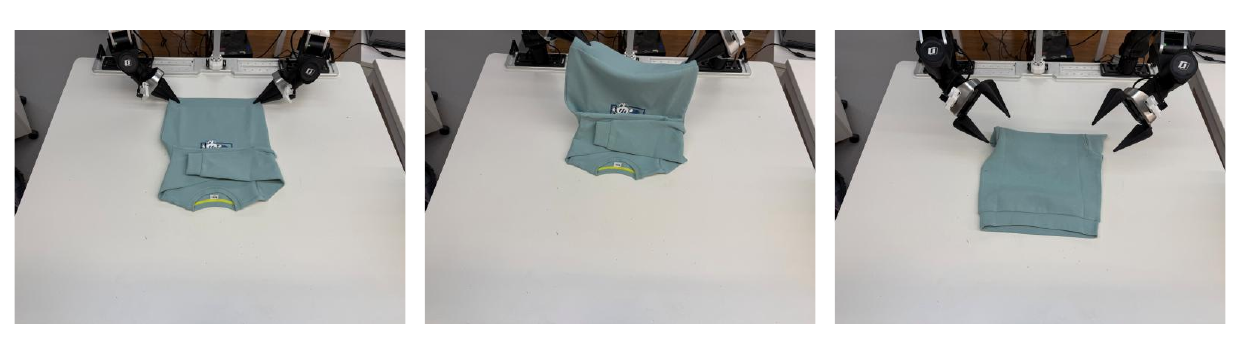}
    \caption{
        \textbf{Real-world failure missed
        by EqC during Garment Folding.}
        Consecutive frames show excessive upward folding
        that nearly inverts the garment.
        Despite the incorrect manipulation outcome,
        the energy-consistency signal does not trigger
        a persistent intervention alarm.
    }
    \label{fig:eqc-missed-failure}
\end{figure}

\FloatBarrier

\paragraph{Potential extensions.}
A promising direction is to complement EqC's
model-internal competence assessment with task-level
verification, such as VLM-based assessment of action
appropriateness and final manipulation outcomes.
Such verification may help identify semantic and
outcome-level errors that internal energy consistency
alone can miss.

\FloatBarrier

\section{Code Availability}\label{app:code-availability}

We have released the code at
\url{https://github.com/hrh6666/EqC}.

\end{document}